\documentclass[10pt]{article} 
\usepackage[accepted]{tmlr}

\usepackage{hyperref,graphicx,amsmath,
amsfonts,amssymb,amsthm,
bm,
url,breakurl,epsf,color,
MnSymbol,mathbbol,
fmtcount,semtrans,caption,multirow,comment,boldline}
\usepackage{wrapfig}
\usepackage{enumitem}
\usepackage{amssymb}
\usepackage{mathrsfs}
\usepackage{nicematrix}

\setlist[itemize]{leftmargin=5mm}

\usepackage[utf8]{inputenc} 
\usepackage[T1]{fontenc}    
\usepackage{url}            
\usepackage{booktabs}       
\usepackage{nicefrac}       
\usepackage{microtype}      
\usepackage{dsfont}
\usepackage{xspace}

\usepackage[bottom,hang,flushmargin,multiple,symbol]{footmisc} 
\usepackage{fnpct}

\newcommand{\mini}{\,\wedge\,}
\newcommand{\maxi}{\,\vee\,}
\newcommand{\Wt}{\widetilde{\W}}
\newcommand{\Ubt}{\widetilde{\Ub}}

\newcommand{\mubp}{\boldsymbol{\mu}_{+}}
\newcommand{\mubn}{\boldsymbol{\mu}_{-}}
\newcommand{\zmu}{Z_\mu}

\newcommand{\znu}{Z_{\nu}}

\newcommand{\zbar}{\bar{Z}}

\newcommand{\attn}{\text{ATTN}\xspace}
\newcommand{\concat}{\operatorname{concat}\xspace}

\newcommand{\iid}{\text{IID}\xspace}

\newcommand{\vb}{{\vct{v}}}
\newcommand{\pb}{{\vct{p}}}
\newcommand{\qb}{{\vct{q}}}

\newcommand{\sft}[1]{\bm{\varphi}(#1)}
\newcommand{\sftt}[2]{\bm{\varphi}_{#1}(#2)}
\newcommand{\sftd}[1]{{\bm{\varphi}}^{\prime}(#1)}
\newcommand{\X}{{\mtx{X}}}

\newcommand{\vct}[1]{\bm{#1}}
\newcommand{\mtx}[1]{\bm{#1}}
\newcommand{\tsub}[1]{\|#1\|_{\psi_2}}

\newcommand{\Deltab}{\mathbf{\Delta}}

\newcommand{\Rcc}{\mathcal{R}^c}

\newcommand{\pare}[1]{{(#1)}}

\usepackage{mathtools}

\usepackage{titlesec}

\usepackage{tikz}
\usepackage{pgfplots}
\usetikzlibrary{pgfplots.groupplots}

\usepackage{movie15}

\usepackage{caption}
\usepackage[bottom,hang,flushmargin]{footmisc} 

\newcommand{\tsn}[1]{{\left\vert\kern-0.25ex\left\vert\kern-0.25ex\left\vert #1 
    \right\vert\kern-0.25ex\right\vert\kern-0.25ex\right\vert}}

\definecolor{darkred}{RGB}{150,0,0}
\definecolor{darkgreen}{RGB}{0,150,0}
\definecolor{darkblue}{RGB}{0,0,200}

\newtheorem{assumption}{Assumption}

\newtheorem{lemma}{Lemma}
\newtheorem{coro}{Corollary}
\newtheorem{fact}{Fact}
\newtheorem{propo}{Proposition}
\newtheorem{theo}{Theorem}
\newtheorem{define}{Definition}
\newtheorem{rem}{Remark}

\newcommand{\Rb}{\mathbf{R}}

\newcommand{\diag}[1]{\operatorname{diag}(#1)}

\DeclareMathOperator{\tr}{tr}

\newcommand{\cut}[1]{\textcolor{red}{}}
\newcommand{\W}{\vct{W}}

\newcommand{\sign}[1]{\texttt{sign}(#1)}

\newcommand{\Ub}{\vct{U}}
\newcommand{\G}{\vct{G}}

\newcommand{\A}{\vct{A}}

\newcommand{\unitvector}[2]{\mathbf{e}^{(#1)}_{#2}}

\newcommand{\mub}{\boldsymbol{\mu}}
\newcommand{\nub}{\boldsymbol{\nu}}

\newcommand{\thetab}{\boldsymbol{\theta}}

\newcommand{\rb}{\vct{r}}

\newcommand{\x}{\vct{x}}

\newcommand{\ub}{\vct{u}}
\newcommand{\w}{{\vct{w}}}

\newcommand{\Wv}{\vec{\W}}
\newcommand{\Ubv}{\vec{\Ub}}
\newcommand{\g}{{\vct{g}}}

\newcommand{\bb}{\vct{b}}

\newcommand{\z}{\vct{z}}

\newcommand{\Iden}{\vct{I}}
\newcommand{\cb}{\vct{c}}  
\newcommand{\ab}{\vct{a}}
\newcommand{\db}{\vct{d}}

\newcommand{\Dc}{\mathcal{D}}

\newcommand{\order}[1]{\mathcal{O}\left(#1\right)}
\newcommand{\ordert}[1]{{\tilde{\mathcal{O}}}\left(#1\right)}

\newcommand{\Nc}{\mathcal{N}}
\newcommand{\Rc}{\mathcal{R}}

\newcommand{\Oc}{\mathcal{O}}

\newcommand{\Lhat}{\widehat{L}}

\newcommand{\beq}{\begin{equation}}
\newcommand{\eeq}{\end{equation}}
\newcommand{\bea}{\begin{align}}
\newcommand{\eea}{\end{align}}

\newcommand{\vp}{\vspace{5pt}}

\newcommand{\R}{\mathbb{R}}
\newcommand{\E}{\mathbb{E}}

\newcommand{\nn}{\notag}

\newcommand{\tn}[1]{\|#1\|}

\newcommand{\term}[1]{\text{Term}_{\rm{#1}}}

  \newcommand{\eps}{\epsilon}

 \newcommand{\deltab}{\boldsymbol\delta}

\newcommand{\thetabt}{\widetilde{\thetab}}
\newcommand{\Phit}{\widetilde{\Phi}}

\newcommand{\Ubstar}{\Ub_\star}
\newcommand{\Wstar}{\W_\star}

\newcommand{\thetabstar}{\thetab_\star}
\newcommand{\gammastar}{\gamma_\star}
\newcommand{\gammalin}{\gamma_{\text{lin}}}

\DeclarePairedDelimiterX{\inp}[2]{\langle}{\rangle}{#1, #2}
\newcommand{\inpb}[2]{\left\langle #1, #2 \right\rangle}

\newcommand{\wt}{\widetilde}

\newcommand{\tsc}[1]{{\fontfamily{pcr}\selectfont #1}}

\newcommand{\ones}{\mathds{1}}

\providecommand{\norm}[1]{\lVert#1\rVert}
\providecommand{\abs}[1]{\left\lvert#1\right\rvert}

\providecommand{\maxnorm}[1]{\lVert#1\rVert_{2,\infty}}

\providecommand{\pare}[1]{(#1)}

\providecommand{\Rad}[1]{\operatorname{Rad}(#1)}
\providecommand{\Unif}[1]{\operatorname{Unif}(#1)}

\newcommand{\ellp}{\ell^\prime}

\newcommand{\Philin}{\Phi_\text{lin}}

\newcommand{\Bnorm}{B_2}
\newcommand{\Bphi}{B_\Phi}

\newcommand{\good}{\texttt{good}\xspace}

\usepackage{hyperref}
\usepackage{overpic}
\usepackage{url}

\title
{On the Optimization and Generalization 
of \\ Multi-head Attention}

\author{\name Puneesh Deora\footnotemark[1] \email puneeshdeora@ece.ubc.ca \\
      \addr University of British Columbia
      \AND
      \name Rouzbeh Ghaderi\footnotemark[1] \email  rghaderi@ece.ubc.ca \\
      \addr University of British Columbia
      \AND
      \name Hossein Taheri\footnotemark[1] \email hossein@ucsb.edu \\
      \addr University of California, Santa Barbara
      \AND
      \name Christos Thrampoulidis \email cthrampo@ece.ubc.ca \\
      \addr University of British Columbia} 

\def\month{04}  
\def\year{2024} 
\def\openreview{\url{https://openreview.net/forum?id=wTGjn7JvYK}} 

\begin{document}

\maketitle
\footnotetext{\footnotemark[1]These authors contributed equally. Alphabetical ordering used.}

\begin{abstract}
The training and generalization dynamics of the Transformer's core mechanism, namely the Attention mechanism, remain under-explored. Besides, existing analyses primarily focus on single-head attention. Inspired by the demonstrated benefits of overparameterization when training fully-connected networks, we investigate the potential optimization and generalization advantages of using multiple attention heads. Towards this goal, we derive convergence and generalization guarantees for gradient-descent training of a single-layer multi-head self-attention model, under a suitable realizability condition on the data. We then establish primitive conditions on the initialization that ensure realizability holds. Finally, we demonstrate that these conditions are satisfied for a simple tokenized-mixture model. We expect the analysis can be extended to various data-model and architecture variations.
\end{abstract}

\addtocontents{toc}{\protect\setcounter{tocdepth}{0}}
\section{Introduction}\label{sec:intro}
Transformers have emerged as a promising paradigm in deep learning, primarily attributable to their distinctive self-attention mechanism.
Motivated by the model's state-of-the-art performance in natural language processing \citep{bert, fewshotlearners, raffel2020transferlearning} and computer vision \citep{dosovitskiy2021image, radford21visualtransfer, touvron21distillation}, 
the theoretical study of the attention mechanism has seen a notable surge in interest recently. 
Numerous studies have already explored the expressivity of Attention, e.g. \citep{baldi2022quarks,dong2021attention,yun2020transformers, yun2020on,sanford2023representational,bietti2023birth}, and initial findings regarding memory capacity have been very recently studied in \citep{baldi2022quarks,dong2021attention,yun2020transformers, yun2020on, mahdavi2023memorization}. In an attempt to comprehend optimization aspects of training attention models, \cite{sahiner2022unraveling,ergen2022convexifying} have  investigated convex-relaxations, while \cite{tarzanagh2023transformers} investigates  the model's implicit bias. 
Additionally, \cite{edelman2021inductive} have presented capacity and Rademacher complexity-based generalization bounds for Self-Attention. 
However, the exploration of the \emph{finite-time} optimization and generalization dynamics of gradient-descent (GD) for training attention models largely remains an open question.

Recent  contributions in this direction, which serve as motivation for our work, include the studies by \cite{jelassi2022vision,li2023theoretical,prompt-attention}. These works concentrate on single-layer attention models with a \emph{single attention head}. Furthermore, despite necessary simplifying assumptions made for the data, the analyses are rather intricate and appear highly specialized on the individual attention and data model. 
These direct and highly specialized analyses present certain challenges. First, it remains uncertain whether they can be encompassed within a broader framework that can potentially be extended to more complex attention architectures and diverse data models. Second, they  appear disconnected from existing frameworks that have been flourishing in recent years for conventional fully-connected and convolutional neural networks e.g., \citep{jacot2018neural,Ji2020Polylogarithmic,richards2021learning,liu2020linearity,taheri2023generalization}. Consequently, 
  it is also unclear how the introduction of attention alters the analysis landscape.

In this work, we study the optimization and generalization properties of multi-head attention mechanism trained by gradient methods.
Our approach
specifically leverages the use of \emph{multiple attention heads}. Despite the operational differences between attention heads in an attention model and hidden nodes in an MLP, we demonstrate, from an analysis perspective, that this parallelism enables the exploitation of frameworks developed for the latter to study the former. 
Particularly for the generalization analysis, we leverage recent advancements in the application of the algorithmic-stability framework to  overparameterized  MLPs \citep{richards2021stability,taheri2023generalization}.

\noindent\textbf{Contributions.}~
We study training and generalization of gradient descent optimization for a multi-head attention (MHA) layer with $H$ heads in a binary classification setting. For this setting, detailed in Section \ref{sec:setup}, we analyze training with logistic loss both the attention weights (parameterizing the softmax logits), as well as, the linear decoder that turns output tokens to label prediction. 

In Section \ref{sec:3}, we characterize key properties of the empirical loss $\Lhat$, specifically establishing that it is self-bounded and satisfies a key self-bounded weak-convexity property, i.e.  $\lambda_{\min}(\nabla^2\Lhat(\thetab))\gtrsim -\frac{\kappa}{\sqrt{H}}\Lhat(\thetab)$ for a parameter $\kappa$ that depends only mildly on the parameter vector $\thetab$. Establishing these properties (and also quantifying $\kappa$) involves carefully computing and bounding the gradient and Hessian of the MHA layer, calculations that can be useful beyond the context of our paper. 

In Sections \ref{sec: train gen}-\ref{sec: gen general}, we present our training and generalization bounds in their most general form. The bounds are given in terms of the empirical loss $\Lhat(\thetab)$ and the distance $\|\thetab-\thetab_0\|$ to initialization $\thetab_0$ of an appropriately chosen target vector $\thetab$. The distance to initialization also controls the minimum number of heads $H\gtrsim\|\thetab-\thetab_0\|^6$ required for the bounds to hold. The choice of an appropriate parameter $\thetab$ that makes the bounds tight is generically specific to the data setting and the chosen initialization. To guide such a choice, in Section \ref{sec:realizability}, we formalize primitive and straightforward-to-check conditions on the initialization $\thetab_0$ that ensure it is possible to find an appropriate $\thetab$. In short, provided the model output at initialization is logarithmic on the train-set size $n$ and the data are separable with respect to the neural-tangent kernel (NTK) features of the MHA model with constant margin $\gamma$, then Corollary \ref{cor:general with good} shows that with step-size $\eta = \widetilde O(1)$  and  $\Theta(n)$ gradient descent steps, the train loss and generalization gap is bounded by $\widetilde{\Oc}(1/n)$ provided only a polylogarithmic number of heads  $H=\Omega(\log^6(n))$. We remark that the aforementioned NTK separability assumption, although related to, differs from the standard NTK analysis. Besides, while this assumption is sufficient to apply  our general bounds, it is not a necessary condition.

In Section \ref{sec:main dm1}, we investigate a tokenized-mixture data model with label-(ir)relevant tokens. We show that after one randomized gradient step from zero initialization, the NTK features of the MHA model separate the data with margin $\gammastar$. Thus, applying our general analysis from Section \ref{sec: train gen}, we establish training and generalization bounds as described above, for a logarithmic number of heads. Towards assessing the optimality of these bounds, we demonstrate that MHA is expressive enough to achieve margin $\gamma_\text{attn}$ that is superior to $\gammastar$. The mechanism to reach $\gamma_\text{attn}$ involves selecting key-query weights of sufficiently large norm, which saturates the softmax nonlinearity by suppressing label-irrelevant tokens. We identify the large-norm requirement as a potential bottleneck in selecting those weights as target parameters in our theory framework and discuss open questions regarding extending the analytical framework into this specific regime.

The remaining parts are organised as follows. Proof sketches of our main training/generalization bounds are given in Section \ref{sec: proof sketch of train gen general}.  The paper concludes in Section \ref{sec: conclusion} with remarks on our findings' implications and  open questions. Detailed proofs are in the appendix, where we also present synthetic numerical experiments.

\noindent \textbf{Related work.} We give a brief overview of the most relevant works on understanding optimization/generalization of self-Attention or its variants. Please see Section \ref{sec:rel} for more detailed exposition. 
\cite{prompt-attention} diverges from traditional self-Attention by focusing on a variant called prompt-Attention, aiming to gain understanding of prompt-tuning. \cite{jelassi2022vision} shed light on how ViTs learn spatially localized patterns using gradient-based methods. \cite{li2023theoretical} provides sample complexity bounds for achieving zero generalization error on training three-layer ViTs for classification tasks for a similar tokenized mixture data model as ours. 
Contemporaneous work \cite{tian2023scan} presents SGD-dynamics of single-layer attention for next-token prediction by re-parameterizing the original problem in terms of the softmax and classification logit matrices, while \cite{tarzanagh2023maxmargin,tarzanagh2023transformers} study the implicit bias of training the softmax weights $\W$ with a fixed decoder $\Ub$. 
All these works focus on a single attention head; instead, we leverage the use of multiple heads to establish connections to the literature on GD training of overparameterized MLPs. Conceptually, \cite{hron2020infinite} drew similar connections, linking multi-head attention to a Gaussian process in the limit as the number of heads approaches infinity.
In contrast, we study the more practical regime of finite heads
and obtain \emph{finite-time} optimization and generalization bounds.

Among the extensive studies on training/generalization of overparameterized MLPs, our work closely aligns with \cite{nitanda2019gradient, Ji2020Polylogarithmic, cao2019generalization, chen2020much, telgarsky2022feature, taheri2023generalization} focusing on classification with logistic loss. Conceptually, our findings extend this research to attention models. The use of algorithmic-stability tools towards order-optimal generalization bounds for overparameterized MLPs has been exploited recently by \cite{richards2021stability, richards2021learning, taheri2023generalization, leistabilitynn}.
To adapt these tools to the MHA layer, we critically utilize the smoothness of the softmax function and derive bounds  on the growth of the model's gradient/Hessian, which establish a self-bounded weak convexity property for the empirical risk (see Corollary \ref{coro:objective gradient/hessian_mainbody}). Our approach  also involves training both the classifier and attention weights, necessitating several technical adjustments detailed in Section \ref{sec: proof sketch of train gen general} and Appendix \ref{sec:app train prep}.

\vspace{-0.1in}
\section{Preliminaries}\label{sec:setup}
\vspace{-0.1in}
\noindent\textbf{Notation.}~$\sft{\cdot}:\R^T\rightarrow\R^T$ denotes the softmax map and $\sftd{\vb} := \nabla\sft{\vb}=\diag{\sft{\vb}}-\sft{\vb}\sft{\vb}^\top$ its gradient at $\vb\in\R^T$. For $t\in[T]$, $\sftt{t}{\vb}$ is the $t$-th entry of $\sft{\vb}\in\R^T$. For $\A \in \R^{n \times m}$, $\A_{i,:}$ is its $i$-th row  and $\A_{:,j}$ is its $j$-th column.
Recall the induced matrix norm $\norm{\A}_{p,q} = \max_{\|\vb\|_p=1} {\norm{\A \vb}_{q}}$ and particularly the following: $\norm{\A}_{2, \infty} = \max_{i \in [n]} \norm{\A_{i,:}} , \, \norm{\A}_{1, 2} = \max_{j \in [m]} \norm{\A_{:,j}},$ and $ \norm{\A}_{1, \infty} = \max_{j \in [m]} \norm{\A_{:,j}}_{\infty}$.
For simplicity, $\|\A\|, \|\vb\|$ denote Euclidean norms and $\lambda_{\min} (\A)$  the minimum eigenvalue. We let $a\mini b = \min\{a,b\}$ and $a\maxi b = \max\{a,b\}$. $\concat$ denotes vector concatenation. 
All logarithms are natural logarithms (base $e$). We represent the line segment between $\w_1,\w_2\in\mathbb{R}^{d'}$ as $[\w_1, \w_2] = \{\w : \w = \alpha \w_1 + (1-\alpha) \w_2, \alpha \in [0, 1]\}$. Finally, to simplify the exposition we use ``$\gtrsim$'' or ``$\lesssim$'' notation to hide absolute constants. We also occasionally use standard notations $\Oc, \Omega$ and $\widetilde\Oc,\widetilde\Omega$ to hide poly-log factors. Unless otherwise stated these order-wise notations are with respect to the training-set size $n$. Whenever used, exact constants are specified in the appendix.


\vp
\noindent\textbf{Single-head Self-attention.}~A single-layer self-attention head $\attn:\R^{T\times d}\rightarrow\R^{T\times d_v}$ with context size $T$ and dimension $d$ parameterized by 
key, query and value matrices $\W_Q,\W_K\in\R^{d\times d_h}, \W_V\in\R^{d\times d_v}$
is given by:
\begin{align*}
\attn(\X;\W_Q,\W_K,\W_V):= \sft{\X\W_Q\W_K^\top\X^\top}\X\W_V\,.
\end{align*}
Here, $\X=[\x_1,\x_2,\ldots,\x_T]^\top\in\R^{T\times d}$ is the input token matrix and  $\sft{\X\W_Q\W_K^\top\X^\top}\in\R^{T\times T}$ is the attention matrix. (Softmax applied to a matrix acts row-wise.) To turn the Attention output in a prediction label, we compose $\attn$ with a linear projection head (aka decoder). Thus, the model's  output is\footnote{While we focus on (i) Full-projection: trainable matrix $\Ub\in\R^{T\times d}$, our results also apply to (ii) Pooling: $\Ub=\ub\ones_T^\top$ with trainable $\ub\in\R^d$, and (iii) Last-token output: $\Ub=\begin{bmatrix} 0_{d\times (T-1)} & \ub\end{bmatrix}^\top$ with trainable $\ub\in\R^d$.}
\begin{align}\label{eq:single-head attn}
    \Phi(\X;\W,\Ub) := \inp{\Ub}{\sft{\X\W\X^\top}\X}\,.
\end{align}
Note that we absorb the value weight matrix $\W_V$ into the projector $\Ub=[\ub_1,\ldots,\ub_T]^\top\in\R^{T\times d}$. Also, we parameterize throughout the key-query product matrix as $\W:=\W_Q\W_K^\top$.
\\
\noindent\textbf{Multi-head Self-attention.} Our focus is on the multi-head attention (MHA) model with $H$ heads:
\[
\sum_{h\in[H]}\attn(\X;{\W_Q}_h,{\W_K}_h,{\W_V}_h){\W_{O}}_h,
\]
for output matrices ${\W_{O}}_h\in\R^{d_v\times d}$. 
Absorbing ${\W_V}_h{\W_{O}}_h$ into a projection layer (similar to the single-head attention) and parameterizing $\W_h:={\W_Q}_h{\W_K}_h^\top$ 
we arrive at the following MHA model:
\begin{align}
    \Phit(\X;\Wt,\Ubt) := \frac{1}{\sqrt{H}}\operatorname{\sum}_{h\in[H]}\Phi(\X;\W_h,\Ub_h)=\frac{1}{\sqrt{H}}\operatorname{\sum}_{h\in[H]}\inp{\Ub_h}{\sft{\X\W_h\X^\top}\X} \, , \label{eq:SA model}
\end{align}
parameterized by $\Wt:=\concat\big(\{\W_h\}_{h\in[H]}\big)$ and $\Ubt:=\concat\big(\{\Ub_h\}_{h\in[H]}\big).$
The $1/\sqrt{H}$ scaling is analogous to the normalization in MLP literature 
e.g. \citep{du2019gradient,ji2021characterizing,richards2021stability}, ensuring the model variance is of constant order when $\Ub_h$ is initialized $\mathcal{O}_H (1)$. Note that these relaxations sacrifice some generality since it is common practice to set $d_h$ and $d_v$  such that $d_v=d/H<d$, thus imposing low-rank restrictions on matrices ${\W_Q}_h{\W_K}_h^\top$, ${\W_V}_h{\W_{O}}_h$. We defer a treatment of these to future work.

Throughout, we will use 
$
\thetab_h := \concat(\Ub_h,\W_h) \in \R^{dT+d^2}\,,
$
to denote the trainable parameters of the $h$-attention head and $
\thetabt:=\concat(\{\thetab_h\}_{h\in[H]}) \in\R^{H(dT+d^2)} \,
$ for the trainable parameters of the overall model. More generally, 
we use the convention of applying ``$\,\widetilde{\,\,\cdot\,\,}\,$'' notation for quantities relating to the multi-head model. Finally, with some slight abuse of 
notation, we  define:  $\maxnorm{\wt\thetab}:=\max_{h\in[H]}\|\thetab_h\|.$



\noindent\textbf{Training.}~
Given training set $(\X_i, y_i)_{i\in[n]}$, with $n$ \iid samples, we  minimize logistic-loss based empirical risk 
    $$\Lhat(\thetabt):= \frac{1}{n}\sum_{i \in [n]} \ell(y_i\Phit(\X_i;\thetabt)) := \frac{1}{n}\sum_{i \in [n]} \log(1+e^{-y_i\Phit(\X_i;\thetabt)}) \, .
    $$
{Our analysis extends to any convex, smooth, Lipschitz and self-bounded loss.  \footnote{A  function $\ell:\R\rightarrow\R$ is self-bounded if $\exists$ $C>0$ such that $|\ell'(t)|\leq C \ell(t).$} The empirical risk is minimized as an approximation of the \emph{test loss} defined as
$
L(\thetabt):= \E_{(\X,y)}[\ell(y\Phit(\X;\thetabt))]\,.
$
We consider standard gradient-descent (GD) applied to empirical risk $\Lhat$. Formally, initialized at $\thetabt^\pare{0}$ and equipped with step-size $\eta >0$, at each iteration $k\geq0$, GD performs the following update:
\begin{align*}
\thetabt^\pare{k+1} =  \thetabt^\pare{k} - \eta \nabla\Lhat(\thetabt^\pare{k})\,.
\end{align*}

\section{Gradient and Hessian bounds of soft-max attention} \label{sec:3}
This section establishes bounds on the gradient and Hessian of the logistic empirical risk $\Lhat(\cdot)$ evaluated on the multi-head attention model. 
To do this, we first derive bounds on the Euclidean norm and spectral-norm for the gradient and Hessian of the self-attention model. In order to simplify notations, we state here the bounds for the single-head model (see App. \ref{app: Gradient/Hessian calculations for multihead-attention} for multi-head model):
$
    \Phi(\X;\thetab) :=\Phi(\X;\W,\Ub)= \inp{\Ub}{\sft{\X\W\X^\top}\X}\,.
$

\begin{lemma}[Gradient/Hessian formulas] \label{lem:grad and hess singlehead}
For all $\ab\in\R^T$,  $\bb,\cb\in\R^d$ the model's gradients satisfy:

    \noindent$\bullet$~ $\begin{aligned}[t]
    \nabla_{\Ub} \Phi(\X; \thetab) = \sft{\X \W \X^\top}\X \,,  \end{aligned}\quad$ and $\qquad \begin{aligned}[t]
    \nabla_{\W} \Phi(\X; \thetab) = \sum_{t=1}^T \x_{t} \ub_{t}^\top \X^\top \sftd{\X \W^\top \x_t} \X \, .\end{aligned}$
    \\
    \noindent$\bullet$~ 
    $\begin{aligned}[t]
    \nabla_{\W} \inp{\ab}{\nabla_{\Ub} \Phi(\X; \thetab)\, \bb} = \sum_{t=1}^T \x_{t} a_t \bb^\top \X^\top \sftd{\X \W^\top \x_t} \X \, , \end{aligned}\quad$ and
    \\
    $ \begin{aligned}[t]
    \nabla_{\W}\inp{\cb}{\nabla_{\W} \Phi(\X; \thetab) \,\bb} &= \sum_{t=1}^T (\cb^\top \x_t) \, \x_t  \mathbf{d}^\top \, \sftd{\X \W^\top \x_t} \X \,\\
    \text{where  } \mathbf{d}&:=\operatorname{diag}(\X \bb)\X\ub_t  - \X\ub_t\bb^\top\X^\top\sft{\X \W^\top \x_t}  - \X\bb\ub_t^\top\X^\top \sft{\X \W^\top \x_t}\,.\end{aligned}$
\end{lemma}
These calculations imply the following useful bounds. 
\vp
\begin{propo}[Model Gradient/Hessian bounds] \label{propo:model bounds general}
    The Euclidean norm of the gradient and the spectral norm of the Hessian of the single-head Attention model \eqref{eq:single-head attn} are bounded as follows:

     \noindent$\bullet$~$\begin{aligned}[t]
    \|\nabla_{\thetab}\Phi(\X;\thetab)\| \leq 2 \, \norm{\X}^2_{2, \infty} \, \sum_{t=1}^T \, \norm{\X \ub_{t}}_\infty + \sqrt{T} \, \norm{\X}_{2, \infty} \, . \end{aligned}$
    
     \noindent$\bullet$~$ \begin{aligned}[t]
    \|\nabla^2_{\thetab}\Phi(\X;\thetab)\| \leq 6 \, d^2 \, \norm{\X}_{2, \infty} \, \norm{\X}_{1, \infty}^3 \, \sum_{t=1}^T \, \norm{\X \ub_t}_\infty + 2 \, d \, \sqrt{T \, d} \, \norm{\X}_{2, \infty} \, \norm{\X}_{1, \infty}^2 \, .
    \end{aligned}$

\end{propo}

Next, we focus on the empirical loss $\Lhat$.
To derive bounds on its gradient and Hessian, we leverage the model's bounds from Proposition \ref{propo:model bounds general} and the fact that logistic loss is self-bounded, i.e., $|\ell'(t)| \leq \ell(t)$. To provide concrete statements, we introduce first a mild boundedness assumption. 
\vp
\begin{assumption}[Bounded data]\label{ass:features} Data  $(\X,y)\in\R^{T\times d}\times\R$ satisfy the following conditions almost surely: $y \in \{\pm 1\}$, and for some $R \geq 1$, it holds for all $t \in [T]$ that $\|\x_t\| \leq R$. 
\end{assumption}


\begin{coro}[Loss properties]
    \label{coro:objective gradient/hessian_mainbody}
    Under Assumption \ref{ass:features}, the objective's gradient and Hessian satisfy the  bounds:\footnote{In all the bounds in this paper involving $\|\thetabt\|_{2,\infty}$, it is possible to substitute this term with $\max_{h\in[H]}\|\Ub_h\|$. However, for the sake of notation simplicity, we opt for a slightly looser bound $\max_{h\in[H]}\|\Ub_h\|\leq \max_{h\in[H]}\|\thetab_h\|=: \|\thetabt\|_{2,\infty}$.} 
    
    \noindent\textbf{(1)}~ $\norm{\nabla \Lhat(\wt\thetab)} \leq  \beta_1(\wt\thetab)\,\Lhat(\wt\thetab) \,,\qquad\qquad~
    \beta_1(\wt\thetab) := \sqrt{T} \, R \, \left(2 \, R^2 \, \maxnorm{\wt\thetab} + 1 \right) \, .$
    
  \noindent\textbf{(2)}~ $\norm{\nabla^2 \Lhat(\wt\thetab)} \leq \beta_2(\wt\thetab) \,, \qquad\qquad~~~~~~~ \beta_2(\wt\thetab) := \frac{1}{\sqrt{H}} \, \beta_3(\wt\thetab) + \, \frac{1}{4} \, \beta_1(\wt\thetab)^2 \, .$
  
  \noindent\textbf{(3)} $\lambda_{\min}(\nabla^2 \Lhat(\wt\thetab)) \geq - \frac{\beta_3(\wt\thetab)}{\sqrt{H}} \Lhat(\wt\thetab) \,\qquad  \beta_3(\wt\thetab) := 2 \, d \, \sqrt{T \, d} \, R^3 \, \left(3 \, \sqrt{d} \, R^2 \, \maxnorm{\wt\thetab} + 1 \right) \, $.
\end{coro}
The loss properties above are crucial for the training and generalization analysis. Property (1) establishes self-boundedness of the empirical loss, which is used to analyze stability of GD updates for generalization.  Property (2) is used to establish descent of gradient updates for appropriate choice of step-size $\eta$. Note that the smoothness upper bound is $\thetabt$-dependent, hence to show descent we need to also guarantee boundedness of the updates. Finally, 
 property (3) establishes a self-bounded weak-convexity property of the loss, which is crucial to both the training and generalization analysis. Specifically, as the number of heads $H$ increases, the minimum eigenvalue becomes less negative, indicating an approach towards convex-like behavior. 

\section{Main results}\label{sec:train gen general}

In this section, we present our training and generalization bounds for multi-head attention.

\subsection{Training bounds}\label{sec: train gen}

We state our main result on  train loss convergence in the following theorem. 
See App. \ref{sec:train} for exact constants and the detailed proofs.

\begin{theo}[Training loss]\label{thm:train}
   Fix iteration horizon $K\geq 1$ and any $\thetabt\in\R^{H(dT+d^2)}$ and $H$ satisfying
\begin{align}
\label{eq: order number of heads train}
{\sqrt{H}} \gtrsim d^2  \, T^{1/2}  R^5 \maxnorm{\thetabt}\norm{\thetabt-\thetabt^{(0)}}^3.
\end{align}
Fix step-size 
$
    \eta \leq 1 \mini 1/\rho(\wt\thetab) \mini \frac{\|\wt\thetab-\wt\thetab^{(0)}\|^2}{K\Lhat(\wt\thetab)} \mini \frac{\|\wt\thetab-\wt\thetab^{(0)}\|^2}{\Lhat(\wt\thetab^{(0)})},$ with 
    $
    \rho(\thetabt) \gtrsim d^{9/2} \, T^{3/2} \, R^{13}\,\maxnorm{\thetabt}^2\,\norm{\thetabt-\thetabt^{(0)}}^2.
$
 Then, the following bounds hold for the training loss and the weights' norm  at iteration $K$ of GD:
\begin{align}\label{eq:train_thm main body}
      &\Lhat(\wt\thetab^{(K)})\leq\frac{1}{K}\sum_{k=1}^K\Lhat(\thetabt_k)\leq 2 \Lhat(\wt\thetab) + \frac{5\norm {\wt\thetab - \wt\thetab^{(0)}}^2}{4\eta K}, \\ \nn
      &\|\wt\thetab^{(K)}-\wt\thetab^{(0)}\|\le 4\|\wt\thetab-\wt\thetab^{(0)}\|.
\end{align}
\end{theo}

Yielding a concrete train loss bound requires an appropriate set of target parameters $\thetabt$ in the sense of minimizing the bound in \eqref{eq:train_thm main body}. Hence, $\thetabt$ should simultaneously attain small loss ($\Lhat(\thetabt)$) and distance to initialization ($\|\thetabt-\thetabt^{(0)}\|$). This desiderata is formalized in Assumption \ref{ass:NNR}  below. The distance to initialization, as well as $\|\thetabt\|_{2,\infty}$, determine how many heads are required for our bounds to hold. Also, in view of the bound in \eqref{eq:train_thm main body}, it is reasonable that an appropriate choice for $\thetabt$ attains $\Lhat(\thetabt)$  of same order as $\|\thetabt-\thetabt^{(0)}\|^2/K.$
 Hence, the theorem's restriction on the step-size is governed by the inverse local-smoothness of the loss: $\eta\lesssim 1/\rho(\thetabt)$.

\subsection{Generalization bounds}\label{sec: gen general}
Next we bound the expected generalization gap. Expectations  are with respect to (w.r.t) randomness of the train set. See App. \ref{sec:generalization} for the detailed proof, which is based on algorithmic-stability. 

\begin{theo}[Generalization loss]\label{thm:gen}
Fix any $K\geq 1$, any $\wt\thetab$ and $H$ satisfying \eqref{eq: order number of heads train}, and any 
step-size $\eta$ satisfying the conditions of Thm. \ref{thm:train}.
Then the expected generalization gap at iteration $K$ satisfies,
\begin{align}
    \E\big[L(\wt\thetab^{(K)})-\Lhat(\wt\thetab^{(K)})\big] \le \frac{4}{n} \, \E\big[2 \, K \, \Lhat (\wt\thetab) + \frac{9\|\wt\thetab-\wt\thetab^{(0)}\|^2}{4\eta}\big] \, .
\end{align}
\end{theo}

The condition on the number of heads is same up to constants to the corresponding condition in Theorem \ref{thm:train}. Also, the generalization-gap bound translates to test-loss bound by combining with Thm. \ref{thm:train}. Finally, similar to Thm. \ref{thm:train}, we can get concrete bounds under the realizability assumption; see Cor. \ref{coro:gen interpol} in App. \ref{app:D2}. 
For the generalization analysis, we require that the realizability assumption holds almost surely over all training sets sampled from the data distribution.


The bounds on optimization and generalization are up to constants same as analogous bounds for logistic regression \citep{soudry2018implicit,ji2018risk,shamir2021gradient}. Yet, for these bounds to be valid,  we require sufficiently large number of heads as well as the existence of an appropriate set of target parameters $\thetabt$, as stated in the conditions of theorem. 
Namely, these conditions are related to the realizability condition, which guarantees small training error near initialization.  The next assumption formalizes these conditions. 


\begin{assumption}[Realizability]\label{ass:NNR}
There exist non-increasing functions $g \; : \; \R_+ \to \R_+$ and $g_0 \; : \; \R_+ \to \R_+$ such that $\forall \eps>0$, there exists model parameters $\thetabt_\pare{\eps} \in \R^{H(dT+d^2)}$ for which: (i) the empirical loss over $n$ data samples satisfies $\Lhat(\thetabt_{(\varepsilon)}) \leq \varepsilon$, (ii)
    $\| \thetabt_{(\varepsilon)} - \thetabt^\pare{0} \|\leq g_0(\varepsilon)$, and, (iii) $ \maxnorm{\thetabt_{(\varepsilon)}} \leq g(\varepsilon).$ 
\end{assumption}

With this assumption, we can specialize the result of Thms. above to specific data settings; see Cor. \ref{coro:train interpol} and \ref{coro:gen interpol} in App. \ref{app:C2} and \ref{app:D2}. In the next section 
we will further show how the realizability assumption is satisfied.  


\subsection{Primitive conditions for checking realizability}\label{sec:realizability}

Here, we introduce a set of more primitive and straightforward-to-check conditions on the data and initialization that ensure the realizability Assumption \ref{ass:NNR} holds.

\begin{define}[Good initialization]\label{def:good init} We say   $\thetabt^\pare{0}=\concat(\thetab_1^\pare{0},\ldots,\thetab_H^\pare{0})$ is a \good initialization with respect to training data $(\X_i,y_i)_{i\in[n]}$ provided the following three properties hold. 
    \\
    P1. \textbf{Parameter $L_{2,\infty}$-bound:} There exists parameter $\Bnorm\geq 1$ such that $\forall h\in[H]$ it holds $\|\thetab^\pare{0}_h\|_2\leq \Bnorm\,.$ 
    \\
    P2. \textbf{Model bound:}  There exists parameter $\Bphi \geq 1$ such that $\forall i \in [n]$ it holds $\abs{\Phit(\X_i; \thetabt^\pare{0})} \leq \Bphi$.
    \\
    P3. \textbf{NTK separability:} There exists $ \thetabt_\star \in \R^{H(dT+d^2)}$ and $\gamma > 0$ such that $\| \thetabt_\star \| = \sqrt{2}$ and $\forall i \in [n]$, it holds
    $y_i\left\langle\nabla\Phit\left(\X_i; \thetabt^\pare{0} \right), \thetabt_\star \right\rangle \geq \gamma.$     
%
%
\end{define}

Prop. \ref{propo:init to real} in the appendix shows that starting from a \good initialization we can always find $\thetabt_\pare{\eps}$ satisfying the realizability Assumption \ref{ass:NNR} provided large enough number of heads. Thus, given \good initialization, we can immediately apply Theorems \ref{thm:train} and \ref{thm:gen} to get the following concrete bounds.

\begin{coro}[General bounds under \good initialization]\label{cor:general with good}
Suppose \good initialization $\thetabt^\pare{0}$ and let 

\begin{align*}
\sqrt{H} \gtrsim d^2 \, T^{1/2} \, R^5\, \Bnorm^2 \, \big(g_0(1/K)\big)^3,   \qquad\text{where}~~ g_0(\frac{1}{K}) = \frac{2\Bphi + \log(K)}{\gamma}\,.
\end{align*}

Further fix  step-size $\eta \leq 1 \mini 1 / \rho(K) 
\mini \frac{4\Bphi^2}{\gamma^2\log(1+e^{\Bphi})}$ with 
$
    \rho(K) \gtrsim d^{9/2}\, T^{3/2} \, R^{13} \, g_0(\frac{1}{K})^4.$
    Then, it holds that
\begin{align*} 
    &\Lhat(\wt\thetab^{(K)})\leq \frac{2}{K} + \frac{5 \left( 2\Bphi + \log(K) \right)^2}{4 \gamma^2 \, \eta \, K},  \quad\text{and}\quad
    \E\big[L(\wt\thetab^{(K)})-\Lhat(\wt\thetab^{(K)})\big] \le \frac{17 \left( 2\Bphi + \log(K) \right)^2}{\gamma^2 \, \eta \, n} \,.
\end{align*}
\end{coro}
Consider training loss after $K$ GD steps: Assuming $\Bphi = \widetilde\Oc_K(1)$ and $\gamma = \Oc_K(1)$, then choosing $\eta=\widetilde\Oc_K({1})$, the corollary guarantees train loss is $\widetilde\Oc_K(\frac{1}{K})$ provided polylogarithmic number of heads $H = \Omega(\log^6(K))$. Moreover, after $K\approx n$ GD steps the expected test loss is $\Oc(1/n)$.


\begin{rem}
    The last two conditions (P2 and P3) for \good initialization are similar to the conditions needed in \citep{taheri2023generalization,Ji2020Polylogarithmic,nitanda2019gradient} for analysis of two-layer MLPs. Compared to \citep{Ji2020Polylogarithmic,nitanda2019gradient} which assume random Gaussian initialization $\thetabt^{(0)}$,  and similar to \citep{taheri2023generalization}  
    the NTK separability assumption (P3) can potentially accommodate deterministic $\thetabt^{(0)}$.
    Condition (P1) appears because we allow training both layers of the model. Specifically the $L_{2,\infty}$ norm originates from the Hessian bounds in Corollary \ref{coro:objective gradient/hessian_mainbody}.   
\end{rem}


\section{Application to tokenized-mixture model }\label{sec:main dm1}
We now demonstrate through an example how our results apply to specific data models.

\paragraph{Data model: An example.}
Consider $M+2$ distinct patterns $\{\mubp,\mubn,\nub_1, \nub_2, ..., \nub_M\}$, where discriminative patterns $\mub_\pm$ correspond to labels $y=\pm1$. 
The tokens are split into (i) a label-relevant set ($\Rc$) and  (ii) a label-irrelevant set ($\Rc^c:= [T]\backslash\Rc$).
 Conditioned on the label and $\Rc$, the tokens $\x_t, t \in [T]$ are \iid as follows
\begin{align}\label{model1}\tag{DM1}
\x_t|y = \begin{cases}
\mub_y &\text{,\,$t \in \Rc$}\\
\nub_{j_t} + \z_t &\text{,\,$j_t \sim$ Unif$(1, ..., M)$ and $t \in \Rc^c$}, \\
\end{cases}
\end{align}

where $\z_t$ are noise vectors. Let $\mathcal{D}$ denote the joint distribution induced by the described $(\mathbf{X}, y)$ pairs. 


\begin{assumption}\label{ass:data} The labels are equi-probable and we further assume the following:
\vspace{2pt}
\\
 $\bullet$ \textbf{Orthogonal, equal-energy means:} All patterns are orthogonal to each other, i.e. $\mubp\perp\mubn\perp\nub_\ell\perp\nub_{\ell'}, \,\, \forall \,\ell,\ell'\in[M].$ 
    Also,  for all $y\in\{\pm1\}, \ell\in[M]$ that $\tn{\mub_y}=\tn{\nub_\ell}=S$, where $S$ denotes the \emph{signal strength}. 
    \vspace{2pt}
\\
 $\bullet$ \textbf{Sparsity level:} The number of label-relevant tokens is $|\Rc|=\zeta T$;  for sparsity level $\zeta\in(0,1)$.
 \vspace{2pt}
    \\
     $\bullet$ \textbf{Noise distribution:} The noise tokens $\z_t$ are sampled from a distribution $\Dc_z$, such that 
  it holds  almost surely for $\z_t\sim\Dc_z$ that
 $
            \abs{\inp{\z_t}{\mub_y}} 
            \leq \zmu, \,\,y\in\{\pm1\}
            $
            and
            $
            \abs{\inp{\z_t}{\nub_\ell}} \leq \znu/M, \,\, \forall \, \ell \in [M]\,.
 $
        Moreover, 
$
        \|\z_t\|\leq Z\,. 
        $
        Overall, Assumption \ref{ass:features} is satisfied with $R=\sqrt{S^2+Z^2+2Z_\nu/M}$.
\end{assumption}
The above assumptions can be relaxed, but without contributing new insights. We have chosen to present a model that is representative and transparent in its analysis. 


\subsection{Finding a \good initialization}
To apply the general Corollary \ref{cor:general with good} to the specific data model \ref{model1}, it suffices to find \good initialization.
While we cannot directly show that $\thetabt^{(0)}=\mathbf{0}$ is \texttt{good}, we can show this is the case for first step of gradient descent $\thetabt^{(1)}$. Thus, we consider training in two phases as follows.

\paragraph{First phase: One step of GD as initialization.}
We use $n_1$ training samples to update the model parameters by running one-step of gradient descent starting from zero initialization. Specifically, 
\begin{align*}
    (\Ub_h^\pare{1},\W_h^\pare{1}) = \thetab_{h}^\pare{1} = \thetab_{h}^\pare{0} - \alpha_h\sqrt{H}\cdot\nabla_{\thetab_h}\Lhat_{n_1}(\thetab_{h}^\pare{0}), \, \, \text{where} \quad \thetab^\pare{0}_h=\mathbf{0} \, \,  \forall \, \, h\in[H].
\end{align*}    
Here, $\alpha_h$ denotes the step-size for head $h\in[H]$ and the scaling by $\sqrt{H}$ guarantees the update of each head is $\order{1}$. The lemma below shows that at the end of this phase, we have
$
    \norm{\bm{U}_{h}^{(1)} - \frac{\zeta \alpha_h}{2} \bm{1}_T \bm{u}_{\star}\top}_F = \mathcal{O}\left(1 / \sqrt{n_1} \right)
$
where $\ub_\star$ is the oracle classifier $\ub_\star=\mubp-\mubn$. On the other hand, the attention weight-matrix does \emph{not} get updated; the interesting aspect of the training  lies in the second phase, which involves updating $\W$. 

\begin{lemma}[First phase]\label{lem:first phase}
After the first-gradient step as described above, we have $
  \Ub_h^\pare{1} = \alpha_h\, \ones_T \,\big(\frac{\zeta}{2}\ub_\star^\top + \pb^\top \big) $ and $ \W_h^\pare{1} = \mathbf{0}$.
where with probability at least $1-\delta\in(0,1)$ over the randomness of labels there exists positive universal constant $C>0$ such that
\begin{align}
\label{eq:p bound}
    \|\pb\| \leq C\,\left(2S+Z\right)\,\big(\sqrt{\frac{d}{n_1}} + \sqrt{\frac{\log(1/\delta)}{n_1}}\big)=:P\,.
\end{align}
\end{lemma}

\paragraph{Second phase: GD with constant step size.}
During the second phase, $K$ gradient steps are performed on $n$ new samples (distinct from those used in the first phase). Concretely,
$
\thetabt^\pare{k+1}=\thetabt^\pare{k}-\eta\cdot \nabla_{\thetabt}\Lhat_n(\thetabt^\pare{k}), \quad k=1,\ldots,K\,,
$
with $\thetabt^\pare{1}=\concat\big(\{\thetab_h^\pare{1}\}_{h\in[H]}\big)$ the step obtained by the first-phase update and $\eta$  the  step-size of the second phase. 
In order to analyze the second phase, during which both $\Wt$ and $\Ubt$ get updated, we employ the general results of Section \ref{sec:train gen general}. To do so, we show that $\thetabt^{(1)}$ serves as \good initialization as per Definition \ref{def:good init}.

\begin{propo}\label{propo:good init dm1}
    Consider the first-phase iterate $\{\thetab_h^\pare{1}\}_{h\in[H]}$ and condition on the event $\|\pb\|\leq P$ (depending only on the data randomness in the first phase) of  Lemma \ref{lem:first phase}. Suppose the step-size of the first phase is chosen \iid $\alpha_h\sim\Unif{\pm1}, \, h\in[H]$. Then, the initialization $\thetabt^\pare{1}=\concat\big(\thetab_1^\pare{1},\ldots,\thetab_H^\pare{1}\big)$ is \good with respect to data sampled from \ref{model1} and satisfying Assumption \ref{ass:data}. Specifically,  the three desired properties hold as follows. 
    \\
    $\bullet$  Almost surely, \textbf{P1} holds with $\Bnorm=\sqrt{T}(S+P)\,.$
    \\
        $\bullet$ With probability $1-\delta\in(0,1)$, \textbf{P2} holds  with $\Bphi= T R (S + P)\sqrt{2\log(n/\delta)}\,.$ 
        \\
   $\bullet$
          Suppose 
        $
        \sqrt{H}  \gtrsim  \frac{R^4T(S+P)}{\gammastar}\cdot\sqrt{2\log(n/\delta)}\,.
        $
        Then, with probability $1-\delta\in(0,1)$,
        \textbf{P3} holds with $\gamma=\gammastar/2$ where
        \begin{align}\label{eq:gammastar}
        \gammastar := \frac{T(1-\zeta)\zeta\big(\zeta S^4-7\zbar S^2-12\zbar^2 -16\frac{\zbar^3}{S^2}\big)}{4\sqrt{2(M+1)}} -PT^{5/2}(S+Z)^3 + \frac{S\,\sqrt{T}\Big(\zeta- 2(1-\zeta) \frac{\zmu}{S^2}\Big)}{\sqrt{2}}\,,
        \end{align}
        and $\zbar := \zmu \maxi \znu$.   The randomness is with respect to the sampling of $\alpha_h , \, h\in[H]$.

\end{propo}

The parameter $\gammastar$ in \eqref{eq:gammastar} represents the NTK margin of the model at initialization $\thetabt^{(1)}$. By Corollary \ref{cor:general with good}, larger  margin translates to better train/generalization bounds and smaller requirements on the number of heads. For a concrete example, suppose $T \maxi M=\order{1}$ and $Z \maxi \zbar=\order{S}$. Then, provided  first-phase sample size $n_1\gtrsim S^2d$ so that $P=\order{1}$, it holds $\gammastar=\gammalin + \Omega(\zeta^2(1-\zeta)S^4)$, where $\gammalin = \Omega({\zeta S})$ is the margin of a linear model for the same dataset (see App. \ref{app:linear}). Overall, applying Cor. \ref{cor:general with good} for $K=n$ and a polylogarithmic $\operatorname{polylog}(n)$ number of heads  leads to $\ordert{\frac{1}{\eta \gammastar^2 n}}$ train loss and expected generalization gap.

\subsection{Proof sketch of \textbf{P3}: NTK separability}
It is instructive to see how \textbf{P3} follows as it sheds light on the choice of an appropriate target parameter $\thetabt$ as per Thms. \ref{thm:train} and \ref{thm:gen}. We choose 
\[
\W_\star = \mub_+\mub_+^{\top} + \mub_{-}\mub_{-}^{\top} + \sum_{\ell\in[M]}\nub_\ell(\mub_++\mub_{-})^{\top} \, \quad\text{and}\quad
\Ub_\star = \ones_T \ub_\star^{\top} = \ones_T (\mub_+-\mub_-)^{\top} \, ,
\]
and normalize parameters such that
$\thetab_\star := (\overline{\Ub}_\star = \frac{1}{\norm{\Ub_\star}_F} \Ub_\star \, , \,\, \sign{\alpha} \overline{\W}_\star = \sign{\alpha} \frac{1}{\norm{\W_\star}_F} \W_\star) \,$. It is easy to see that $\Ub_\star$ is the optimal classifier for the label-relevant tokens. To gain intuition on the choice of $\W_\star$, note that $
\Wstar=\W_{K,\star}\W_{Q,\star}^\top$, with key-query matrices chosen as $\W_{K,\star}=\begin{bmatrix}\mub_+ & \mub_- & \nub_1 & \cdots & \nub_M\end{bmatrix}\in\R^{d\times(M+2)}$ and $\W_{Q,\star}=\begin{bmatrix}\mub_+ & \mub_- & \mub_++\mub_- & \cdots & \mub_++\mub_-\end{bmatrix}\in\R^{d\times(M+2)}$. With these choices, the relevance scores (aka softmax logits) of relevant tokens turn out to be strictly larger compared to the irrelevant tokens. Concretely, we show in App. \ref{sec:ntk separability} that the $t$-th row $\rb_t(\X;\Wstar) =  \X \W_\star^\top \x_t$  of the softmax-logit matrix satisfies the following:
\begin{align}\label{eq:rttmain}
\forall t: [\rb_t]_{t'} =
\begin{cases}
    \order{S^4} & \text{,} \, t' \in \Rc\,, \\
    \order{S^2} & \text{,} \, t' \in \Rc^c\,. \\
\end{cases}
\end{align}
Thus, under this parameter choice, softmax can attend  to label-relevant tokens and supresses noisy irrelevant tokens. In turn, this increases the signal-to-noise ratio for classification using $\Ub_\star$. 

We now show how to compute $\E_{\thetab^\pare{1}} \,y\inp{\nabla_{\thetab}\Phi(\X; \thetab^\pare{1})}{\thetab_\star}$ for a single head. Recall $\thetab_\star$ consists of $\overline{\Ub}_\star$, $\overline{\W}_\star$. First, since $\W^{\pare{1}} = \textbf{0}$, using Assumption \ref{ass:data}, a simple calculation shows 
    $y\,\inp{\nabla_{\Ub} \Phi(\X; \thetab^\pare{1})} {\overline{\Ub}_\star} \geq \frac{S\,\sqrt{T}}{\sqrt{2}}\, \Big(\zeta- 2(1-\zeta) \frac{\zmu}{S^2}\Big).$
Second, to compute $\E_{\alpha \sim \Unif{\pm1}} y \inp{\nabla_{\W} \Phi\left(\X;\thetab^\pare{1} \right)}{\sign{\alpha} \overline{\W}_\star}$ it follows from  
 Lemma \ref{lem:grad and hess singlehead}  that 
\begin{align*}
\nabla_{\W}\Phi(\X; \thetab^\pare{1}) = \frac{\alpha\zeta}{2} \sum_{t \in [T]} \x_t \ub_\star^\top \X^{\top} \sftd{\textbf{0}} \X + \alpha \sum_{t \in [T]} \x_t \pb^\top \X^{\top} \sftd{\textbf{0}} \X \, .
\end{align*}
Note the first term is dominant here since the second term can be controlled by making $\|\pb\|_2$ small as per Lemma \ref{lem:first phase}. Thus, ignoring here the second term (see Appendix \ref{sec:ntk separability} for full calculation) $y \inp{\nabla_{\W}\Phi(\X; \thetab^\pare{1})}{\Wstar}$ is governed by the following term: $
  \frac{\alpha\zeta}{2}\sum_{t \in [T]} y\, \ub_\star^\top \X^{\top} \sftd{\textbf{0}} \X \W_\star^\top\x_t = \frac{\alpha\zeta}{2}\sum_{t \in [T]} y\, \ub_\star^\top \X^{\top} \sftd{\textbf{0}} \, \rb_t\,. 
$
Note that $\sftd{\textbf{0}} = \Iden_T-\frac{1}{T} \ones_T \ones_T^\top$. To simplify the exposition here, let us focus on the identity component and leave treatment of the  the rank-one term to the detailed proof. The corresponding term then becomes
\begin{align*}
\frac{\alpha\zeta}{2}\sum_{t\in[T]}\sum_{t'\in[T]} \underbrace{\left(y\, \ub_\star^\top \x_{t'} \right)}_\text{class. logits}\cdot\underbrace{\left([\rb_t]_{t'}\right)}_\text{softmax logits},
\end{align*}
which involves for each output token $t$, the sum of products over all tokens $t'\in[T]$ of softmax logits (i.e. relevant scores $[\rb_t]_{t'}$) and corresponding classification logits (i.e. $y\, \ub_\star^\top\x_{t'}$). Note that by choice of $\ub_\star$ and $\W_\star$, both the classification and softmax logits are large from label-relevant tokens, while being small for noise tokens. Intuitively, this allows for a positive margin $\gamma_\star$ as stated in Proposition \ref{propo:good init dm1}. We defer the detailed calculations to Appendix \ref{sec:ntk separability}.

In the appendix, we also detail how to yield the computation for the MHA, which builds on the calculations for the single-head attention model above. In short, 
we simply choose multi-head parameter $\thetabt_\star$ as $
    \thetabt_\star=\frac{1}{\sqrt{H}}\concat\left(\thetab_\star\big(\thetab_1^\pare{1}\big),\ldots,\thetab_\star\big(\thetab_H^\pare{1}\big)\right) \, .
$
This guarantees that $\|\thetabt_\star\| = \sqrt{2}$ and maintains the multi-head NTK margin be at least $\gamma_\star$ in expectation. To complete the proof, it remains to get a high-probability version of this bound. To do this, notice that $\thetab_h^{(1)}$ are \iid, hence we can apply Hoeffding's inequality, which finally gives the desired bound on the NTK margin provided sufficient number of heads $H$, which controls the degree of concentration when applying Hoeffding's inequality. See Lemmas \ref{lem:empirical margin general} and \ref{lem: margin whp} for details.



\subsection{Is the NTK margin optimal?} \label{sec: Is the NTK margin optimal?}
Below, we discuss the optimality of the NTK margin $\gammastar$. First, define set of parameters $\thetab_\text{opt}  := (\Ub_\text{opt}, \W_\text{opt})$:
\begin{align}
        \Ub_\text{opt}:= \frac{1}{S\sqrt{2T}} \, \Ub_\star \,  \qquad \text{and} \qquad
        \W_\text{opt} := \frac{1}{S^2\sqrt{2(M+1)}} \W_\star \, , \label{eq:Wstar}
\end{align}
    normalized so that $\|\thetab_\text{opt}\|_F = \sqrt{2}$. Recall here the definitions of $\Ub_\star,\W_\star$ in the section above. As we already explained above, this choice of parameters guarantees that relevant tokens are assigned larger relevance and classification scores compared to irrelevant ones. Specifically about $\W_\star$, we saw in Eq. \eqref{eq:rttmain} that it ensures a gap of $\Oc(S^2)$ between relevance scores of label-relevant and label-irrelevant tokens. Thanks to this gap, it is possible for softmax to fully attend to the label-relevant tokens by saturating the softmax. To do this, it suffices to scale-up $\Wstar$ by an amount $\propto 1/S^2$. This is formalized in the proposition below.
\begin{propo}[Attention expressivity for tokenized-mixture model]\label{propo:single head good theta}
    Consider single-head attention model. Suppose the noise level is such that $\zmu=\znu\leq S^2/8$. For any $\eps>0$, consider $\Gamma_\eps$ satisfying
    $
        \Gamma_\eps \geq 
        \frac{8\sqrt{2(M+1)}}{3S^2}\log\left(\frac{\zeta^{-1}-1}{\eps}\right)\,.
    $
    Then, the attention scores corresponding to weights $\Gamma_\eps\cdot\W_\text{opt}$ satisfy 
    \begin{align}\label{eq: attn bound for W statement}
    \forall t\in[T]\,:\, 0 \leq 1-\sum_{t'\in\Rc} \sftt{t'}{\x_t^\top \, \Gamma_\eps\W_\text{opt}\,\X^T} = \sum_{t'\in\Rcc} \sftt{t'}{\x_t^\top \, \Gamma_\eps\W_\text{opt}\,\X^T} \leq \eps\,.
\end{align}
    Thus, almost surely over data $(\X,y)$ generated from data model \ref{model1} the margin of single-head attention with parameters $(\Ub_\text{opt},\Gamma_\eps \cdot \W_\text{opt})$ satisfies
    \begin{align}\label{eq:margin single} 
    y\Phi(\X; \Ub_\text{opt}, \Gamma_\eps \cdot \W_\text{opt}) \geq \gamma_\text{attn}:=\gamma_\text{attn}(\eps):=\frac{\sqrt{T}}{\sqrt{2}S}\left(S^2(1-\eps)
-2\eps\zmu\right)\,.
    \end{align}
\end{propo}
From Eq. \eqref{eq: attn bound for W statement}, note that as $\eps\rightarrow 0$ and $\Gamma_\eps \rightarrow \infty$, the softmax map saturates, i.e. it approaches a hard-max map that attends only to the label-relevant tokens ($\Rc$) and suppress the rest ($\Rc^c$). As a consequence of this, Eq. \eqref{eq:margin single} shows that the achieved margin approaches $\gamma_\text{attn}:=S\sqrt{T}/\sqrt{2}$. Note this is independent of the sparsity level $\zeta$. In particular, $\gamma_\text{attn}\geq \gammastar \geq \gammalin$ and the gap increases with decreasing sparsity. See appendix for experiments and discussion regarding the margin achieved by GD for data model \ref{model1}. 

{Following Proposition \ref{propo:single head good theta}, a natural question arises: Is it possible to choose ``good'' parameters $\thetabt=(\widetilde\Ub,\widetilde\W)$ based on the set of optimal parameters $\thetab_\text{opt}$? This would then yield train-loss and expected generalization-gap bounds $\ordert{1/({\eta\gamma_\text{attn}^2 n})}$ after $\Theta(n)$ steps of GD starting at $\thetabt^\pare{0} = \mathbf{0}$. To investigate this question, define the following parameters  for each head,  aligning with the aforementioned ``good'' directions of Proposition \ref{propo:single head good theta}:
\begin{align*}
\Ub_h := \frac{\log(n)}{\gamma_\text{attn}}\frac{1}{H^{1/2}} \, \Ub_\text{opt}, \hspace{2mm} \W_h := \frac{C}{H^p} \, \W_\text{opt} \, ,
\end{align*}
for some $C>0, \, p>0, \, \text{and} \, \forall h \in [H]$. To yield the margin $\gamma_\text{attn}$ of \eqref{eq: attn bound for W statement}, we need that each $\W_h$ has norm at least $\Gamma_\eps\propto1/S^2$. Thus, we need 
$
    \norm{\W_h} \gtrsim \frac{1}{S^2} \implies S^2 \gtrsim \frac{1}{C} \cdot H^p \, .
$
Now, in order to apply Thms. \ref{thm:train} and \ref{thm:gen}, the requirement on the number of heads $H$ in terms of distance of $\thetabt$ to  $\thetabt^\pare{0}=\mathbf{0}$ yields the following condition:
\begin{align} \label{eq: condition on scaling}
    H^{1/2} \gtrsim S^5 \norm{\thetabt}^3 \, .
\end{align}
Note that $\norm{\Ubt} = \frac{\log(n)}{\gamma_\text{attn}} \approx \frac{\log(n)}{S} , \, \norm{\Wt} = C \cdot H^{1/2-p}$. Hence, in computing $\norm{\thetabt}$, we distinguish 
 two  cases. 
 \\
 First, assume that $\norm{\Wt} \geq \norm{\Ubt}$ which implies that $S \gtrsim \frac{\log(n)}{C} \cdot H^{p-1/2}$ and 
 $\norm{\thetabt} \gtrsim \norm{\Wt} \maxi \norm{\Ubt} = C \cdot H^{1/2 - p}$.
 Since
 \begin{align*}
     S \gtrsim \frac{1}{C^{1/2}} \cdot H^{p/2} \maxi \frac{\log(n)}{C} \cdot H^{p-1/2} \, ,
 \end{align*}
 by using Eq. \eqref{eq: condition on scaling}, we get the following conditions on $H$: 
\begin{align*}
    H^{1/2} \gtrsim S^5 \cdot C^3 \cdot H^{3/2 - 3p} \gtrsim C^{1/2} \cdot H^{3/2 - p/2} \maxi \frac{\log^5(n)}{C^2} \cdot H^{2p - 1} \implies H^{p-2} \gtrsim C \,\,\,\, \text{and} \,\,\,\, H^{p-3/4} \lesssim \frac{C}{\log^{5/2}(n)} \, .
\end{align*}
Combining these two gives
$
    C \lesssim \frac{C}{\log^{5/2}(n)} \implies \log(n) \lesssim 1 \, ,
$
a contradiction since $n > 1$. Thus, there are no possible choices for $p$ and $C$ that satisfy both conditions.} The case $\norm{\Wt} \leq \norm{\Ubt}$ can be treated similarly leading to the same conclusion; thus, is omitted for brevity.



Intuitively, this contradiction arises because of the large $\norm{\W_h}$ requirement to achieve margin $\gamma_\text{attn}$. Finally, one can ask if it is possible to resolve the contradiction by changing the scaling of normalization with respect to $H$ in the MHA model Eq. \eqref{eq:SA model}, from $1/H^{1/2}$ to $1/H^c$ for $c>0$. It can be shown via the same argument that no such value of $c$ exists for which $\thetabt$ constructed above satisfies the overparameterization requirement $H^c\gtrsim S^5\|\thetabt\|^3$. We thus conclude that the construction of weights in Proposition \ref{propo:single head good theta} does not yield a target parameter that simultaneously achieves low empirical loss and allows choosing $H$ large enough as per \eqref{eq: order number of heads train}. This triggers interesting questions for future research: Does GD converge to weights attaining margin $\gamma_\text{attn}$ as in Proposition \ref{propo:single head good theta}? If so, under what conditions on initialization? See also the remarks in Section \ref{sec: conclusion}.
\section{Proof Sketch of Section \ref{sec:train gen general}} \label{sec: proof sketch of train gen general}
Throughout this section we drop the $\widetilde{\cdot}$ in $\thetabt$ and $\Phit(\X_i; \thetabt)$ as everything refers to the full model. Moreover, $\thetabt^{(K)}$ and $\thetabt^{(0)}$ are denoted by $\thetab_{K}$ and $\thetab_{0}$. Refer to Figure \ref{fig:train-pf-schema} in the App. for a summary of the sketch.


\subsection{Training analysis}
The proof begins by showing step-wise descent for any iteration $k \geq 0$ of GD (see Lemma \ref{lem:descent}), where step-size at each iteration $\eta_k \leq \frac{1}{\rho_k}$ depends on the objective's local smoothness parameters $\rho_k = \beta_2(\thetab_k) \maxi \beta_2(\thetab_{k+1})$: 
\begin{align}\label{eq:descent main}
\Lhat(\thetab_{k+1}) \leq \Lhat(\thetab_k) - \frac{\eta_k}{2}\left\|\nabla \Lhat \left( \thetab_k \right) \right\|^2 \, .
\end{align}
Now, using Taylor's theorem we can link $\Lhat(\thetab_k)$ to $\Lhat(\thetab)$ for any $\thetab$ as follows:
\begin{align}\label{eq:taylor sketch}
\Lhat(\thetab) \geq \Lhat(\thetab_k)  + \langle \nabla \Lhat(\thetab_k), \thetab - \thetab_k \rangle +\frac{1}{2} \min_{\thetab_{k_{\alpha}}}\lambda_{\min} \left(\nabla^2 \Lhat(\thetab_{k_{\alpha}})\right) \norm{\thetab - \thetab_k}^2,
\end{align}
where $\thetab_{k_{\alpha}}:= \alpha \thetab_k + (1-\alpha)\thetab, \, \, \alpha \in [0, 1]$. We can plug this into \eqref{eq:descent main} to relate the loss at iterates $\thetab_{k}$ and $\thetab_{k+1}$. 
To continue, we need to lower bound $\min_{\thetab_{k_{\alpha}}}\lambda_{\min} \left(\nabla^2 \Lhat(\thetab_{k_{\alpha}})\right)$. For this, we use the following property of the loss objective from Corollary \ref{coro:objective gradient/hessian_mainbody}:
$
\forall\thetab\,:\,\,\,\lambda_{\min}\left(\nabla^2 \Lhat(\thetab)\right)\geq - \kappa(\thetab)\cdot \Lhat(\thetab),$ where $\kappa(\thetab):=\frac{\beta_3(\thetab)}{\sqrt{H}}.
 $
 Note from the definition of $\beta_3(\cdot)$ that 
 $
 \forall \thetab_1,\thetab_2\,:\,\,\, \max_{\thetab\in[\thetab_1,\thetab_2]}\beta_3(\thetab) = \beta_3(\thetab_1) \maxi \beta_3(\thetab_2)\,.
 $
 Thus, the above property of the loss implies the following \emph{local self-bounded weak convexity} property on the line $[\thetab_1,\thetab_2]$ for arbitrary points $\thetab_1,\thetab_2$:
 \begin{align}\label{eq:local weak convexity main}
     \forall\thetab_1,\thetab_2\,:\,\,\, \min_{\thetab\in[\thetab_1,\thetab_2]} \lambda_{\min}\left(\nabla^2 \Lhat(\thetab)\right)\geq - \frac{\beta_3(\thetab_1)\maxi \beta_3(\thetab_2)}{\sqrt{H}}\cdot \max_{\thetab\in[\thetab_1,\thetab_2]} \Lhat(\thetab) \, .
 \end{align}
Therefore, using Eq. \eqref{eq:local weak convexity main} in Eq. \eqref{eq:taylor sketch}, we can get:
\begin{align}\label{eq:taylor sketch prime}
\Lhat(\thetab) \geq \Lhat(\thetab_k) + \langle \nabla \Lhat(\thetab_k), \thetab - \thetab_k \rangle - \frac{1}{2} \frac{\beta_3(\thetab_1)\maxi \beta_3(\thetab_2)}{\sqrt{H}}\cdot \max_{\alpha \in [0,1]} \Lhat(\thetab_{k_{\alpha}}) \, \norm{\thetab - \thetab_k}^2 \, .
\end{align}
To apply the Descent Lemma in \eqref{eq:descent main}, we need to fix a step-size such that satisfies the condition of the Lemma at each iteration $\eta \leq \eta_k$ for all $k<K$. Then, combining with Eq. \eqref{eq:taylor sketch prime} and applying standard telescope summation, we arrive at the following:
\begin{align} \label{eq: regret step 0}
\frac{1}{K}\sum_{k=1}^K\Lhat(\thetab_k) \leq \Lhat(\thetab) + \frac{\norm {\thetab - \thetab_0}^2}{2\eta K} + \frac{1}{2K}\sum_{k=0}^{K-1} \frac{\beta_3(\thetab) \maxi \beta_3(\thetab_k)}{\sqrt{H}} \cdot \max_{\alpha \in [0,1]} \Lhat(\thetab_{k_{\alpha}}) \norm{\thetab - \thetab_k}^2 \, .
\end{align}

Next, we use the following generalized local quasi-convexity (GLQC) of the loss function.

\begin{propo}[GLQC property: Slight variation of Prop. 8 of \cite{taheri2023generalization}]\label{prop:GLQCapp}
Let $\thetab_1$ and $\thetab_2$ be two points that are sufficiently close to each other, such that
\begin{align}\label{eq:condition to GLQC}
{2\,\left(\beta_3(\thetab_1)\maxi \beta_3(\thetab_2)\right)}\,\|\thetab_1-\thetab_2\|^2 \leq \sqrt{H}     \,.
\end{align}
Then, 
$
\max_{\thetab\in[\thetab_1,\thetab_2]} \widehat L(\thetab)\le \frac{4}{3}\,\left(\widehat L(\thetab_1)\maxi \widehat L(\thetab_2)\right).
$
\end{propo}

Using Proposition \ref{prop:GLQCapp} in Eq. \eqref{eq: regret step 0} and assuming sufficiently large heads $H$ such that $\sqrt{H} \geq {2\,\left(\beta_3(\thetab)\maxi \beta_3(\thetab_k)\right)}\,\|\thetab-\thetab_k\|^2$, we can get the advertised regret bound in \eqref{eq:train_thm main body}.

In order to remove the dependence of $H$ on iteration $k$, by an induction argument we can show bounded iterates-norm i.e. $\|\thetab_k-\thetab\| \leq 3\| \thetab-\thetab_0\|$ (see Lemma \ref{lem:iteratenormbound}). Using this and the definition of $\beta_3(\cdot)$ we can control  $\beta_3(\thetab)\maxi \beta_3(\thetab_k)$ as $(\beta_3(\thetab)\maxi \beta_3(\thetab_k)) \lesssim \norm{\thetab-\thetab_0} + \maxnorm{\thetab}$ to get the desired requirement of heads ${\sqrt{H}} \gtrsim \maxnorm{\thetab}\norm{\thetab-\thetab_{0}}^3$ stated in Eq. \eqref{eq: order number of heads train}. 

The remaining piece to guarantee descent at each step is establishing a $\rho(\thetab)$ such that $\rho_k \leq \rho(\thetab)$ for all $k < K$. 
To do this, we recall that $\rho_k = \beta_2(\thetab_k) \maxi \beta_2(\thetab_{k+1})$. By definition of $\beta_2(\cdot)$ in Corollary \ref{coro:objective gradient/hessian_mainbody}, we can control $\beta_2(\thetab_k) \maxi \beta_2(\thetab_{k+1})$ with controlling $\maxnorm{\thetab_k} \maxi \maxnorm{\thetab_{k+1}}$ as $\left( \maxnorm{\thetab_k} \maxi \maxnorm{\thetab_{k+1}} \right) \lesssim \norm{\thetab - \thetab_k} + \maxnorm{\thetab} + 1$. Using iterates-norm bound and setting 
$\rho(\thetab) = \Big(\frac{2 \, d \, \sqrt{T \, d} \, R^3}{\sqrt{H}} +  \frac{T \, R^2}{4}\Big)\, \alpha(\thetab)^2$ {with} 
    $\alpha(\thetab) := 3 \, d \, \sqrt{d} \, R^2 \, \left[3 \, \sqrt{T} \, R^3 \left(3 \, \norm{\thetab - \thetab_0} + \maxnorm{\thetab} \right) +  2\, \sqrt{T} \, R \right]$,  
satisfies the desired condition for the Descent Lemma completing the proof.

\subsection{Generalization analysis}
In order to bound the expected generalization gap, we leverage the  algorithmic stability framework. To begin, consider the leave-one-out (loo) training loss $\Lhat^{\neg i}(\thetab):=\frac{1}{n}\sum_{j\neq i} \ell_j(\thetab)$ for $i\in[n]$, where $\ell_j(\thetab) := \ell(y_j \Phi(\X_j;\thetab))$ denotes the $j$-th sample loss. With these, define the loo model updates of GD on the loo loss for $\eta >0$:
\[
\thetab_{k+1}^{\neg i}:=\thetab_{k}^{\neg i} -\eta \nabla \Lhat ^{\neg i}(\thetab_k^{\neg i}),~k\geq 0,\qquad \thetab_0^{\neg i}=\thetab_0 \, .
\]
The following lemma relates expected generalization loss to average model stability for any $G$-Lipschitz loss.
\begin{lemma}[\cite{lei2020fine}, Thm. 2]\label{lem:on-average}
For $G$-Lipschitz loss and for all iterates $K$, it holds that $\E\Big[L(\thetab_K)-\Lhat(\thetab_K)\Big] \leq  2G \cdot\E\Big[\frac{1}{n}\sum_{i=1}^n\|\thetab_K-\thetab_K^{\neg i}\|\Big] \,.$
\end{lemma}
To bound the average model-stability on the r.h.s of the lemma's inequality, we use GD expansiveness. Specifically applying \cite[Lemma B.1.]{taheri2023generalization} to our setting, gives $\forall \thetab,\, \thetab'$: 
\begin{align}\label{eq:stab1}
    \norm{(\thetab-\eta\nabla \Lhat(\thetab)) - (\thetab'-\eta\nabla \Lhat(\thetab'))} \leq \max_{\alpha\in[0,1]}\Big\{\Big(1+\frac{\eta \beta_3(\thetab_\alpha)}{\sqrt{H}} \Lhat (\thetab_\alpha) \Big) \, \maxi \, \eta \beta_2(\thetab_\alpha)  \Big\} \norm{\thetab-\thetab'} \, ,
\end{align}
where, $\thetab_\alpha = \alpha \thetab+ (1-\alpha)\thetab'$, $\alpha\in[0,1]$. Using this and gradient self-boundedness from Corollary \ref{coro:objective gradient/hessian_mainbody}, we get:
\begin{align}\label{eq:expansive at k}
\left\|\thetab_{k+1}-\thetab_{k+1}^{\neg i}\right\| \leq \max_{\alpha\in[0,1]} \Big\{ \Big(1 + \frac{\eta \beta_3(\thetab_{k_\alpha}^{\neg i})}{\sqrt{H}} \Lhat^{\neg i} (\thetab_{k_\alpha}^{\neg i}) \Big) \, \maxi \, \eta \beta_2(\thetab_{k_\alpha}^{\neg i}) \Big\}\cdot \left\|\thetab_{k}-\thetab_{k}^{\neg i}\right\|+\frac{\eta \beta_1(\thetab_k)}{n} \ell_i\left(\thetab_{k}\right),
\end{align}
where $\thetab_{k_\alpha}^{\neg i}: = \alpha \thetab_k + (1-\alpha)\thetab_k^{\neg i}$ for $\alpha\in[0,1]$. 
Further using the bounded iterates-norm property from the training analysis, we can control $\beta_2(\thetab_{k_\alpha}^{\neg i}) \leq \tilde\beta_2(\thetab)$ and $\beta_3(\thetab_{k_\alpha}^{\neg i}) \leq \tilde\beta_3(\thetab)$ making them independent of $k$ (See Lemma \ref{lem:gen_stab} for the  definitions of $\tilde\beta_2(\cdot), \tilde\beta_3(\cdot)$). In order to invoke the Descent Lemma, we set the step-size same as in the training analysis. Thus, \eqref{eq:expansive at k} becomes:
\begin{align}\label{eq:above}
        \left\|\thetab_{k+1}-\thetab_{k+1}^{\neg i}\right\| \le \Big((1+\frac{\eta\tilde\beta_3(\thetab)}{\sqrt{H}})\max_{\alpha\in[0,1]} \Lhat^{\neg i} (\thetab_{k_\alpha}^{\neg i}) \Big) \left\|\thetab_{k}-\thetab_{k}^{\neg i}\right\|+\frac{\eta \beta_1(\thetab_k)}{n} \ell_i\left(\thetab_{k}\right) \, .
\end{align}
As in the training analysis, we can control the loo empirical loss $\Lhat^{\neg i}$ for any point on the line $[\thetab_k, \thetab_k^{\neg i}]$ of two sufficiently close points satisfying $\sqrt{H} \geq 2 \, \left(\beta_3(\thetab_k) \maxi \beta_3(\thetab_k^{\neg i}) \right) \, \|\thetab_k-\thetab_k^{\neg i}\|^2$. Using Prop. \ref{prop:GLQCapp}, Eq. \eqref{eq:above} becomes
\begin{align}\label{eq:stability simplify}
    \left\|\thetab_{k+1}-\thetab_{k+1}^{\neg i}\right\| \leq (1+\alpha_{k,i})\left\|\thetab_{k}-\thetab_{k}^{\neg i}\right\|+\frac{\eta\tilde\beta_1(\thetab)}{n}\ell_i\left(\thetab_{k}\right),
\end{align}
where $\alpha_{k,i}:=\frac{4\eta\tilde\beta_3(\thetab)}{3\sqrt{H}} \left(\Lhat^{\neg i}(\thetab_k) + \Lhat^{\neg i} (\thetab_k^{\neg i})\right)$ and $\beta_1(\thetab_k) \leq \tilde\beta_1(\thetab)$ similar to $\beta_2(\cdot), \beta_3(\cdot)$ using bounded iterates-norm. Unrolling the iterates in \eqref{eq:stability simplify}, summing over $i \in [n]$ and using training regret bounds, we have the following average model stability bound for any iterate $K$:
$
    \frac{1}{n} \sum_{i=1}^n \left\|\thetab_{K}-\thetab_{K}^{\neg i}\right\|\le \frac{2 \eta \tilde\beta_1(\thetab)}{n} \left(2 \, K \Lhat (\thetab) + \frac{9\|\thetab-\thetab_0\|^2}{4\eta}\right) \, ,
$
Combining this with an application of Lemma \ref{lem:on-average} for our objective, which is $G \leq \tilde\beta_1(\thetab)$-Lipschitz from Corollary \ref{coro:objective gradient/hessian_mainbody}, and using $\eta \leq \frac{1}{\rho(\thetab)} \leq \frac{1}{(\tilde\beta_1(\thetab))^2}$, we get the desired generalization gap stated in Thm. \ref{thm:gen}.


\section{Concluding remarks} \label{sec: conclusion}
We studied convergence and generalization of GD for training a multi-head attention layer in a classification task. 
Our training and generalization bounds hold under an appropriate realizability condition asking for the existence of an a target model $\thetabt$ achieving good train loss while being sufficiently close to initialization. In particular, from the condition on the number of heads $H$ in \eqref{eq: order number of heads train}, we need $\thetabt$ at most 
$\ordert{d^{-2/3}T^{-1/6}R^{-5/3}H^{1/6}}$
far from initialization (provided $\maxnorm{\thetabt}=\order{1}$). In Sec. \ref{sec:realizability} we showed that such a model exists if the initialization is chosen appropriately. Specifically it suffices that $\maxnorm{\thetabt^{(0)}}=\order{1}$,  the model output at initialization is $\ordert{1}$-bounded  and that the data are linearly separable with margin $\gamma$ with respect to the NTK features of the model at initialization. Then, 
$\order{d^4TR^{10}\operatorname{polylog}(n)/\gamma^6}$
number of heads guarantee that $\Theta(n)$ GD steps result in train and test loss bounds $\ordert{1/({\eta\gamma^2 n})}$. 
In Sec. \ref{sec:main dm1} we applied our results to a tokenized-mixture model. We showed that after one randomized gradient step from $\mathbf{0}$, the model satisfies the above conditions for \good intialization. For this initialization, we computed the NTK margin $\gamma_\star$ which in turn governs the guaranteed rate of convergence and generalization based on our general bounds. This opens several interesting questions for future work.

First, does random initialization of attention weights satisfy NTK separability, and if so, what is the corresponding margin? Second, are there other initialization strategies that guarantee the realizability conditions are satisfied? Here, note that our conditions for \good initialization are only shown to be sufficient for realizability leaving room for improvements. 
Third,  how suboptimal is the best NTK margin (among other potential natural initializations) compared to the model's global margin $\arg\max_{\|\thetabt\|=\sqrt{2}}\min_{i\in[n]} y_i\Phit(\X_i;\thetabt)$? 
In Proposition \ref{propo:single head good theta} 
 we showed for the data model \ref{model1} that there exists single-head attention model $\thetab_\text{opt}=(\Ub_\text{opt},\W_\text{opt})$ with $\|\thetab_\text{opt}\|=\sqrt{2}$ such that $y\Phi(\X;\Ub_\text{opt},\Gamma_\eps\cdot\W_\text{opt})=\frac{S\sqrt{T}}{\sqrt{2}}\left((1-\eps)-2\eps Z_\mu/S\right)$ for all $\Gamma_\eps\gtrsim \frac{\log((\zeta^{-1}-1)/\eps)}{S}$  and any $\eps\in(0,1)$ (see App. \ref{app:F}). In particular, as $\eps\rightarrow 0$ and $\Gamma_\eps \rightarrow \infty$ (for which the softmax map gets saturated and attends to tokens with highest relevance score) the achieved margin approaches $\gamma_\text{attn}:=S\sqrt{T}/\sqrt{2}$, which is independent of the sparsity level $\zeta$. In particular, $\gamma_\text{attn}\geq \gammastar \geq \gammalin$ and the gap increases with decreasing sparsity. Is it possible to establish finite-time convergence bounds to models with margin $\approx \gamma_\text{attn}$ under appropriate initialization?  How is the answer affected by the fact that the optimal attention weights in this case are diverging in norm ($\Gamma_\eps\rightarrow\infty$)? Using our approach, we argued in Sec. \ref{sec: Is the NTK margin optimal?} that the key challenge is the saturation of norm of $\W_\text{opt}$ ($\Gamma_\eps$), which does not allow the appropriate realizability condition to hold (at least for $\mathbf{0}$ initialization). Finally, it is interesting to consider other data models for which multiple heads are necessary  to interpolate the data.

\section{Acknowledgements}

This work is supported by NSERC Discovery Grant RGPIN-2021-03677, NSF Grant CCF-2009030, and a CIFAR AI Catalyst grant. The authors also acknowledge use of the Sockeye cluster by UBC Advanced Research Computing and thank Paul S. Lintilhac for identifying a miscalculation in the Hessian upper bound in the initial version of the paper. PD thanks Bhavya Vasudeva for the helpful discussions.
 









\bibliographystyle{tmlr}
\bibliography{refs,transformers}

\begin{thebibliography}{69}
\providecommand{\natexlab}[1]{#1}
\providecommand{\url}[1]{\texttt{#1}}
\expandafter\ifx\csname urlstyle\endcsname\relax
  \providecommand{\doi}[1]{doi: #1}\else
  \providecommand{\doi}{doi: \begingroup \urlstyle{rm}\Url}\fi

\bibitem[Aky{\"u}rek et~al.(2023)Aky{\"u}rek, Schuurmans, Andreas, Ma, and Zhou]{akyrek2023what}
Ekin Aky{\"u}rek, Dale Schuurmans, Jacob Andreas, Tengyu Ma, and Denny Zhou.
\newblock What learning algorithm is in-context learning? investigations with linear models.
\newblock In \emph{The Eleventh International Conference on Learning Representations}, 2023.
\newblock URL \url{https://openreview.net/forum?id=0g0X4H8yN4I}.

\bibitem[Allen-Zhu et~al.(2019)Allen-Zhu, Li, and Song]{allen2019convergence}
Zeyuan Allen-Zhu, Yuanzhi Li, and Zhao Song.
\newblock A convergence theory for deep learning via over-parameterization.
\newblock In \emph{International Conference on Machine Learning}, pp.\  242--252. PMLR, 2019.

\bibitem[Arora et~al.(2019)Arora, Du, Hu, Li, and Wang]{arora2019fine}
Sanjeev Arora, Simon Du, Wei Hu, Zhiyuan Li, and Ruosong Wang.
\newblock Fine-grained analysis of optimization and generalization for overparameterized two-layer neural networks.
\newblock In \emph{International Conference on Machine Learning}, pp.\  322--332. PMLR, 2019.

\bibitem[Baldi \& Vershynin(2022)Baldi and Vershynin]{baldi2022quarks}
Pierre Baldi and Roman Vershynin.
\newblock The quarks of attention.
\newblock \emph{arXiv preprint arXiv:2202.08371}, 2022.

\bibitem[Banerjee et~al.(2022)Banerjee, Cisneros-Velarde, Zhu, and Belkin]{banerjee2022restricted}
Arindam Banerjee, Pedro Cisneros-Velarde, Libin Zhu, and Mikhail Belkin.
\newblock Restricted strong convexity of deep learning models with smooth activations.
\newblock \emph{arXiv preprint arXiv:2209.15106}, 2022.

\bibitem[Bietti et~al.(2023)Bietti, Cabannes, Bouchacourt, Jegou, and Bottou]{bietti2023birth}
Alberto Bietti, Vivien Cabannes, Diane Bouchacourt, Herve Jegou, and Leon Bottou.
\newblock Birth of a transformer: A memory viewpoint.
\newblock \emph{arXiv preprint arXiv:2306.00802}, 2023.

\bibitem[Brown et~al.(2020)Brown, Mann, Ryder, Subbiah, Kaplan, Dhariwal, Neelakantan, Shyam, Sastry, Askell, Agarwal, Herbert-Voss, Krueger, Henighan, Child, Ramesh, Ziegler, Wu, Winter, Hesse, Chen, Sigler, Litwin, Gray, Chess, Clark, Berner, McCandlish, Radford, Sutskever, and Amodei]{fewshotlearners}
Tom Brown, Benjamin Mann, Nick Ryder, Melanie Subbiah, Jared~D Kaplan, Prafulla Dhariwal, Arvind Neelakantan, Pranav Shyam, Girish Sastry, Amanda Askell, Sandhini Agarwal, Ariel Herbert-Voss, Gretchen Krueger, Tom Henighan, Rewon Child, Aditya Ramesh, Daniel Ziegler, Jeffrey Wu, Clemens Winter, Chris Hesse, Mark Chen, Eric Sigler, Mateusz Litwin, Scott Gray, Benjamin Chess, Jack Clark, Christopher Berner, Sam McCandlish, Alec Radford, Ilya Sutskever, and Dario Amodei.
\newblock Language models are few-shot learners.
\newblock In H.~Larochelle, M.~Ranzato, R.~Hadsell, M.F. Balcan, and H.~Lin (eds.), \emph{Advances in Neural Information Processing Systems}, volume~33, pp.\  1877--1901. Curran Associates, Inc., 2020.
\newblock URL \url{https://proceedings.neurips.cc/paper_files/paper/2020/file/1457c0d6bfcb4967418bfb8ac142f64a-Paper.pdf}.

\bibitem[Cao \& Gu(2019)Cao and Gu]{cao2019generalization}
Yuan Cao and Quanquan Gu.
\newblock Generalization bounds of stochastic gradient descent for wide and deep neural networks.
\newblock \emph{Advances in neural information processing systems}, 32, 2019.

\bibitem[Chen et~al.(2020)Chen, Cao, Zou, and Gu]{chen2020much}
Zixiang Chen, Yuan Cao, Difan Zou, and Quanquan Gu.
\newblock How much over-parameterization is sufficient to learn deep relu networks?
\newblock In \emph{International Conference on Learning Representations}, 2020.

\bibitem[Cheng et~al.(2016)Cheng, Dong, and Lapata]{cheng-etal-2016-long}
Jianpeng Cheng, Li~Dong, and Mirella Lapata.
\newblock Long short-term memory-networks for machine reading.
\newblock In \emph{Proceedings of the 2016 Conference on Empirical Methods in Natural Language Processing}, pp.\  551--561. Association for Computational Linguistics, November 2016.
\newblock \doi{10.18653/v1/D16-1053}.
\newblock URL \url{https://aclanthology.org/D16-1053}.

\bibitem[Devlin et~al.(2019)Devlin, Chang, Lee, and Toutanova]{bert}
Jacob Devlin, Ming-Wei Chang, Kenton Lee, and Kristina Toutanova.
\newblock {BERT}: Pre-training of deep bidirectional transformers for language understanding.
\newblock pp.\  4171--4186, Minneapolis, Minnesota, June 2019. Association for Computational Linguistics.
\newblock \doi{10.18653/v1/N19-1423}.
\newblock URL \url{https://aclanthology.org/N19-1423}.

\bibitem[Dong et~al.(2021)Dong, Cordonnier, and Loukas]{dong2021attention}
Yihe Dong, Jean-Baptiste Cordonnier, and Andreas Loukas.
\newblock Attention is not all you need: Pure attention loses rank doubly exponentially with depth.
\newblock In \emph{International Conference on Machine Learning}, pp.\  2793--2803. PMLR, 2021.

\bibitem[Dosovitskiy et~al.(2021)Dosovitskiy, Beyer, Kolesnikov, Weissenborn, Zhai, Unterthiner, Dehghani, Minderer, Heigold, Gelly, Uszkoreit, and Houlsby]{dosovitskiy2021image}
Alexey Dosovitskiy, Lucas Beyer, Alexander Kolesnikov, Dirk Weissenborn, Xiaohua Zhai, Thomas Unterthiner, Mostafa Dehghani, Matthias Minderer, Georg Heigold, Sylvain Gelly, Jakob Uszkoreit, and Neil Houlsby.
\newblock An image is worth 16x16 words: Transformers for image recognition at scale, 2021.

\bibitem[Du et~al.(2019)Du, Lee, Li, Wang, and Zhai]{du2019gradient}
Simon Du, Jason Lee, Haochuan Li, Liwei Wang, and Xiyu Zhai.
\newblock Gradient descent finds global minima of deep neural networks.
\newblock In \emph{International conference on machine learning}, pp.\  1675--1685. PMLR, 2019.

\bibitem[Edelman et~al.(2021)Edelman, Goel, Kakade, and Zhang]{edelman2021inductive}
Benjamin~L Edelman, Surbhi Goel, Sham Kakade, and Cyril Zhang.
\newblock Inductive biases and variable creation in self-attention mechanisms.
\newblock \emph{arXiv preprint arXiv:2110.10090}, 2021.

\bibitem[Ergen et~al.(2022)Ergen, Neyshabur, and Mehta]{ergen2022convexifying}
Tolga Ergen, Behnam Neyshabur, and Harsh Mehta.
\newblock Convexifying transformers: Improving optimization and understanding of transformer networks.
\newblock \emph{arXiv preprint arXiv:2211.11052}, 2022.

\bibitem[Hron et~al.(2020)Hron, Bahri, Sohl-Dickstein, and Novak]{hron2020infinite}
Jiri Hron, Yasaman Bahri, Jascha Sohl-Dickstein, and Roman Novak.
\newblock Infinite attention: Nngp and ntk for deep attention networks.
\newblock In \emph{International Conference on Machine Learning}, pp.\  4376--4386. PMLR, 2020.

\bibitem[Jacot et~al.(2018)Jacot, Gabriel, and Hongler]{jacot2018neural}
Arthur Jacot, Franck Gabriel, and Cl{\'e}ment Hongler.
\newblock Neural tangent kernel: Convergence and generalization in neural networks.
\newblock \emph{Advances in neural information processing systems}, 31, 2018.

\bibitem[Jelassi et~al.(2022)Jelassi, Sander, and Li]{jelassi2022vision}
Samy Jelassi, Michael~Eli Sander, and Yuanzhi Li.
\newblock Vision transformers provably learn spatial structure.
\newblock In Alice~H. Oh, Alekh Agarwal, Danielle Belgrave, and Kyunghyun Cho (eds.), \emph{Advances in Neural Information Processing Systems}, 2022.
\newblock URL \url{https://openreview.net/forum?id=eMW9AkXaREI}.

\bibitem[Ji \& Telgarsky(2018)Ji and Telgarsky]{ji2018risk}
Ziwei Ji and Matus Telgarsky.
\newblock Risk and parameter convergence of logistic regression.
\newblock \emph{arXiv preprint arXiv:1803.07300}, 2018.

\bibitem[Ji \& Telgarsky(2020)Ji and Telgarsky]{Ji2020Polylogarithmic}
Ziwei Ji and Matus Telgarsky.
\newblock Polylogarithmic width suffices for gradient descent to achieve arbitrarily small test error with shallow relu networks.
\newblock In \emph{International Conference on Learning Representations}, 2020.

\bibitem[Ji \& Telgarsky(2021)Ji and Telgarsky]{ji2021characterizing}
Ziwei Ji and Matus Telgarsky.
\newblock Characterizing the implicit bias via a primal-dual analysis.
\newblock In \emph{Algorithmic Learning Theory}, pp.\  772--804. PMLR, 2021.

\bibitem[Lei \& Ying(2020)Lei and Ying]{lei2020fine}
Yunwen Lei and Yiming Ying.
\newblock Fine-grained analysis of stability and generalization for stochastic gradient descent.
\newblock In \emph{International Conference on Machine Learning}, pp.\  5809--5819. PMLR, 2020.

\bibitem[Lei et~al.(2022)Lei, Jin, and Ying]{leistabilitynn}
Yunwen Lei, Rong Jin, and Yiming Ying.
\newblock Stability and generalization analysis of gradient methods for shallow neural networks.
\newblock In \emph{Advances in Neural Information Processing Systems}, 2022.

\bibitem[Li et~al.(2023{\natexlab{a}})Li, Weng, Liu, and Chen]{li2023theoretical}
Hongkang Li, Meng Weng, Sijia Liu, and Pin-Yu Chen.
\newblock A theoretical understanding of shallow vision transformers: Learning, generalization, and sample complexity.
\newblock In \emph{International Conference on Learning Representations}, 2023{\natexlab{a}}.

\bibitem[Li et~al.(2023{\natexlab{b}})Li, Ildiz, Papailiopoulos, and Oymak]{li2023transformers}
Yingcong Li, M.~Emrullah Ildiz, Dimitris Papailiopoulos, and Samet Oymak.
\newblock Transformers as algorithms: Generalization and stability in in-context learning, 2023{\natexlab{b}}.

\bibitem[Likhosherstov et~al.(2021)Likhosherstov, Choromanski, and Weller]{likhosherstov2021expressive}
Valerii Likhosherstov, Krzysztof Choromanski, and Adrian Weller.
\newblock On the expressive power of self-attention matrices, 2021.

\bibitem[Lin et~al.(2017)Lin, Feng, dos Santos, Yu, Xiang, Zhou, and Bengio]{lin2017structured}
Zhouhan Lin, Minwei Feng, Cicero~Nogueira dos Santos, Mo~Yu, Bing Xiang, Bowen Zhou, and Yoshua Bengio.
\newblock A structured self-attentive sentence embedding.
\newblock In \emph{International Conference on Learning Representations}, 2017.
\newblock URL \url{https://openreview.net/forum?id=BJC_jUqxe}.

\bibitem[Liu et~al.(2020)Liu, Zhu, and Belkin]{liu2020linearity}
Chaoyue Liu, Libin Zhu, and Misha Belkin.
\newblock On the linearity of large non-linear models: when and why the tangent kernel is constant.
\newblock \emph{Advances in Neural Information Processing Systems}, 33:\penalty0 15954--15964, 2020.

\bibitem[Liu et~al.(2019)Liu, Ott, Goyal, Du, Joshi, Chen, Levy, Lewis, Zettlemoyer, and Stoyanov]{roberta}
Yinhan Liu, Myle Ott, Naman Goyal, Jingfei Du, Mandar Joshi, Danqi Chen, Omer Levy, Mike Lewis, Luke Zettlemoyer, and Veselin Stoyanov.
\newblock Roberta: A robustly optimized bert pretraining approach, 2019.

\bibitem[Loshchilov \& Hutter(2019)Loshchilov and Hutter]{adamW}
Ilya Loshchilov and Frank Hutter.
\newblock Decoupled weight decay regularization, 2019.

\bibitem[Lu et~al.(2021)Lu, Mao, and Nayak]{diff-eqn-attn}
Haoye Lu, Yongyi Mao, and Amiya Nayak.
\newblock On the dynamics of training attention models.
\newblock In \emph{International Conference on Learning Representations}, 2021.
\newblock URL \url{https://openreview.net/forum?id=1OCTOShAmqB}.

\bibitem[Mahdavi et~al.(2023)Mahdavi, Liao, and Thrampoulidis]{mahdavi2023memorization}
Sadegh Mahdavi, Renjie Liao, and Christos Thrampoulidis.
\newblock Memorization capacity of multi-head attention in transformers.
\newblock \emph{arXiv preprint arXiv:2306.02010}, 2023.

\bibitem[Nguyen et~al.(2021)Nguyen, Mondelli, and Montufar]{nguyensmallest21}
Quynh Nguyen, Marco Mondelli, and Guido~F Montufar.
\newblock Tight bounds on the smallest eigenvalue of the neural tangent kernel for deep relu networks.
\newblock In Marina Meila and Tong Zhang (eds.), \emph{Proceedings of the 38th International Conference on Machine Learning}, volume 139 of \emph{Proceedings of Machine Learning Research}, pp.\  8119--8129. PMLR, 18--24 Jul 2021.

\bibitem[Nguyen \& Mondelli(2020)Nguyen and Mondelli]{nguyen2020global}
Quynh~N Nguyen and Marco Mondelli.
\newblock Global convergence of deep networks with one wide layer followed by pyramidal topology.
\newblock \emph{Advances in Neural Information Processing Systems}, 33:\penalty0 11961--11972, 2020.

\bibitem[Nikolakakis et~al.(2022)Nikolakakis, Haddadpour, Karbasi, and Kalogerias]{nikolakakis2022beyond}
Konstantinos~E Nikolakakis, Farzin Haddadpour, Amin Karbasi, and Dionysios~S Kalogerias.
\newblock Beyond {L}ipschitz: Sharp generalization and excess risk bounds for full-batch gd.
\newblock \emph{arXiv preprint arXiv:2204.12446}, 2022.

\bibitem[Nitanda et~al.(2019)Nitanda, Chinot, and Suzuki]{nitanda2019gradient}
Atsushi Nitanda, Geoffrey Chinot, and Taiji Suzuki.
\newblock Gradient descent can learn less over-parameterized two-layer neural networks on classification problems.
\newblock \emph{arXiv preprint arXiv:1905.09870}, 2019.

\bibitem[OpenAI(2022)]{chatgpt}
OpenAI.
\newblock Openai: Introducing chatgpt, 2022.
\newblock URL \url{https://openai.com/blog/chatgpt, 2022}.

\bibitem[OpenAI(2023)]{gpt4}
OpenAI.
\newblock Gpt-4 technical report, 2023.

\bibitem[Oymak \& Soltanolkotabi(2020)Oymak and Soltanolkotabi]{oymak2020toward}
Samet Oymak and Mahdi Soltanolkotabi.
\newblock Toward moderate overparameterization: Global convergence guarantees for training shallow neural networks.
\newblock \emph{IEEE Journal on Selected Areas in Information Theory}, 1\penalty0 (1):\penalty0 84--105, 2020.

\bibitem[Oymak et~al.(2023)Oymak, Rawat, Soltanolkotabi, and Thrampoulidis]{prompt-attention}
Samet Oymak, Ankit~Singh Rawat, Mahdi Soltanolkotabi, and Christos Thrampoulidis.
\newblock On the role of attention in prompt-tuning.
\newblock In \emph{ICLR 2023 Workshop on Mathematical and Empirical Understanding of Foundation Models}, 2023.

\bibitem[Parikh et~al.(2016)Parikh, T{\"a}ckstr{\"o}m, Das, and Uszkoreit]{parikh-etal-2016-decomposable}
Ankur Parikh, Oscar T{\"a}ckstr{\"o}m, Dipanjan Das, and Jakob Uszkoreit.
\newblock A decomposable attention model for natural language inference.
\newblock In \emph{Proceedings of the 2016 Conference on Empirical Methods in Natural Language Processing}, pp.\  2249--2255, Austin, Texas, November 2016. Association for Computational Linguistics.
\newblock \doi{10.18653/v1/D16-1244}.
\newblock URL \url{https://aclanthology.org/D16-1244}.

\bibitem[Radford et~al.(2021)Radford, Kim, Hallacy, Ramesh, Goh, Agarwal, Sastry, Askell, Mishkin, Clark, Krueger, and Sutskever]{radford21visualtransfer}
Alec Radford, Jong~Wook Kim, Chris Hallacy, Aditya Ramesh, Gabriel Goh, Sandhini Agarwal, Girish Sastry, Amanda Askell, Pamela Mishkin, Jack Clark, Gretchen Krueger, and Ilya Sutskever.
\newblock Learning transferable visual models from natural language supervision.
\newblock In Marina Meila and Tong Zhang (eds.), \emph{Proceedings of the 38th International Conference on Machine Learning}, volume 139 of \emph{Proceedings of Machine Learning Research}, pp.\  8748--8763. PMLR, 18--24 Jul 2021.
\newblock URL \url{https://proceedings.mlr.press/v139/radford21a.html}.

\bibitem[Raffel et~al.(2020)Raffel, Shazeer, Roberts, Lee, Narang, Matena, Zhou, Li, and Liu]{raffel2020transferlearning}
Colin Raffel, Noam Shazeer, Adam Roberts, Katherine Lee, Sharan Narang, Michael Matena, Yanqi Zhou, Wei Li, and Peter~J. Liu.
\newblock Exploring the limits of transfer learning with a unified text-to-text transformer.
\newblock 21\penalty0 (1), 2020.
\newblock ISSN 1532-4435.

\bibitem[Richards \& Kuzborskij(2021)Richards and Kuzborskij]{richards2021stability}
Dominic Richards and Ilja Kuzborskij.
\newblock Stability \& generalisation of gradient descent for shallow neural networks without the neural tangent kernel.
\newblock \emph{Advances in Neural Information Processing Systems}, 34:\penalty0 8609--8621, 2021.

\bibitem[Richards \& Rabbat(2021)Richards and Rabbat]{richards2021learning}
Dominic Richards and Mike Rabbat.
\newblock Learning with gradient descent and weakly convex losses.
\newblock In \emph{International Conference on Artificial Intelligence and Statistics}, pp.\  1990--1998. PMLR, 2021.

\bibitem[Safran et~al.(2021)Safran, Yehudai, and Shamir]{student-teacher-shamir}
Itay~M Safran, Gilad Yehudai, and Ohad Shamir.
\newblock The effects of mild over-parameterization on the optimization landscape of shallow relu neural networks.
\newblock In Mikhail Belkin and Samory Kpotufe (eds.), \emph{Proceedings of Thirty Fourth Conference on Learning Theory}, volume 134 of \emph{Proceedings of Machine Learning Research}, pp.\  3889--3934. PMLR, 15--19 Aug 2021.

\bibitem[Sahiner et~al.(2022)Sahiner, Ergen, Ozturkler, Pauly, Mardani, and Pilanci]{sahiner2022unraveling}
Arda Sahiner, Tolga Ergen, Batu Ozturkler, John Pauly, Morteza Mardani, and Mert Pilanci.
\newblock Unraveling attention via convex duality: Analysis and interpretations of vision transformers.
\newblock \emph{International Conference on Machine Learning}, 2022.

\bibitem[Sanford et~al.(2023)Sanford, Hsu, and Telgarsky]{sanford2023representational}
Clayton Sanford, Daniel Hsu, and Matus Telgarsky.
\newblock Representational strengths and limitations of transformers.
\newblock \emph{arXiv preprint arXiv:2306.02896}, 2023.

\bibitem[Schliserman \& Koren(2022)Schliserman and Koren]{pmlr-v178-schliserman22a}
Matan Schliserman and Tomer Koren.
\newblock Stability vs implicit bias of gradient methods on separable data and beyond.
\newblock In Po-Ling Loh and Maxim Raginsky (eds.), \emph{Proceedings of Thirty Fifth Conference on Learning Theory}, volume 178 of \emph{Proceedings of Machine Learning Research}, pp.\  3380--3394. PMLR, 02--05 Jul 2022.

\bibitem[Shamir(2021)]{shamir2021gradient}
Ohad Shamir.
\newblock Gradient methods never overfit on separable data.
\newblock \emph{Journal of Machine Learning Research}, 22\penalty0 (85):\penalty0 1--20, 2021.

\bibitem[Socher et~al.(2013)Socher, Perelygin, Wu, Chuang, Manning, Ng, and Potts]{sst2}
Richard Socher, Alex Perelygin, Jean Wu, Jason Chuang, Christopher~D. Manning, Andrew Ng, and Christopher Potts.
\newblock Recursive deep models for semantic compositionality over a sentiment treebank.
\newblock In \emph{Proceedings of the 2013 Conference on Empirical Methods in Natural Language Processing}, pp.\  1631--1642, Seattle, Washington, USA, October 2013. Association for Computational Linguistics.

\bibitem[Soudry et~al.(2018)Soudry, Hoffer, Nacson, Gunasekar, and Srebro]{soudry2018implicit}
Daniel Soudry, Elad Hoffer, Mor~Shpigel Nacson, Suriya Gunasekar, and Nathan Srebro.
\newblock The implicit bias of gradient descent on separable data.
\newblock \emph{The Journal of Machine Learning Research}, 19\penalty0 (1):\penalty0 2822--2878, 2018.

\bibitem[Taheri \& Thrampoulidis(2023)Taheri and Thrampoulidis]{taheri2023generalization}
Hossein Taheri and Christos Thrampoulidis.
\newblock Generalization and stability of interpolating neural networks with minimal width.
\newblock \emph{arXiv preprint arXiv:2302.09235}, 2023.

\bibitem[Tarzanagh et~al.(2023{\natexlab{a}})Tarzanagh, Li, Thrampoulidis, and Oymak]{tarzanagh2023transformers}
Davoud~Ataee Tarzanagh, Yingcong Li, Christos Thrampoulidis, and Samet Oymak.
\newblock Transformers as support vector machines, 2023{\natexlab{a}}.

\bibitem[Tarzanagh et~al.(2023{\natexlab{b}})Tarzanagh, Li, Zhang, and Oymak]{tarzanagh2023maxmargin}
Davoud~Ataee Tarzanagh, Yingcong Li, Xuechen Zhang, and Samet Oymak.
\newblock Max-margin token selection in attention mechanism, 2023{\natexlab{b}}.

\bibitem[Telgarsky(2013)]{telgarsky2013margins}
Matus Telgarsky.
\newblock Margins, shrinkage, and boosting.
\newblock In \emph{International Conference on Machine Learning}, pp.\  307--315. PMLR, 2013.

\bibitem[Telgarsky(2022)]{telgarsky2022feature}
Matus Telgarsky.
\newblock Feature selection and low test error in shallow low-rotation relu networks.
\newblock In \emph{The Eleventh International Conference on Learning Representations}, 2022.

\bibitem[Tian et~al.(2023)Tian, Wang, Chen, and Du]{tian2023scan}
Yuandong Tian, Yiping Wang, Beidi Chen, and Simon Du.
\newblock Scan and snap: Understanding training dynamics and token composition in 1-layer transformer, 2023.

\bibitem[Touvron et~al.(2021)Touvron, Cord, Douze, Massa, Sablayrolles, and Jegou]{touvron21distillation}
Hugo Touvron, Matthieu Cord, Matthijs Douze, Francisco Massa, Alexandre Sablayrolles, and Herve Jegou.
\newblock Training data-efficient image transformers \& distillation through attention.
\newblock In Marina Meila and Tong Zhang (eds.), \emph{Proceedings of the 38th International Conference on Machine Learning}, volume 139 of \emph{Proceedings of Machine Learning Research}, pp.\  10347--10357. PMLR, 18--24 Jul 2021.
\newblock URL \url{https://proceedings.mlr.press/v139/touvron21a.html}.

\bibitem[Touvron et~al.(2023)Touvron, Lavril, Izacard, Martinet, Lachaux, Lacroix, Rozière, Goyal, Hambro, Azhar, Rodriguez, Joulin, Grave, and Lample]{touvron2023llama}
Hugo Touvron, Thibaut Lavril, Gautier Izacard, Xavier Martinet, Marie-Anne Lachaux, Timothée Lacroix, Baptiste Rozière, Naman Goyal, Eric Hambro, Faisal Azhar, Aurelien Rodriguez, Armand Joulin, Edouard Grave, and Guillaume Lample.
\newblock Llama: Open and efficient foundation language models, 2023.

\bibitem[Vaswani et~al.(2017)Vaswani, Shazeer, Parmar, Uszkoreit, Jones, Gomez, Kaiser, and Polosukhin]{Vaswani2017AttentionIA}
Ashish Vaswani, Noam~M. Shazeer, Niki Parmar, Jakob Uszkoreit, Llion Jones, Aidan~N. Gomez, Lukasz Kaiser, and Illia Polosukhin.
\newblock Attention is all you need.
\newblock In \emph{NIPS}, 2017.

\bibitem[von Oswald et~al.(2022)von Oswald, Niklasson, Randazzo, Sacramento, Mordvintsev, Zhmoginov, and Vladymyrov]{Oswald2022TransformersLI}
Johannes von Oswald, Eyvind Niklasson, E.~Randazzo, Jo{\~a}o Sacramento, Alexander Mordvintsev, Andrey Zhmoginov, and Max Vladymyrov.
\newblock Transformers learn in-context by gradient descent.
\newblock \emph{ArXiv}, abs/2212.07677, 2022.

\bibitem[Xu \& Du(2023)Xu and Du]{simondu-gd-slow}
Weihang Xu and Simon Du.
\newblock Over-parameterization exponentially slows down gradient descent for learning a single neuron.
\newblock In Gergely Neu and Lorenzo Rosasco (eds.), \emph{Proceedings of Thirty Sixth Conference on Learning Theory}, volume 195 of \emph{Proceedings of Machine Learning Research}, pp.\  1155--1198. PMLR, 2023.

\bibitem[Yun et~al.(2020{\natexlab{a}})Yun, Bhojanapalli, Rawat, Reddi, and Kumar]{yun2020transformers}
Chulhee Yun, Srinadh Bhojanapalli, Ankit~Singh Rawat, Sashank~J. Reddi, and Sanjiv Kumar.
\newblock Are transformers universal approximators of sequence-to-sequence functions?, 2020{\natexlab{a}}.

\bibitem[Yun et~al.(2020{\natexlab{b}})Yun, Chang, Bhojanapalli, Rawat, Reddi, and Kumar]{yun2020on}
Chulhee Yun, Yin-Wen Chang, Srinadh Bhojanapalli, Ankit~Singh Rawat, Sashank~J. Reddi, and Sanjiv Kumar.
\newblock $o(n)$ connections are expressive enough: Universal approximability of sparse transformers, 2020{\natexlab{b}}.

\bibitem[Zhang et~al.(2023)Zhang, Frei, and Bartlett]{zhang2023trained}
Ruiqi Zhang, Spencer Frei, and Peter~L. Bartlett.
\newblock Trained transformers learn linear models in-context, 2023.

\bibitem[Zhou et~al.(2021)Zhou, Ge, and Jin]{student-teacher-local}
Mo~Zhou, Rong Ge, and Chi Jin.
\newblock A local convergence theory for mildly over-parameterized two-layer neural network.
\newblock In Mikhail Belkin and Samory Kpotufe (eds.), \emph{Proceedings of Thirty Fourth Conference on Learning Theory}, volume 134 of \emph{Proceedings of Machine Learning Research}, pp.\  4577--4632. PMLR, 15--19 Aug 2021.

\bibitem[Zhu et~al.(2023)Zhu, Liu, Chrysos, Locatello, and Cevher]{zhu2023benign}
Zhenyu Zhu, Fanghui Liu, Grigorios Chrysos, Francesco Locatello, and Volkan Cevher.
\newblock Benign overfitting in deep neural networks under lazy training.
\newblock In \emph{International Conference on Machine Learning}, pp.\  43105--43128. PMLR, 2023.

\end{thebibliography}

\newpage
\clearpage
\appendix
\onecolumn
\addtocontents{toc}{\protect\setcounter{tocdepth}{3}}
\tableofcontents



\section{Gradients and Hessian Calculations}

We define the following for convenience
\begin{subequations}
    \begin{align}
        \nn\Lhat'(\thetab) &= \frac{1}{n} \sum_{i \in [n]} \abs{\ell'(y_i \Phi(\X_i, \thetab))}\, , \qquad\text{and}\qquad
        \nn\Lhat''(\thetab) = \frac{1}{n} \sum_{i \in [n]} \abs{\ell''(y_i \Phi(\X_i,\thetab))} \, .
    \end{align}
\end{subequations}
For logistic loss $\ell(z):= \log(1+e^{-z})$:
\begin{align} \label{eq:logisticlossGandH}
    \Lhat'(\thetab) &\leq \Lhat(\thetab) \leq 1 \, , \qquad\text{and}\qquad \Lhat''(\thetab) \leq \frac{1}{4} \, .
\end{align}
We use $\vec{\A}$ and $\text{vec}(\A)$ to denote the the vectorization of a matrix $\A$ and
$\odot$ to denote the Hadamard product.
Finally, we define $\unitvector{n}{i}$ as the $i$-th standard basis vector in $\R^n$.

\subsection{Gradient/Hessian calculations for multihead-attention} \label{app: Gradient/Hessian calculations for multihead-attention}

\begin{lemma} \label{lem:grad and hess singlehead appendix}
Let the softmax attention model $\Phi(\X;\thetab)$ in Eq. \eqref{eq:single-head attn}. Then, for all vectors $\ab\in\R^T$ and $\bb,\cb\in\R^d$ it holds:
\begin{enumerate}
  \item $\begin{aligned}[t]
  \nabla_{\Ub} \Phi(\X; \thetab) = \sft{\X \W \X^\top}\X \, . \end{aligned}$
  \item $ \begin{aligned}[t]
  \nabla_{\W} \Phi(\X; \thetab) = \sum_{t=1}^T \x_{t} \ub_{t}^\top \X^\top \sftd{\X \W^\top \x_t} \X \, .\end{aligned}$
  \item
  $\begin{aligned}[t]
  \nabla_{\W} \inp{\ab}{\nabla_{\Ub} \Phi(\X; \thetab)\, \bb} = \sum_{t=1}^T \x_{t} a_t \bb^\top \X^\top \sftd{\X \W^\top \x_t} \X \, . \end{aligned}$
  \item 
  $ \begin{aligned}[t]
  \nabla_{\W}\inp{\cb}{\nabla_{\W} \Phi(\X; \thetab) \,\bb} = \sum_{t=1}^T \cb^\top \x_t \, \x_t &\Big( \ub_t^\top \X^\top \operatorname{diag}(\X \bb) - \sft{\X \W^\top \x_t}^\top \X \bb \, (\X \ub_t)^\top \\  &\qquad - \sft{\X \W^\top \x_t}^\top \X \ub_t \, (\X \bb)^\top \Big) \, \sftd{\X \W^\top \x_t} \X \, .\end{aligned}$
\end{enumerate}
\end{lemma}

\begin{proof}
For simplicity, we denote $\G := \X\W\X^\top$ and the rows of $\bm{\varphi}(\G)$ as $\bm{\varphi}(\g_t) = \bm{\varphi}(\X \W^\top \x_t)$, $t \in [T]$. Recall, for any $\vb \in \R^d$,
\begin{align} \label{eq: grad softmax}
    \sftd{\db} = \nabla_{\vb} \bm{\varphi}(\vb) = \operatorname{diag}(\bm{\varphi}(\vb)) - \bm{\varphi}(\vb) \, \bm{\varphi}(\vb)^\top \, ,
\end{align}
i.e.,
\begin{align} \label{eq: approx softmax}
    \bm{\varphi}(\vb + \deltab) = \bm{\varphi}(\vb) + \left(\operatorname{diag}(\bm{\varphi}(\vb)) - \bm{\varphi}(\vb) \, \bm{\varphi}(\vb)^\top\right) \deltab + o(\norm{\deltab}^2) \, .
\end{align}
We start with the gradient with respect to $\Ub$,
\begin{align} \label{eq: grad U single PL}
    \nabla_{\Ub} \Phi(\X; \thetab) = \sft{\G}\X \, .
\end{align}
Next step is to compute the gradient with respect to $\W$, $\nabla_{\W} \Phi(\X; \thetab) = \sum_{t=1}^T \nabla_{\W} \left(\ub_t^\top \X^\top \bm{\varphi}(\g_t) \right) \,$. Using Eq. \eqref{eq: approx softmax},
\begin{align*}
    \ub_t^\top \X^\top \bm{\varphi}(\X (\W + \Deltab)^\top \x_t) &= \ub_t^\top \X^\top \left(\bm{\varphi}(\g_t) + \left(\operatorname{diag}(\bm{\varphi}(\g_t)) - \bm{\varphi}(\g_t) \, \bm{\varphi}(\g_t)^\top\right) \X \Deltab^\top \x_t \right) + o(\norm{\Deltab}^2) \nn \\ &= \ub_t^\top \X^\top \bm{\varphi}(\g_t) + \tr(\x_t \ub_t^\top \X^\top \sftd{\X \W^\top \x_t} \X \Deltab^\top) + o(\norm{\Deltab}^2) \, .
\end{align*}
Thus,
\begin{align} \label{eq: grad W PL raw}
    \nabla_{\W} \left(\ub_t^\top \X^\top \bm{\varphi}(\g_t) \right) = \x_t \ub_t^\top \X^\top \sftd{\X \W^\top \x_t} \X
\end{align}
and
\begin{align} \label{eq: gradW PL}
    \nabla_{\W} \Phi(\X; \thetab) = \sum_{t=1}^T \x_t \ub_t^\top \X^\top \sftd{\X \W^\top \x_t} \X \, .
\end{align}
For the third statement of the lemma, we have the following sequence of equalities,
\begin{align}
    \nabla_{\W}\inp{\ab}{\nabla_{\Ub} \Phi(\X; \thetab)\, \bb} &= \nabla_{\W} \inp{\ab}{\bm{\varphi}(\G) \X \bb} = \nabla_{\W} \left( \ab^\top \bm{\varphi}(\G) \, \X \bb \right) \nn \\ &= \nabla_{\W} \tr\left( \bb \ab^\top \bm{\varphi}(\G) \X \right) = \sum_{t=1}^T \nabla_{\W} \left( a_t \bb^\top \X^\top \sftt{t}{\g} \right)  \nn \\ &= \sum_{t=1}^T \x_{t} a_t \bb^\top \X^\top \sftd{\X \W^\top \x_t} \X \, ,
\end{align}
where in the last equality we used Eq. \eqref{eq: grad W PL raw}. \\
For the Hessian with respect to $\W$, we have:
\begin{align} \label{eq: hess W PL}
    \nabla_{\W} \inp{\cb}{\nabla_{\W} \Phi(\X; \thetab) \,\bb} &= \sum_{t=1}^T \nabla_{\W} \Big( \underbrace{\cb^\top \x_t \ub_t^\top \X^\top}_{\pb_{t}^\top} \sftd{\X \W^\top \x_t} \underbrace{\X \bb}_{\qb_t} \Big) \nn \\ &= \sum_{t=1}^T \Big( \underbrace{\nabla_{\W} \pb_{t}^\top \operatorname{diag}(\qb_t) \bm{\varphi}(\g_t)}_{\term{I}} - \underbrace{\nabla_{\W} \pb_{t}^\top \bm{\varphi}(\g_t) \, \qb_t^\top \bm{\varphi}(\g_t)}_{\term{II}} \Big) \, ,
\end{align}
where we used the property $\operatorname{diag}(\ab_1) \, \ab_2 = \operatorname{diag}(\ab_2) \, \ab_1$ for vectors $\ab_1 , \ab_2 \in \R^T$. \\
First, we compute the $\term{I}$ above. Using Eq. \eqref{eq: approx softmax}:
\begin{align}
    \pb_{t}^\top \operatorname{diag}(\qb_t) \bm{\varphi}(\X (\W + \Deltab)^\top \x_t) &= \pb_{t}^\top \operatorname{diag}(\qb_t) \left(\bm{\varphi}(\g_t) + \sftd{\X \W^\top \x_t} \X \Deltab^\top \x_t \right) + o(\norm{\Deltab}^2) \nn \\ &= \pb_{t}^\top \operatorname{diag}(\qb_t) \bm{\varphi}(\g_t) + \tr(\x_t \pb_{t}^\top \operatorname{diag}(\qb_t) \sftd{\X \W^\top \x_t} \X \Deltab^\top) + o(\norm{\Deltab}^2) \, .
\end{align}
Therefore,
\begin{align} \label{eq: hess T1 W PL}
    \nabla_{\W} \pb_{t}^\top \operatorname{diag}(\qb_t) \bm{\varphi}(\g_t) &= \x_t \pb_{t}^\top \operatorname{diag}(\qb_t) \sftd{\X \W^\top \x_t} \X \nn \\ &= \x_t \, \cb^\top \x_t \, \ub_t^\top \X^\top \operatorname{diag}(\X \bb) \sftd{\X \W^\top \x_t} \X \, .
\end{align}
Similarly for $\term{II}$ we have:
\begin{align}
    \pb_{t}^\top \bm{\varphi}(\X (\W + \Deltab)^\top \x_t) \, \qb_t^\top \bm{\varphi}(\X (\W + \Deltab)^\top \x_t) &= \left( \pb_t^\top \bm{\varphi}(\g_t) + \pb_t^\top \sftd{\X \W^\top \x_t} \X \Deltab^\top \x_t + o(\norm{\Deltab}^2) \right) \nn \\ &\quad \left( \qb_t^\top \bm{\varphi}(\g_t) + \qb_t^\top \sftd{\X \W^\top \x_t} \X \Deltab^\top \x_t + o(\norm{\Deltab}^2) \right) \nn \\ &= \pb_t^\top \bm{\varphi}(\g_t) \, \qb_t^\top \bm{\varphi}(\g_t) + \tr\Big(\x_t (\qb^\top \bm{\varphi}(\g_t) \pb^\top \nn \\
    &\quad + \pb^\top \bm{\varphi}(\g_t) \qb^\top) \, \sftd{\X \W^\top \x_t} \X \Deltab^\top\Big) + o(\norm{\Deltab}^2) \, .
\end{align}
Thus,
\begin{align} \label{eq: hess T2 W PL}
    \nabla_{\W} \pb_{t}^\top \bm{\varphi}(\g_t) \, \qb_t^\top \bm{\varphi}(\g_t) &= \x_t (\qb^\top \bm{\varphi}(\g_t) \pb^\top + \pb^\top \bm{\varphi}(\g_t) \qb^\top) \sftd{\X \W^\top \x_t} \X \nn \\ &= \cb^\top \x_t\, \x_t \, \bm{\varphi}(\g_t)^\top \, \left( \X \ub_t \, (\X \bb)^\top + \X \bb \, (\X \ub_t)^\top  \right) \sftd{\X \W^\top \x_t} \X \, .
\end{align}
Combining Eqs. \eqref{eq: hess T1 W PL} and \eqref{eq: hess T2 W PL} and plugging in Eq. \eqref{eq: hess W PL} we conclude that
\begin{align}
    \nabla_{\W} \inp{\cb}{\nabla_{\W} \Phi(\X; \thetab) \,\bb} = \sum_{t=1}^T \cb^\top \x_t \, \x_t &\Big( \ub_t^\top X^\top \operatorname{diag}(\X \bb) - \bm{\varphi}(\g_t)^\top \X \bb \, (\X \ub_t)^\top \nn \\  &\qquad - \bm{\varphi}(\g_t)^\top \X \ub_t \, (\X \bb)^\top \Big) \, \sftd{\X \W^\top \x_t} \X \, .
\end{align}
\end{proof}
We restate Proposition \ref{propo:model bounds general} here for convenience.
\begin{propo} [Restatement of Prop. \ref{propo:model bounds general}] \label{propo:singlehead}
Let the softmax attention model $\Phi(\X;\thetab)$ in Eq. \eqref{eq:single-head attn}. Then, it holds: 

\begin{enumerate}
  \item $\begin{aligned}[t]
  \norm{\nabla_{\Ubv} \Phi(\X; \thetab)} \leq \sqrt{T} \, \norm{\X}_{2, \infty} \, . \end{aligned}$
  \item $\begin{aligned}[t]
  \norm{\nabla_{\Wv} \Phi(\X; \thetab)} \leq 2 \, \norm{\X}^2_{2, \infty} \, \sum_{t=1}^T \, \norm{\X \ub_{t}}_\infty \, . \end{aligned}$
  \item $\begin{aligned}[t]
  \norm{\nabla_{\thetab} \Phi(\X; \thetab)} \leq 2 \, \norm{\X}^2_{2, \infty} \, \sum_{t=1}^T \, \norm{\X \ub_{t}}_\infty + \sqrt{T} \, \norm{\X}_{2, \infty} \, . \end{aligned}$
  \item $ \begin{aligned}[t]
  \norm{\nabla_{\thetab}^2 \Phi(\X; \thetab)} \leq 6 \, d^2 \, \norm{\X}_{2, \infty} \, \norm{\X}_{1, \infty}^3 \, \sum_{t=1}^T \norm{\X \ub_t}_\infty + 2 \, d \, \sqrt{T \, d} \, \norm{\X}_{2, \infty} \, \norm{\X}_{1, \infty}^2 \, .\end{aligned}$
\end{enumerate}
Moreover, if Assumption \ref{ass:features} holds,
\begin{enumerate}
  \item $\begin{aligned}[t]
  \norm{\nabla_{\thetab} \Phi(\X; \thetab)} \leq \sqrt{T} \, R \, \left(2 \, R^2 \, \norm{\Ub}_F + 1 \right) \, . \end{aligned}$
  \item $ \begin{aligned}[t]
  \norm{\nabla_{\thetab}^2 \Phi(\X; \thetab)} \leq 2 \, d \, \sqrt{T \, d} \, R^3 \, \left(3 \, \sqrt{d} \, R^2 \, \norm{\Ub}_F + 1 \right) \, .\end{aligned}$
\end{enumerate}
\end{propo}

\begin{proof}
Using the first statement of Lemma \ref{lem:grad and hess singlehead appendix},
\begin{align} \label{eq:GradVbound}
    \norm{\nabla_{\Ubv} \Phi(\X; \thetab)} &= \sqrt{\sum_{t=1}^T \norm{\X^\top \bm{\varphi}(\g_t)}_2^2} \leq \sqrt{\sum_{t=1}^T \norm{\X^\top}^2_{1,2} \norm{\bm{\varphi}(\g_t)}^2_1} \nn \\ &= \sqrt{T} \, \max_{t \in [T]} \norm{\x_t} = \sqrt{T} \, \norm{\X}_{2, \infty} \, ,
\end{align}
where we used that $\norm{\bm{\varphi}(\g_t)}_1 = 1$. \\
Using Lemma \ref{lem:grad and hess singlehead appendix} for the gradient with respect to $\W$,
\begin{align} \label{eq:GradW}
    \nabla_{\W} \Phi(\X; \thetab) = \sum_{t=1}^T \x_{t} \ub_{t}^\top \X^\top \sftd{\X \W^\top \x_t} \X \, .
\end{align}
Therefore, by applying triangle inequality,
\begin{align*}
    \norm{\nabla_{\Wv} \Phi(\X; \thetab)} \leq \sum_{t=1}^T \norm{\x_{t}} \norm{\X^\top \sftd{\g_t} \X \ub_{t}} \, .
\end{align*}
In the next step we bound $\norm{\X^\top \sftd{\g_t} \X \ub_{t}}$,
\begin{align} \label{eq: bound terms gard w}
    \norm{\X^\top \sftd{\X \W^\top \x_t} \X \ub_{t}} \leq \underbrace{\norm{\X^\top \operatorname{diag}(\bm{\varphi}(\g_t))
    \X \ub_{t}}}_{\term{I}} + \underbrace{\norm{\X^\top  \bm{\varphi}(\g_t) \bm{\varphi}(\g_t)^\top
    \X \ub_{t}}}_{\term{II}} \, .
\end{align}
For $\term{I}$ we do as follows:
\begin{align} \label{eq: grad W term I}
    \norm{\X^\top \operatorname{diag}(\bm{\varphi}(\g_t)) \X \ub_{t}} &= \max_{\vb \in \R^d \text{ and } \norm{\vb} = 1} (\X \vb)^\top \operatorname{diag}(\bm{\varphi}(\g_t)) \X \ub_{t} \notag \\ &= \max_{\norm{\vb} = 1} \sum_{\tau=1}^T [\operatorname{diag}(\bm{\varphi}(\g_t))]_{\tau\tau} \, [\X \ub_{t}]_\tau \, [\X \vb]_\tau \notag \\ &= \max_{\norm{\vb} = 1} \bm{\varphi}(\g_t)^\top \, \left( (\X \ub_{t}) \odot (\X \vb) \right) \notag \\ &\leq \max_{\norm{\vb} = 1} \norm{\bm{\varphi}(\g_t)}_1 \, \norm{(\X \ub_{t}) \odot (\X \vb)}_\infty \notag \\ &\leq  \norm{\X \ub_{t}}_\infty \, \max_{\norm{\vb} = 1} \norm{\X \vb}_\infty \notag \\ &= \norm{\X \ub_{t}}_\infty \, \norm{\X}_{2, \infty} \, ,
\end{align}
where we used Hölder's inequality in the first inequality. \\
Then, we compute $\term{II}$:
\begin{align} \label{eq: grad W term II}
    \norm{\X^\top \bm{\varphi}(\g_t) \bm{\varphi}(\g_t)^\top \X \ub_{t}} &= \max_{\vb \in \R^d \text{ and } \norm{\vb} = 1} (\X \vb)^\top \bm{\varphi}(\g_t) \bm{\varphi}(\g_t)^\top \X \ub_{t} \notag \\ &\leq \max_{\norm{\vb} = 1} \norm{\bm{\varphi}(\g_t)}^2_1 \, \norm{\X \ub_{t}}_\infty \, \norm{\X \vb}_\infty \notag \\ &=  \norm{\X \ub_{t}}_\infty \, \max_{\norm{\vb} = 1} \norm{\X \vb}_\infty \notag \\ &= \norm{\X \ub_{t}}_\infty \, \norm{\X}_{2, \infty} \, ,
\end{align}
where we used Hölder's inequality in the first inequality. Combining Eqs. \eqref{eq: grad W term I} and \eqref{eq: grad W term II} and plugging in Eq. \eqref{eq: bound terms gard w} yields
\begin{align} \label{eq: useful inequality in grad}
    \norm{(\X \ub_{t})^\top \sftd{\g_t} \X} \leq 2 \norm{\X \ub_{t}}_\infty \, \norm{\X}_{2, \infty} \, .
\end{align}
Therefore,
\begin{align} \label{eq:GradWbound}
    \norm{\nabla_{\Wv} \Phi(\X; \thetab)} \leq 2 \, \norm{\X}_{2, \infty} \, \sum_{t=1}^T \norm{\x_{t}} \, \norm{\X \ub_{t}}_\infty \, .
\end{align}
Using Eqs. \eqref{eq:GradVbound} and \eqref{eq:GradWbound} we conclude that
\begin{align} \label{eq:Gradboundraw}
    \norm{\nabla \Phi(\X; \thetab)} \leq 2 \, \norm{\X}^2_{2, \infty} \, \sum_{t=1}^T \, \norm{\X \ub_{t}}_\infty + \sqrt{T} \, \norm{\X}_{2, \infty} \, .
\end{align}
By using Assumption \ref{ass:features} to simplify Eq. \eqref{eq:Gradboundraw},
\begin{align} \label{eq:Gradbound}
    \norm{\nabla_{\thetab} \Phi(\X; \thetab)} \leq 2 \, R^2 \, \sum_{t=1}^T \, \norm{\X \ub_{t}}_\infty + \sqrt{T} \, R \, ,
\end{align}
or
\begin{align} \label{eq:Gradbound2}
    \norm{\nabla_{\thetab} \Phi(\X; \thetab)} \leq 2 \, \sqrt{T} \, R^3 \, \norm{\Ub}_F + \sqrt{T} \, R \, .
\end{align}
We need to derive the Hessian to bound the greatest eigenvalue of it.
\begin{align*}
    \nabla_{\thetab}^2 \Phi(\X; \thetab) =
    \begin{bmatrix}
        \nabla_{\vec{\Ub}}^2 \Phi(\X; \thetab) & \nabla_{\vec{\W} \vec{\Ub}}^2 \Phi(\X; \thetab) \\
        \nabla_{\vec{\W} \vec{\Ub}}^2 \Phi(\X; \thetab)^\top & \nabla_{\vec{\W}}^2 \Phi(\X; \thetab)
    \end{bmatrix} \, .
\end{align*}
First, we compute the Hessian with respect to $\Ub$,
\begin{align} \label{eq:hessV}
    \nabla_{\Ubv}^2 \Phi(\X; \thetab) = \nabla_{\Ub} (\bm{\varphi}(\G) \X) = \textbf{0}_{Td \times Td} \, .
\end{align}
In the next step, we use the third statement of Lemma \ref{lem:grad and hess singlehead appendix} and set $\ab = \unitvector{T}{t}$ and $\bb = \unitvector{d}{j}$ to compute the Hessian with respect to $\W$ and $\Ub$,
\begin{align}
    \nabla_{\W} ([\nabla_{\Ub} \Phi(\X; \thetab)]_{tj}) = \x_{t} \X_{:,j}^\top \sftd{\g_t} \X \, .
\end{align}
recall, $\X_{:,j}$ is the $j$-th column of $\X$. Therefore,
\begin{align} \label{eq:hessWV}
\NiceMatrixOptions{xdots/shorten=0.5em}
    \nabla_{\Wv \Ubv}^2 \Phi(\X; \thetab) =
    \begin{bNiceMatrix}
        \text{vec}(\nabla_{\W} ([\nabla_{\Ub} \Phi(\X; \thetab)]_{11}))^\top \\ \vdots \\ \text{vec}(\nabla_{\W} ([\nabla_{\Ub} \Phi(\X; \thetab)]_{1d}))^\top \\ \vdots \\ \text{vec}(\nabla_{\W} ([\nabla_{\Ub} \Phi(\X; \thetab)]_{Td}))^\top
    \end{bNiceMatrix} \, .
\end{align}
Lastly, in order to compute the Hessian with respect to $\W$, we use the last statement of Lemma \ref{lem:grad and hess singlehead appendix} and set $\cb = \unitvector{d}{i}$ and $\bb = \unitvector{d}{j}$,
\begin{align}
    \nabla_{\W} ([\nabla_{\W} \Phi(\X; \thetab)]_{ij})& = \sum_{t=1}^T X_{ti} \x_t \big( \ub_t^\top \X^\top \operatorname{diag} (\X_{:,j}) - \bm{\varphi}(\g_t)^\top \X_{:,j} \ub_t^\top \X^\top \nn \\
    &\qquad- \bm{\varphi}(\g_t)^\top \X \ub_t \X_{:,j}^\top \big) \sftd{\g_t} \X \, .
\end{align}
Thus,
\begin{align} \label{eq:hessW}
\NiceMatrixOptions{xdots/shorten=0.5em}
    \nabla_{\Wv}^2 \Phi(\X; \thetab) =
    \begin{bNiceMatrix}
        \text{vec}(\nabla_{\W} ([\nabla_{\W} \Phi(\X; \thetab)]_{11}))^\top \\ \vdots \\ \text{vec}(\nabla_{\W} ([\nabla_{\W} \Phi(\X; \thetab)]_{1d}))^\top \\ \vdots \\ \text{vec}(\nabla_{\W} ([\nabla_{\W} \Phi(\X; \thetab)]_{dd}))^\top
    \end{bNiceMatrix} \, .
\end{align}
To find the maximum eigenvalue of the Hessian we need to upper-bound
\begin{align*}
    \max_{\norm{\vb} = 1} \inp{\vb}{\nabla_{\thetab}^2 \Phi(\X; \thetab) \, \vb} \, ,
\end{align*}
where
\begin{align*}
    \NiceMatrixOptions{xdots/shorten=0.5em}
    \vb = \concat\left({
        \pb_1 , \ldots, \pb_T , \qb_1 , \ldots , \qb_d}\right)
    \in \R^{Td+d^2} \qquad \text{and} \qquad \pb_t \in \R^d \, , t\in[T] \; \, \text{and} \; \, \qb_j \in \R^d \, , j\in[d] \, .
\end{align*}
Then,
\begin{align} \label{eq: Spectral norm bound}
    &\norm{\nabla_{\thetab}^2 \Phi(\X; \thetab)} \nn \\
    &\quad\leq \underbrace{\max_{\norm{\vb} = 1} \sum_{i=1}^d \sum_{j=1}^d \sum_{t=1}^T \sum_{k=1}^d X_{ti} \, q_{ij} X_{tk} \, \Big( (\X \ub_t)^\top \operatorname{diag} (\X_{:,j}) - \sft{\g_t}^\top \X_{:,j} (\X \ub_t)^\top - \sft{\g_t}^\top \X \ub_t \X_{:,j}^\top \Big) \, \sftd{\g_t} \X \qb_k}_{\term{1}} \nn \\ &\qquad + \underbrace{2 \, \max_{\norm{\vb} = 1} \sum_{j=1}^d \sum_{t=1}^T p_{tj} \sum_{i=1}^d X_{ti} \, \X_{:,j}^\top \sftd{\g_t} \X \qb_i}_{\term{2}} \, .
\end{align}
First, we bound $\term{1}$:
\begin{align}
    &\max_{\norm{\vb} = 1} \sum_{i=1}^d \sum_{j=1}^d \sum_{t=1}^T \sum_{k=1}^d X_{ti} \, q_{ij} X_{tk} \, \Big( (\X \ub_t)^\top \operatorname{diag} (\X_{:,j}) - \sft{\g_t}^\top \X_{:,j} (\X \ub_t)^\top - \sft{\g_t}^\top \X \ub_t \X_{:,j}^\top \Big) \, \sftd{\g_t} \X \qb_k \nn \\ &\quad \leq \norm{\X}_{1, \infty}^2 \, \Bigg( \underbrace{\max_{\norm{\vb} = 1} \sum_{i=1}^d \sum_{j=1}^d \sum_{t=1}^T \sum_{k=1}^d \norm{\qb_{i}} \, \norm{\qb_{k}} \, \norm{(\X \ub_t)^\top \operatorname{diag} (\X_{:,j}) \sftd{\g_t} \X}}_{\term{I}} + \nn \\ &\qquad \underbrace{\max_{\norm{\vb} = 1} \sum_{i=1}^d \sum_{j=1}^d \sum_{t=1}^T \sum_{k=1}^d \norm{\qb_{i}} \, \norm{\qb_{k}} \, \norm{\sft{\g_t}^\top \X_{:,j} (\X \ub_t)^\top \sftd{\g_t} \X}}_{\term{II}} + \nn \\ &\qquad  \underbrace{\max_{\norm{\vb} = 1} \sum_{i=1}^d \sum_{j=1}^d \sum_{t=1}^T \sum_{k=1}^d \norm{\qb_{i}} \, \norm{\qb_{k}} \, \norm{\sft{\g_t}^\top \X \ub_t \X_{:,j}^\top \sftd{\g_t} \X}}_{\term{III}} \Bigg) \, .
\end{align}

Note that we used $\| \cdot \|_\infty \leq \| \cdot\|$. For $\term{I}$, we have:
\begin{align} \label{eq: term I in frob of hess W}
    &\max_{\norm{\vb} = 1} \sum_{i=1}^d \sum_{j=1}^d \sum_{t=1}^T \sum_{k=1}^d \norm{\qb_{i}} \, \norm{\qb_{k}} \, \norm{(\X \ub_t)^\top \operatorname{diag} (\X_{:,j}) \sftd{\g_t} \X} \nn \\
    &\qquad \leq 2 \, \max_{\norm{\vb} = 1} \sum_{i=1}^d \sum_{j=1}^d \sum_{t=1}^T \sum_{k=1}^d \norm{\qb_{i}} \, \norm{\qb_{k}} \, \norm{\X}_{2, \infty} \, \norm{\operatorname{diag} (\X_{:,j}) \X \ub_t}_\infty \nn \\ &\qquad\leq 2 \, \max_{\norm{\vb} = 1} \sum_{i=1}^d \sum_{j=1}^d \sum_{t=1}^T \sum_{k=1}^d \norm{\qb_{i}} \, \norm{\qb_{k}} \, \norm{\X}_{2, \infty} \, \norm{\X_{:,j}}_\infty \norm{\X \ub_t}_\infty \leq 2 \, d^2 \, \norm{\X}_{2, \infty} \, \norm{\X}_{1, \infty} \, \sum_{t=1}^T \norm{\X \ub_t}_\infty \, ,
\end{align}

where in the last inequality we used Cauchy-Schwarz inequality. Similarly for $\term{II}$ we can write:
\begin{align} \label{eq: term II in frob of hess W}
    &\max_{\norm{\vb} = 1} \sum_{i=1}^d \sum_{j=1}^d \sum_{t=1}^T \sum_{k=1}^d \norm{\qb_{i}} \, \norm{\qb_{k}} \, \norm{\sft{\g_t}^\top \X_{:,j} (\X \ub_t)^\top \sftd{\g_t} \X} \nn \\
     &\quad\leq 2 \, \max_{\norm{\vb} = 1} \sum_{i=1}^d \sum_{j=1}^d \sum_{t=1}^T \sum_{k=1}^d \norm{\qb_{i}} \, \norm{\qb_{k}} \, \abs{\bm{\varphi}(\g_t)^\top \X_{:,j}} \, \norm{\X}_{2, \infty} \, \norm{\X \ub_t}_\infty \leq 2 \, d^2 \, \norm{\X}_{2, \infty} \, \norm{\X}_{1, \infty} \, \sum_{t=1}^T \norm{\X \ub_t}_\infty \, ,
\end{align}

where we used Hölder's inequality in the last inequality. And at the end, for $\term{III}$ we have:
\begin{align} \label{eq: term III in frob of hess W}
    &\max_{\norm{\vb} = 1} \sum_{i=1}^d \sum_{j=1}^d \sum_{t=1}^T \sum_{k=1}^d \norm{\qb_{i}} \, \norm{\qb_{k}} \, \norm{\sft{\g_t}^\top \X \ub_t \X_{:,j}^\top \sftd{\g_t} \X} \nn \\
    &\quad \leq 2 \, \max_{\norm{\vb} = 1} \sum_{i=1}^d \sum_{j=1}^d \sum_{t=1}^T \sum_{k=1}^d \norm{\qb_{i}} \, \norm{\qb_{k}} \, \abs{\bm{\varphi}(\g_t)^\top \X \ub_t} \, \norm{\X}_{2, \infty} \, \norm{\X_{:,j}}_\infty \leq 2 \, d^2 \, \norm{\X}_{2, \infty} \, \norm{\X}_{1, \infty} \, \sum_{t=1}^T \norm{\X \ub_t}_\infty \, .
\end{align}
Then, we upper-bound $\term{2}$:
\begin{align} \label{eq: term 2 in hess}
    2 \, \max_{\norm{\vb} = 1} \sum_{j=1}^d \sum_{t=1}^T p_{tj} \sum_{i=1}^d X_{ti} \, \X_{:,j}^\top \sftd{\g_t} \X \qb_i &\leq 2 \, \max_{\norm{\vb} = 1} \sum_{j=1}^d \sum_{i=1}^d \sum_{t=1}^T \norm{\pb_{t}} \, \norm{\qb_{i}} \, \norm{\X}_{1, \infty} \, \norm{\X_{:,j}^\top \sftd{\g_t} \X} \nn \\ &\leq 4d \, \frac{\sqrt{T \, d}}{2} \, \norm{\X}_{2, \infty} \, \norm{\X}_{1, \infty}^2 = 2 \, d \, \sqrt{T \, d} \, \norm{\X}_{2, \infty} \, \norm{\X}_{1, \infty}^2 \, .
\end{align}

The last inequality comes from the calculation of the gradient concluded in Eq. \eqref{eq: useful inequality in grad}. Plugging in Eqs. \eqref{eq: term I in frob of hess W}, \eqref{eq: term II in frob of hess W}, \eqref{eq: term III in frob of hess W}, and \eqref{eq: term 2 in hess} in Eq. \eqref{eq: Spectral norm bound}, we conclude that
\begin{align} \label{eq: spect norm bound final}
    \norm{\nabla_{\thetab}^2 \Phi(\X; \thetab)} \leq 6 \, d^2 \, \norm{\X}_{2, \infty} \, \norm{\X}_{1, \infty}^3 \, \sum_{t=1}^T \norm{\X \ub_t}_\infty + 2 \, d \, \sqrt{T \, d} \, \norm{\X}_{2, \infty} \, \norm{\X}_{1, \infty}^2 \, .
\end{align}
Applying Assumption \ref{ass:features}, we have:
\begin{align} \label{eq: hess bound with ass}
    \norm{\nabla_{\thetab}^2 \Phi(\X; \thetab)} \leq 6 \, d^2 \, R^4 \, \sum_{t=1}^T \norm{\X \ub_t}_\infty + 2 \, d \, \sqrt{T \, d} \, R^3 \, ,
\end{align}
or
\begin{align} \label{eq: hess bound with ass2}
    \norm{\nabla_{\thetab}^2 \Phi(\X; \thetab)} &\leq 6 \, d^2 \, R^4 \, \sum_{t=1}^T \max_{\tau\in[T]} \x_\tau^\top \ub_t + 2 \, d \, \sqrt{T \, d} \, R^3 \nn \\ &\leq 6 \, d^2 \, R^4 \, \sum_{t=1}^T \max_{\tau\in[T]} \norm{\x_\tau} \norm{\ub_t} + 2 \, d \, \sqrt{T \, d} \, R^3 \nn \\ &\leq 6 \, \sqrt{T} \, d^2 \, R^5 \, \norm{\Ub}_F + 2 \, d \, \sqrt{T \, d} \, R^3 \, ,
\end{align}
where in the second and the last inequalities we used Cauchy-Schwartz inequality.

\end{proof}

We now derive bounds for the multi-head attention model.

\begin{lemma} \label{lem:multihead}
Recall the multihead model in Eq. \eqref{eq:SA model}.
\begin{align}
    \Phit(\X; \thetabt) = \frac{1}{\sqrt{H}} \sum_{h=1}^H \inp{\Ub_h}{\sft{\X\W_h\X^\top}\X}\,.\label{eq:SA model multi}
\end{align}
The following are true under Assumption \ref{ass:features}.

\begin{enumerate}
  \item $\begin{aligned}[t] \norm{\nabla_{\thetabt} \Phit(\X; \thetabt)} \leq \sqrt{T} \, R \, \left(2 \, R^2 \, \max_{h \in [H]} \norm{\Ub_h}_F + 1 \right) \, \end{aligned}$.
  \item $\begin{aligned}[t] \norm{\nabla_{\thetabt}^2 \Phit(\X; \thetabt)} \leq \frac{2 \, d \, \sqrt{T \, d} \, R^3}{\sqrt{H}} \, \left( 3 \, \sqrt{d} \, R^2 \, \max_{h \in [H]} \norm{\Ub_h}_F + 1 \right) \, \end{aligned}$.
\end{enumerate}

\end{lemma}

\begin{proof}
From Equation \eqref{eq:Gradbound2} for single-head attention,
\begin{align}
    \norm{\nabla_{\thetabt} \Phit(\X; \thetabt)}^2 &\leq \frac{1}{H} \sum_{h=1}^H \Big(2 \, \sqrt{T} \, R^3 \, \norm{\Ub_h}_F + \sqrt{T} \, R \Big)^2 \notag \\ &\leq \Big( 2 \, \sqrt{T} \, R^3 \, \max_{h \in [H]} \norm{\Ub_h}_F + \sqrt{T} \, R \Big)^2 \, .
\end{align}
Denote $\NiceMatrixOptions{xdots/shorten=0.5em} \vb^\top = \begin{bNiceMatrix}
    \vb^\top_{1} & \cdots & \vb^\top_{H}
\end{bNiceMatrix} \in \R^{H(Td+d^2)}$, where $\norm{\vb} = 1$. Using Equation \eqref{eq: hess bound with ass2},
\begin{align}
    \inp{\vb}{\nabla_{\thetabt}^2 \Phit(\X; \thetabt) \vb} &\leq \frac{1}{\sqrt{H}} \, \Big( 6 \, \sqrt{T} \, d^2 \, R^5 \, \max_{h \in [H]} \norm{\Ub_h}_F + 2 \, d \, \sqrt{T \, d} \, R^3 \Big) \, \sum_{h=1}^H \norm{\vb_{h}}^2 \notag \\ &= \frac{1}{\sqrt{H}} \, \Big( 6 \, \sqrt{T} \, d^2 \, R^5 \, \max_{h \in [H]} \norm{\Ub_h}_F + 2 \, d \, \sqrt{T \, d} \, R^3 \Big) \, .
\end{align}

\end{proof}


\subsection{Proof of Corollary \ref{coro:objective gradient/hessian_mainbody}: Objective's Gradient/Hessian} \label{app:B2}

    Corollary \ref{coro:objective gradient/hessian_mainbody} follows immediately by the result of Lemma \ref{lem:objective gradient/hessianapp} below and using a more relaxed bound $\max_{h \in [H]} \norm{\Ub_h}_F \leq \maxnorm{\thetabt}$. 

\begin{lemma}[Tight version of Corollary \ref{coro:objective gradient/hessian_mainbody}] \label{lem:objective gradient/hessianapp}
    Let Assumption \ref{ass:features} hold and we use logistic loss function. Then, the following are true for the loss gradient and Hessian:
\begin{enumerate}
  \item $\begin{aligned}[t] \norm{\nabla \Lhat(\thetabt)} \leq \sqrt{T} \, R \, \left( 2 \, R^2 \, \max_{h \in [H]} \norm{\Ub_h}_F + 1 \right) \Lhat(\thetabt)  \, \end{aligned}$.
  \item $\begin{aligned}[t] \norm{\nabla^2 \Lhat(\thetabt)} &\leq \frac{2 \, d \, \sqrt{T \, d} \, R^3}{\sqrt{H}} \, \left( 3 \, \sqrt{d} \, R^2 \, \max_{h \in [H]} \norm{\Ub_h}_F + 1 \right)  +  \frac{T \, R^2}{4} \, \left( 2 \, R^2 \, \max_{h \in [H]} \norm{\Ub_h}_F + 1 \right)^2 \, . \end{aligned}$
  \item $\begin{aligned}[t] \lambda_{\text{min}}(\nabla^2 \Lhat(\thetabt)) \geq - \frac{2 \, d \, \sqrt{T \, d} \, R^3}{\sqrt{H}} \, \left( 3 \, \sqrt{d} \, R^2 \, \max_{h \in [H]} \norm{\Ub_h}_F + 1 \right) \, \Lhat(\thetabt) \, . \end{aligned}$
\end{enumerate}
\end{lemma}

\begin{proof}
    The loss gradient is derived as follows,
    \begin{align*}
        \nabla \Lhat(\thetabt) = \frac{1}{n} \sum_{i=1}^n \ell'(y_i \, \Phit(\X_i; \thetabt)) \, y_i \, \nabla_{\thetabt} \Phit(\X_i; \thetabt) \, .
    \end{align*}
    Recalling that $y_i \in \{ \pm 1 \}$, we can write
    \begin{align}
        \norm{\nabla \Lhat(\thetabt)} &= \frac{1}{n} \norm{\sum_{i=1}^n \ell'(y_i \, \Phit(\X_i; \thetabt)) \, y_i \, \nabla_{\thetabt} \Phit(\X_i; \thetabt)} \notag  \leq \frac{1}{n} \sum_{i=1}^n \abs{\ell'(y_i \, \Phit(\X_i; \thetabt))} \, \norm{\nabla_{\thetabt} \Phit(\X_i; \thetabt)} \notag \, .
    \end{align}
    Thus, using Lemma \ref{lem:multihead} to bound the norm of the model's gradient:
    \begin{align}
        \norm{\nabla \Lhat(\thetabt)} \leq \sqrt{T} \, R \, \left( 2 \, R^2 \, \max_{h \in [H]} \norm{\Ub_h}_F + 1 \right) \Lhat(\thetabt) \, .
    \end{align}
    For the Hessian of loss, note that
    \begin{align} \label{eq:Hessloss}
        \nabla^2 \Lhat(\thetabt) = \frac{1}{n} \sum_{i=1}^n \ell''(y_i \, \Phit(\X_i; \thetabt)) \, \nabla_{\thetabt} \Phit(\X_i; \thetabt) \, \nabla_{\thetabt} \Phit(\X_i; \thetabt)^\top + \ell'(y_i \, \Phit(\X_i; \thetabt)) \, y_i \, \nabla_{\thetabt}^2 \Phit(\X_i; \thetabt) \, .
    \end{align}
    It follows that
    \begin{align}
        \norm{\nabla^2 \Lhat(\thetabt)} &= \norm{\frac{1}{n} \sum_{i=1}^n \ell''(y_i \, \Phit(\X_i; \thetabt)) \, \nabla_{\thetabt} \Phit(\X_i; \thetabt) \, \nabla_{\thetabt} \Phit(\X_i; \thetabt)^\top + \ell'(y_i \, \Phi(\X_i; \thetabt)) \, y_i \, \nabla_{\thetabt}^2 \Phit(\X_i; \thetabt)} \notag \\ &\leq \frac{1}{n} \sum_{i=1}^n \abs{\ell'(y_i \, \Phit(\X_i; \thetabt))} \, \norm{\nabla_{\thetabt}^2 \Phit(\X_i; \thetabt)} + \abs{\ell''(y_i \, \Phit(\X_i; \thetabt))} \, \norm{\nabla_{\thetabt} \Phit(\X_i; \thetabt)}^2 \notag 
        \\ &\leq \frac{2 \, d \, \sqrt{T \, d} \, R^3}{\sqrt{H}} \, \left( 3 \, \sqrt{d} \, R^2 \, \max_{h \in [H]} \norm{\Ub_h}_F + 1 \right) +  \frac{T \, R^2}{4} \, \left( 2 \, R^2 \, \max_{h \in [H]} \norm{\Ub_h}_F + 1 \right)^2 \, .
    \end{align}
    To lower-bound the minimum eigenvalue of the Hessian of loss, note that $\ell(\cdot)$ is convex and thus $\ell''(\cdot) \geq 0$. Therefore, the first term in \eqref{eq:Hessloss} is positive semi-definite and the second term can be lower-bounded as follows,
    \begin{align}
        \lambda_{\text{min}}(\nabla^2 \Lhat(\thetabt)) &\geq - \frac{1}{n} \sum_{i=1}^n \norm{\ell'(y_i \, \Phit(\X_i; \thetabt)) \, y_i \, \nabla_{\thetabt}^2 \Phit(\X_i; \thetabt)} \notag \\ &\geq - \frac{1}{n}  \sum_{i=1}^n \abs{\ell'(y_i \, \Phit(\X_i; \thetabt))} \, \norm{\nabla_{\thetabt}^2 \Phit(\X_i; \thetabt)} \notag 
        \\ &\geq - \frac{2 \, d \, \sqrt{T \, d} \, R^3}{\sqrt{H}} \, \left( 3 \, \sqrt{d} \, R^2 \, \max_{h \in [H]} \norm{\Ub_h}_F + 1 \right) \, \Lhat(\thetabt) \, .\nn
    \end{align}
\end{proof}

\vspace{-10mm}
\section{Training Analysis}\label{sec:train}


\subsection{Preliminaries}\label{sec:app train prep}

 Throughout this section we drop the $\widetilde{\cdot}$ in $\thetabt$ and $\Phit(\X_i; \thetabt)$ as everything refers to the full model. Moreover, $\thetabt^{(K)}$ and $\thetabt^{(0)}$ are denoted by $\thetab_{K}$ and $\thetab_{0}$.

 The proof of both the training and generalization analysis follows the high-level steps outlined in \cite{taheri2023generalization}. However, our analysis focuses on the self-attention model, which differs from the two-layer perceptron studied in the referenced work. Notably, we train both the attention weights and the classifier head, whereas \cite{taheri2023generalization} assumes fixed outer layer weights. This introduces a new challenge as the smoothness and curvature of the objective function at point $\thetab$ become dependent on $\thetab$ itself (see Corollary \ref{coro:objective gradient/hessian_mainbody}). Consequently, we make careful adjustments in the proof to account for this challenge.

 Our analysis critically uses the following property of the loss objective from Corollary \ref{coro:objective gradient/hessian_mainbody}:
 $
\forall\thetab\,:\,\,\,\lambda_{\min}\left(\nabla^2 \Lhat(\thetab)\right)\geq - \kappa(\thetab)\cdot \Lhat(\thetab)$, $ \kappa(\thetab):=\frac{\beta_3(\thetab)}{\sqrt{H}}\,.
 $
 Note from the definition of $\beta_3(\cdot)$ that 
 $
 \forall \thetab_1,\thetab_2\,:\,\,\, \max_{\thetab\in[\thetab_1,\thetab_2]}\beta_3(\thetab) = \beta_3(\thetab_1) \maxi \beta_3(\thetab_2)\,.
 $
 Thus, the above property of the loss implies the following \emph{local self-bounded weak convexity} property on the line $[\thetab_1,\thetab_2]$:
 \begin{align}\label{eq:local weak convexity}
    \forall\thetab_1,\thetab_2\,: \,\,\, \min_{\thetab\in[\thetab_1, \thetab_2]} \lambda_{\min} \left(\nabla^2 \Lhat(\thetab)\right) \geq - \frac{\beta_3(\thetab_1) \maxi \beta_3(\thetab_2)}{\sqrt{H}}\cdot \Lhat(\thetab).
 \end{align}
 In turn, Equation \eqref{eq:local weak convexity} can be used to prove that the loss satisfies a generalized local quasi-convexity (GLQC) property as formalized in the proposition below. The proposition is only a slight modification of \cite[Prop. 8]{taheri2023generalization}. While, a direct application of their result is not possible since the self-bounded weak convexity property in \eqref{eq:local weak convexity} holds only locally on the line $[\thetab_1,\thetab_2]$, an inspection of their proof shows that this is sufficient.

\subsection{Proof of Proposition \ref{prop:GLQCapp}}
First, we restate the proposition below for the reader's convenience.

\begin{propo}[GLQC property]
Let $\thetab_1, \thetab_2$ be points sufficiently close to each other, such that
\[
{2\,\left(\beta_3(\thetab_1)\maxi \beta_3(\thetab_2)\right)}\,\|\thetab_1-\thetab_2\|^2 \leq \sqrt{H}     \,.
\]
Then, the following generalized local quasi-convexity (GLQC) property holds:
\begin{align}\nn
\max_{\thetab\in[\thetab_1,\thetab_2]} \widehat L(\thetab)\le \frac{4}{3}\,\left(\widehat L(\thetab_1)\maxi \widehat L(\thetab_2)\right).
\end{align}
\end{propo}

\begin{proof}
    Although the loss $\Lhat$ is not uniformly self-bounded weakly convex as assumed in \cite[Prop. 8]{taheri2023generalization}, an inspection of their proof shows that for every $\thetab_1,\thetab_2$ it suffices that the loss is locally self-bounded weakly convex. With this observation, we can apply their proposition with  $\kappa\leftarrow \kappa(\thetab_1)\maxi \kappa(\thetab_2)$, which gives the desired.
\end{proof}

\subsection{Key Lemmas}

{The proof of Theorem \ref{thm:train}
 consists of several intermediate lemma, which we state and prove in this section.}
\begin{lemma}[Descent Lemma]\label{lem:descent} Let Assumption \ref{ass:features} hold. Then, for any iteration $k \geq 0$ we have step-wise descent:
\begin{align}\label{eq: descent}
\Lhat(\thetab_{k+1}) \leq \Lhat(\thetab_k) - \frac{\eta}{2}\left\|\nabla \Lhat\left(\thetab_k\right)\right\|^2,
\end{align}
provided $\eta \leq \frac{1}{\rho_k}$ where $\rho_k$ is the objective's local smoothness parameter defined as below,
\begin{align*}
    \rho_k := \beta_2(\thetab_k) \maxi \beta_2(\thetab_{k+1}).
\end{align*}
\end{lemma}
\begin{proof}    
 By Taylor's expansion, there exists a $\thetab' \in\left[\thetab_k, \thetab_{k+1}\right]$ such that,
\begin{align*}
\hat{L}\left(\thetab_{k+1}\right) & =\Lhat\left(\thetab_k\right)+\left\langle\nabla \Lhat\left(\thetab_k\right), \thetab_{k+1}-\thetab_k\right\rangle+\frac{1}{2}\left\langle \thetab_{k+1}-\thetab_k, \nabla^2 \Lhat(\thetab')\left(\thetab_{k+1}-\thetab_k\right)\right\rangle \\
& \leq \Lhat\left(\thetab_k\right)+\left\langle\nabla \Lhat\left(\thetab_k\right), \thetab_{k+1}-\thetab_k\right\rangle+\frac{1}{2} \max _{\thetab' \in\left[\thetab_k, \thetab_{k+1}\right]}\left\|\nabla^2 \Lhat(\thetab')\right\| \cdot\left\|\thetab_{k+1}-\thetab_k\right\|^2 \\
& \leq \Lhat\left(\thetab_k\right)-\eta\left\|\nabla \Lhat\left(\thetab_k\right)\right\|^2+\frac{\eta^2\rho_k}{2} \cdot\left\|\nabla \Lhat\left(\thetab_k\right)\right\|^2,
\end{align*}
where the last step follows from Corollary \ref{coro:objective gradient/hessian_mainbody}, and using $\max _{\thetab' \in\left[\thetab_k, \thetab_{k+1}\right]} \beta_2(\thetab') = \beta_2(\thetab) \maxi \beta_2(
\thetab_k) = \rho_k$. For $\eta \leq \frac{1}{\rho_k}$, we conclude the claim.
\end{proof}
\begin{lemma}\label{lem:train_induction}
 Assume $\eta >0$ such that step-wise descent \eqref{eq: descent} holds for all $k \in [K-1]$.
Then, for any $\thetab$, the following holds:
\begin{align} \label{eq: regret-step-1}
\frac{1}{K}\sum_{k=1}^K\Lhat(\thetab_k) \leq \Lhat(\thetab) + \frac{\norm {\thetab - \thetab_0}^2}{2\eta K} + \frac{1}{2K}\sum_{k=0}^{K-1}\tau_k \norm{\thetab - \thetab_k}^2 \, ,
\end{align}
where $\tau_{k} :=  \frac{1}{\sqrt{H}}\left(\beta_3(\thetab)\maxi\beta_3(\thetab_k)\right)\max_{\alpha \in [0,1]}\, \Lhat(\thetab_{{k_{\alpha}}}) $. 
\end{lemma}
\begin{proof}
    By Taylor's theorem, for $\thetab_{k_{\alpha}}:= \alpha \thetab_k + (1-\alpha)\thetab, \, \, \alpha \in [0, 1]$ we know that:
    \begin{align*}
    \Lhat(\thetab) \geq \Lhat(\thetab_k) + \langle \nabla \Lhat(\thetab_k), \thetab - \thetab_k \rangle +\frac{1}{2} \lambda_{\min} \left(\nabla^2 \Lhat(\thetab_{k_{\alpha}})\right) \norm{\thetab - \thetab_k}^2.
\end{align*}
Using \eqref{eq: descent}, we get:
\begin{align}
    \Lhat(\thetab_{k+1}) &\leq \Lhat(\thetab) + \langle \nabla \Lhat(\thetab_k), \thetab_k - \thetab \rangle  - \frac{1}{2} \lambda_{\min} \left(\nabla^2 \Lhat(\thetab_{k_{\alpha}})\right) \norm{\thetab - \thetab_k}^2 - \frac{\eta}{2} \norm{\nabla \Lhat(\thetab_{k})}^2 \nonumber \\
    &\leq \Lhat(\thetab) + \frac{\norm {\thetab - \thetab_k}^2}{2\eta} - \frac{\norm {\thetab - \thetab_{k+1}}^2}{2\eta} + \frac{1}{2} \tau_{k}\norm{\thetab - \thetab_k}^2, \label{eq:taylor-min}
\end{align}
where the last step follows by completion of squares using $\thetab_{k+1} - \thetab_k = -\eta \nabla \Lhat(\thetab_{k})$, and also uses Corollary \ref{coro:objective gradient/hessian_mainbody} with $\tau_{k} := \max_{\alpha \in [0,1]}\tau_{k_{\alpha}}$, where $\tau_{k_{\alpha}} = \frac{\beta_3(\thetab_{k_{\alpha}})}{\sqrt{H}}\Lhat(\thetab_{k_{\alpha}})$. Telescoping in \eqref{eq:taylor-min} for $k=0, ..., K-1$, we get the desired.


\end{proof}
 Next, we use the generalized local quasi-convexity property (Proposition \ref{prop:GLQCapp}) to obtain explicit regret bound from Lemma \ref{lem:train_induction}. 

\begin{lemma}\label{lem:train_induction2} 
Suppose the assumptions of Lemma \ref{lem:train_induction} hold. Moreover, assume for all $k\in[K-1]$ it holds that $\sqrt{H} \geq 2\left(\beta_3(\thetab)\maxi\beta_3(\thetab_k)\right) \, \|\thetab-\thetab_k\|^2$. Then,
    \begin{align}
    \frac{1}{K}\sum_{k=1}^K\Lhat(\thetab_k) &\leq 2\Lhat(\thetab) + \frac{3\norm {\thetab - \thetab_0}^2}{4\eta K} + \frac{\Lhat (\thetab_0)}{2K}.
\end{align}
\end{lemma}
\begin{proof}
We have
$ \min_{\thetab' \in [\thetab, \thetab_k]} \lambda_{\text{min}}(\nabla^2 \Lhat(\thetab')) \geq - \frac{\beta_3(\thetab) \maxi \beta_3(\thetab_{k})}{\sqrt{H}} \max_{\thetab' \in [\thetab, \thetab_k]} \Lhat(\thetab')$ from Corollary \ref{coro:objective gradient/hessian_mainbody}. 
Thus, using Proposition \ref{prop:GLQCapp}, we can control $\max_{\alpha \in [0, 1]} \Lhat(\thetab_{k_{\alpha}})$. 
Specifically, by assumption we have for all $k\in[K-1]$ that $$\sqrt{H} \geq 2\left(\beta_3(\thetab)\maxi\beta_3(\thetab_k)\right) \, \|\thetab-\thetab_k\|^2 = 2\max_{\alpha \in [0,1]}\beta_3(\thetab_{k_{\alpha}}) \, \|\thetab-\thetab_k\|^2 \, .$$
Then, by Proposition \ref{prop:GLQCapp}, we have
$$\max_{\alpha \in [0, 1]} \Lhat(\thetab_{k_{\alpha}}) \le \frac{4}{3} \,\max\{\Lhat(\thetab_k), \Lhat(\thetab)\} \leq \frac{4}{3}\Lhat(\thetab_k) + \frac{4}{3}\Lhat(\thetab) \, .$$
Thus, applying this to \eqref{eq: regret-step-1} we have:
\begin{align}
 \frac{1}{K}\sum_{k=1}^K\Lhat(\thetab_k) &\leq \Lhat(\thetab) + \frac{\norm {\thetab - \thetab_0}^2}{2\eta K} + \frac{2}{3K} \sum_{k=0}^{K-1}\frac{1}{\sqrt{H}} \left(\beta_3(\thetab)\maxi\beta_3(\thetab_k)\right)(\Lhat(\thetab_{k}) +\Lhat(\thetab))\norm{\thetab - \thetab_k}^2 \nn \\
 &\leq \Lhat(\thetab) + \frac{\norm {\thetab - \thetab_0}^2}{2\eta K} + \frac{1}{3K} \sum_{k=0}^{K-1}(\Lhat(\thetab_{k}) +\Lhat(\thetab)) \nn \\
 &\le \frac{4}{3}\Lhat (\thetab)+ \frac{\norm {\thetab - \thetab_0}^2}{2\eta K} + \frac{1}{3K} \sum_{k=0}^{K}\Lhat(\thetab_{k}),
\end{align}
where we use the condition on $H$ in the second inequality. Rearranging terms above we conclude the claim of the lemma.
\end{proof}

\begin{lemma}[Iterates-norm bound] \label{lem:iteratenormbound}
   Suppose the descent property \eqref{eq: descent} holds $\forall k \in [K-1]$, and let Assumption \ref{ass:features} hold. Further,  assume that
 \begin{align}\label{eq:w requirement constant}
\|\thetab-\thetab_0\|^2\geq \max\{\eta K \Lhat(\thetab),\eta \Lhat(\thetab_0)\} \, .
\end{align}
 and
\begin{align}\label{eq:m_large_general_constant}
{\sqrt{H}} \geq 36 \, d \, \sqrt{T \, d} \, R^3 \, \left( 3 \, \sqrt{d} \, R^2 \, \bigg(3\norm{\thetab-\thetab_0} + \maxnorm{\thetab}\bigg) + 1 \right) \, \|\thetab-\thetab_0\|^2 \, ,
\end{align}
Then, for all $k\in[K]$,
    \begin{align}\label{eq:wt minus w bound no square}
    \|\thetab_k-\thetab\| \leq 3\| \thetab-\thetab_0\| \, .
    \end{align}
\end{lemma}
\begin{proof}
Denote $A_k=\|\thetab_k-\thetab\|$.
    Start by recalling from Eq. \eqref{eq:taylor-min} that for all $k$:
    \begin{align}
        A_{k+1}^2\leq A_k^2 +  2\eta \Lhat(\thetab) - 2\eta \Lhat(\thetab_{k+1})  +\eta \,(\max_{\alpha\in[0,1]} \tau_{k_{\alpha}})\,A_k^2 \, , \label{eq:tighter bound start}
    \end{align}
where recall that $\tau_{k_{\alpha}} =  \frac{\beta_3(\thetab_{k_{\alpha}})\Lhat(\thetab_{k_{\alpha}})}{\sqrt{H}}$. We will prove the desired statement \eqref{eq:wt minus w bound no square} using induction.  
For $k=0$, $A_0=\|\thetab-\thetab_0\|$. Thus, the assumption of induction holds. Now assume \eqref{eq:wt minus w bound no square} is correct for $j\in[k-1]$, i.e. $A_j\leq 3 \|\thetab-\thetab_0\|, \forall j\in[k-1]$. We will then prove it holds for $k$.  

By induction hypothesis for all $j\in[k-1],$ and all $\alpha \in [0,1]$:
\begin{align}
{\sqrt{H}} &\geq 36 \, d \, \sqrt{T \, d} \, R^3 \, \left( 3 \, \sqrt{d} \, R^2 \, \bigg(3 \norm{\thetab-\thetab_0} + \maxnorm{\thetab}\bigg) + 1 \right) \, \|\thetab-\thetab_0\|^2 \nn \\
&\geq  4 \, d \, \sqrt{T \, d} \, R^3 \, \left( 3 \, \sqrt{d} \, R^2 \, \bigg(3\norm{\thetab-\thetab_0} + \maxnorm{\thetab}\bigg) + 1 \right) \, \|\thetab-\thetab_j\|^2 \nn \\
&\ge  4 \, d \, \sqrt{T \, d} \, R^3 \, \left( 3 \, \sqrt{d} \, R^2 \, \bigg(\maxnorm{\thetab-\thetab_j} + \maxnorm{\thetab}\bigg) + 1 \right)\, \|\thetab-\thetab_j\|^2 \nn \\
&\ge 4 \, d \, \sqrt{T \, d} \, R^3 \, \left( 3 \, \sqrt{d} \, R^2 \, \maxnorm{\thetab_{j_\alpha}} + 1 \right)\, \|\thetab-\thetab_j\|^2 = 2\beta_3(\thetab_{j_{\alpha}})\|\thetab-\thetab_j\|^2 \, \nn.
\end{align}
Thus, by Proposition \ref{prop:GLQCapp}, $\forall j\in[k-1]$ 
\begin{align} \label{eq: glqc-induction}
    \max_{\alpha\in[0,1]} \Lhat(\thetab_{j_\alpha}) &\le  \frac{4}{3} \Lhat(\thetab_j) + \frac{4}{3} \Lhat(\thetab) \, .
\end{align}
Using this in \eqref{eq:tighter bound start} we find for all $j\in[k-1]$,
\begin{align*}
A_{j+1}^2&\leq A_j^2 +  2\eta \Lhat(\thetab) - 2\eta \Lhat(\thetab_{j+1})  +\eta\, \frac{\max_{\alpha \in [0,1]}\beta_3(\thetab_{j_{\alpha}})A_j^2 }{\sqrt{H}}\left(\frac{4}{3} \Lhat(\thetab_j) + \frac{4}{3} \Lhat(\thetab)\right)
\\
&\leq A_j^2 +  2\eta \Lhat(\thetab) - 2\eta \Lhat(\thetab_{j+1})  +\eta\,\left(\frac{2}{3} \Lhat(\thetab_j) + \frac{2}{3} \Lhat(\thetab)\right) \, , 
\end{align*}
where in the second inequality we used \eqref{eq: glqc-induction}.
We proceed by telescoping the above display over $j=0,1,\ldots,k-1$ to get
\begin{align}
    A_k^2&\leq A_0^2 + \frac{8}{3}\eta k \Lhat(\thetab) + \frac{2}{3}\eta \Lhat(\thetab_0) - \frac{4}{3}\eta \sum_{j=1}^{k-1}\Lhat(\thetab_j) - 2\eta \Lhat(\thetab_k)
    \nn \\ &\leq A_0^2 + \frac{8}{3}\eta k \Lhat(\thetab) + \frac{2}{3}\eta \Lhat(\thetab_0) \nn \\ &\leq \|\thetab-\thetab_0\|^2  + \frac{8}{3} \|\thetab-\thetab_0\|^2 + \frac{2}{3} \|\thetab-\thetab_0\|^2 \nn \\ &\leq 9 \|\thetab-\thetab_0\|^2 \, ,
\end{align}
where the second line follows by non-negativity of the loss. Thus, $A_k\le 3\|\thetab-\thetab_0\|$. This completes the induction and proves the lemma.
\end{proof}

\subsection{Proof of Theorem \ref{thm:train}}

We restate the theorem here for convenience, this time also including exact constants.
\begin{theo}[Restatement of Thm. \ref{thm:train}]\label{thm:train appendix}
Fix any $\thetab$ and $H$ satisfying
$$
{\sqrt{H}} \geq 36 \, d \, \sqrt{T \, d} \, R^3 \, \left( 3 \, \sqrt{d} \, R^2 \, \bigg(3\norm{\thetab-\thetab_0} + \maxnorm{\thetab}\bigg) + 1 \right) \, \|\thetab-\thetab_0\|^2.
$$
Further, denote for convenience
\begin{align*}
    \alpha(\thetab) := 3 \, d \, \sqrt{d} \, R^2 \, \left[3 \, \sqrt{T} \, R^3 \left(3 \, \norm{\thetab - \thetab_0} + \maxnorm{\thetab} \right) +  2\, \sqrt{T} \, R \right]
\end{align*}
and 
$
    \rho(\thetab) = \left(\frac{2 \, d \, \sqrt{T \, d} \, R^3}{\sqrt{H}} +  \frac{T \, R^2}{4}\right)\, \alpha(\thetab)^2.
$
Then, for any step-size $\eta \leq 1 \mini 1/\rho(\thetab) \mini \frac{\|\thetab-\thetab_0\|^2}{K\Lhat(\thetab)} \mini \frac{\|\thetab-\thetab_0\|^2}{\Lhat(\thetab_0)}$, 
    the following bounds hold for the training loss and the weights' norm  at iteration $K$ of GD:
\begin{align}\label{eq:train_thm}
      \Lhat(\thetab_K)\leq\frac{1}{K} \sum_{k=1}^K \Lhat(\thetab_k) \leq 2\Lhat(\thetab) + \frac{5\norm {\thetab - \thetab_0}^2}{4\eta K} \qquad\text{and}\qquad \|\thetab_K-\thetab_0\|\le 4\|\thetab-\thetab_0\|.
\end{align}
\end{theo}

\begin{proof}
    Define
\begin{align*}
    \alpha_{\eta}(\thetab) = 3 \, d \, \sqrt{d} \, R^2 \, \left[\left(2 \, \eta \, \sqrt{T} \, R^3 + 1 \right) \left(3 \, \norm{\thetab - \thetab_0} + \maxnorm{\thetab} \right) +  \eta \, \sqrt{T} \, R + 1 \right] \, ,
\end{align*}
and
\begin{align*}
    \rho_{\eta}(\thetab):= \left(\frac{2 \, d \, \sqrt{T \, d} \, R^3}{\sqrt{H}} +  \frac{T \, R^2}{4}\right)\, \alpha_{\eta}(\thetab)^2 \, .
\end{align*}    
From Lemma \ref{lem:descent} recall that:
\begin{align*}
\rho_k:= \beta_2(\thetab_k) \maxi \beta_2(\thetab_{k+1}) \, .
\end{align*}
Now, recalling the definition of $\beta_2(\thetab)$ from Corollary \ref{coro:objective gradient/hessian_mainbody}, we see that $\rho_k$, the objective's smoothness parameter at step $k$ depends on $\maxnorm{\thetab_k}\maxi\maxnorm{\thetab_{k+1}}$, where:
\begin{align*}
        \maxnorm{\thetab_k}\maxi\maxnorm{\thetab_{k+1}} &\leq (\norm{\thetab-\thetab_k} + \maxnorm{\thetab}) \maxi (\norm{\thetab-\thetab_{k+1}} + \maxnorm{\thetab}) \nn \\
        &\leq (\norm{\thetab-\thetab_k} + \maxnorm{\thetab}) \maxi (\norm{\thetab-\thetab_{k}} + \norm{\thetab_{k+1} - \thetab_k} + \maxnorm{\thetab}) \nn \\
        & = \norm{\thetab-\thetab_{k}} + \norm{\thetab_{k+1} - \thetab_k} + \maxnorm{\thetab} \nn \\
        &\leq \eta \sqrt{T}R \left(2 R^2 \, \maxnorm{\thetab_k} + 1 \right) + \norm{\thetab - \thetab_k} + \maxnorm{\thetab} \nn \\
        &\leq  \left(2\eta \sqrt{T}R^3 +1 \right) \left(\norm{\thetab - \thetab_k} + \maxnorm{\thetab}\right) + \eta \sqrt{T} R =: \Theta_k \, ,
\end{align*}
where the second-last inequality follows from Corollary \ref{coro:objective gradient/hessian_mainbody}. For each $\rho_k$, define corresponding $\rho_{\eta}(\Theta_k): = \frac{2 \, d \, \sqrt{T \, d} \, R^3}{\sqrt{H}} \, \left( 3 \, \sqrt{d} \, R^2 \, \Theta_k  + 1 \right)  +  \frac{T \, R^2}{4} \, \left( 2 \, R^2 \, \Theta_k  + 1 \right)^2$. Consider $\rho_0$, it is easy to see that  $\rho_0 \leq \rho_{\eta}(\Theta_0) \leq \rho_{\eta}(\thetab)$. Hence, the descent property of GD holds in first iteration as per Lemma \ref{lem:descent}. Since, for $\eta \leq 1$, $\rho_{\eta}(\thetab) \leq \rho(\thetab)$. Thus, $\eta \leq 1 \maxi \frac{1}{\rho(\thetab)} \, \, \implies \, \eta \leq \frac{1}{\rho(\thetab)} \leq \frac{1}{\rho_{\eta}(\thetab)}  \leq \frac{1}{\rho_{\eta}(\Theta_0)} \leq \frac{1}{\rho_0}$.
    Moreover, note that the assumptions of Lemma \ref{lem:iteratenormbound} are satisfied. Thus, by induction over Lemmas \ref{lem:train_induction2}-\ref{lem:iteratenormbound} and noting that $\rho_k \leq \rho_{\eta}(\Theta_k) \leq \rho_{\eta}(\thetab)$ for all $k\in[K-1]$ by using a similar argument as above, we obtain for any $\eta\le \frac{1}{\rho(\thetab)}$,
    \begin{align}
         &\forall k\in[K] \;\; : \;\;\|\thetab_k-\thetab\| \leq 3\| \thetab-\thetab_0\| \, , \nn\\
         \text{and}\;\;\;\; & \frac{1}{K}\sum_{k=1}^K\Lhat(\thetab_k) \leq 2\Lhat(\thetab) + \frac{3\norm {\thetab - \thetab_0}^2}{4\eta K} + \frac{\Lhat (\thetab_0)}{2K} \, .\label{eq:train_proof}
    \end{align}
    Moreover, by assumptions of the theorem we immediately find that $\frac{1}{2K}\Lhat (\thetab_0) \le \frac{\norm {\thetab - \thetab_0}^2}{2\eta K}$. We also have $\|\thetab_k-\thetab_0\| \le \|\thetab-\thetab_k\|  + \|\thetab-\thetab_0\| \le 4\|\thetab-\thetab_0\|$. This completes the proof.  
\end{proof}

The training proof is summarized in Figure \ref{fig:train-pf-schema}.  

\begin{figure}[h]
    \centering
    \includegraphics[width=0.99\textwidth]{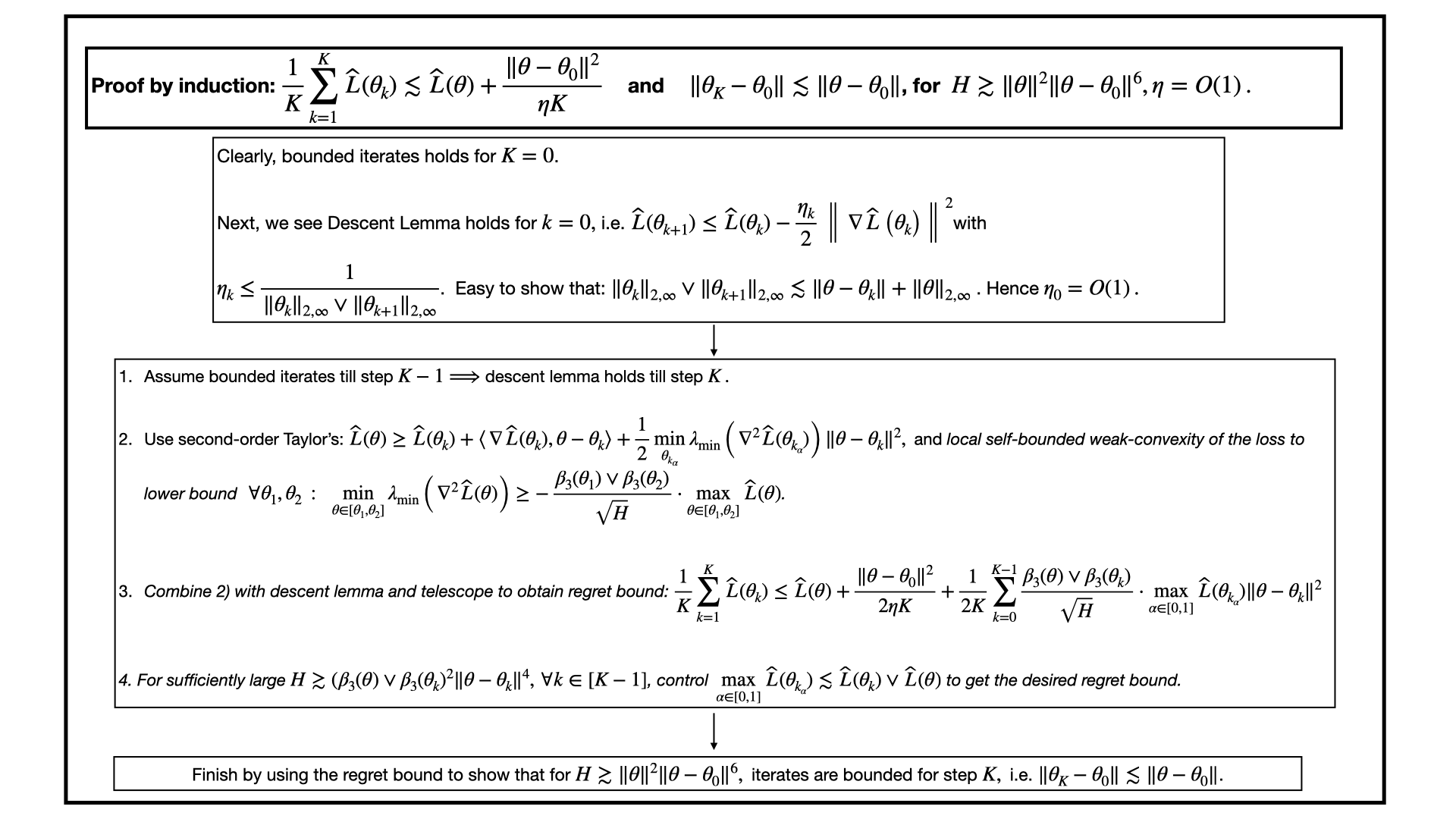}
    \caption{Training proof schema.} 
    \label{fig:train-pf-schema}
\end{figure}


\subsection{Corollary \ref{coro:train interpol}}\label{app:C2}

\begin{coro} [Training loss under realizability] \label{coro:train interpol}
Let Assumptions \ref{ass:features} and \ref{ass:NNR} hold. Fix $K \geq 1$. Assume any $H$ such that 
\begin{align} \label{eq: number of heads train}
{\sqrt{H}} \geq 36 \, d \, \sqrt{T \, d} \, R^3 \, \left( 3 \, \sqrt{d} \, R^2 \, \Big( 3g_0(\frac{1}{K}) + g(\frac{1}{K}) \Big) + 1 \right) \, g_0(\frac{1}{K})^2 \, .
\end{align}
Further, denote for convenience
\begin{align*}
    \alpha(K) := 3 \, d \, \sqrt{d} \, R^2 \, \left[3 \, \sqrt{T} \, R^3 \left(3 g_0(\frac{1}{K}) + g(\frac{1}{K}) \right) +  2\, \sqrt{T} \, R \right]
\end{align*}
and 
$
    \rho(K) = \left(\frac{2 \, d \, \sqrt{T \, d} \, R^3}{\sqrt{H}} +  \frac{T \, R^2}{4}\right)\, \alpha(K)^2.
$
Then, for any step-size $\eta\leq 1 \mini 1/\rho(K) \mini g_0(1)^2 \mini \frac{g_0(1)^2}{\Lhat(\thetab_0)}$, the following bounds hold for the weights' norm and objective at iteration $K$ of GD:
\begin{align}\label{eq:train_cor3}
      \Lhat(\thetab_{K})\leq \frac{2}{K} + \frac{5 g_0(\frac{1}{K})^2}{4\eta K} \qquad\text{and}\qquad \|\thetab_K - \thetab_0\| \le 4 g_0(\frac{1}{K}) \, .
\end{align}
\end{coro}

\begin{proof}
According to Assumption \ref{ass:NNR}, for any sufficiently small $\varepsilon > 0$, there exists a $\thetab^{(\varepsilon)}$ such that $\Lhat(\thetab^{(\varepsilon)}) \leq \varepsilon$ and $\| \thetab^{(\varepsilon)} - \thetab_0 \| = g_0(\varepsilon)$. Pick $\varepsilon = \frac{1}{K}$. Let the step-size $\eta >0$, satisfy the assumption of Descent Lemma \ref{lem:descent}. Since $\Lhat(\thetab^\pare{1/K}) \leq \frac{1}{K}$, we have:
\begin{align*}
    \frac{\| \thetab^\pare{1/K} - \thetab_0 \|^2}{K \Lhat(\thetab^\pare{1/K})} \geq \| \thetab^\pare{1/K} - \thetab_0 \|^2 = g_0(\frac{1}{K})^2 \geq g_0(1)^2\, ,
\end{align*}
and
\begin{align*}
    \frac{\| \thetab^\pare{1/K} - \thetab_0 \|^2}{ \Lhat(\thetab_0)} = \frac{g_0(\frac{1}{K})^2}{\Lhat(\thetab_0)} \geq \frac{g_0(1)^2}{\Lhat(\thetab_0)} \, .
\end{align*}
Therefore, following our assumption on step-size $\eta$, we can conclude that
\begin{align}
    \eta \leq g_0(1)^2 \mini \frac{g_0(1)^2}{\Lhat(\thetab_0)} \leq \frac{\|\thetab^\pare{1/K} - \thetab_0\|^2}{K\Lhat(\thetab^\pare{1/K})} \mini \frac{\|\thetab^\pare{1/K}-\thetab_0\|^2}{\Lhat(\thetab_0)} \, .
\end{align}
where in the second inequality we used the fact that $g_0(\cdot)$ is a non-increasing function. The desired result is
obtained by Theorem \ref{thm:train}.
\end{proof}

\section{Generalization Analysis} \label{sec:generalization}
Throughout this section we drop the $\widetilde{\cdot}$ in $\thetabt$ and $\Phit(\X_i; \thetabt)$ as everything refers to the full model. Moreover, $\thetabt^{(K)}$ and $\thetabt^{(0)}$ are denoted by $\thetab_{K}$ and $\thetab_{0}$.

For the stability analysis below, recall the definition of the leave-one-out (loo) training loss for $i\in[n]$ and note that by denoting $\ell_j(\thetab) := \ell(y_j \Phi(\X_j;\thetab))$ to be the $j$-th sample loss:
$\Lhat^{\neg i}(\thetab):=\frac{1}{n}\sum_{j\neq i} \ell_j(\thetab)$. With these, define the loo model updates of GD on the loo loss for some $\eta >0$:
\[
\thetab_{k+1}^{\neg i}:=\thetab_{k}^{\neg i} -\eta \nabla \Lhat ^{\neg i}(\thetab_k^{\neg i}),~k\geq 0,\qquad \thetab_0^{\neg i}=\thetab_0.
\]
\begin{lemma}\label{lem:gen_stab}
    Assume the conditions of Theorem \ref{thm:train} hold and
\begin{align}
    \sqrt{H} \geq 256 \, d \, \sqrt{T \, d} \, R^3 \, \left(3 \, \sqrt{d} \, R^2 \, \Big(3\|\thetab-\thetab_0\| + \maxnorm{\thetab}\Big)+ 1 \right) \, \|\thetab-\thetab_0\|^2 \, .
\end{align}
    Then, the on-average model stability at iteration $K$ of GD satisfies,
\begin{align}
    \frac{1}{n} \sum_{i=1}^n \left\|\thetab_{K}-\thetab_{K}^{\neg i}\right\|\le \frac{2 \eta}{n} \, \left( \sqrt{T} \, R \, \left( 2 \, R^2 \, \left( 3\|\thetab-\thetab_0\| +\maxnorm{\thetab} \right) + 1 \right) \right) \, \left(2 \, K \Lhat (\thetab) + \frac{9\|\thetab-\thetab_0\|^2}{4\eta}\right) \, .\nn
\end{align}
\end{lemma}
\begin{proof}
First recall from Corollary \ref{coro:objective gradient/hessian_mainbody} that gradient and hessian's norm satisfy
\begin{align*}
&\norm{\nabla \Lhat(\thetab)} \leq  \beta_1(\thetab)\,\Lhat(\thetab) \, ,\\ 
&\norm{\nabla^2 \Lhat(\thetab)} \leq \beta_2(\thetab) \,, \\ 
&\lambda_{\text{min}}(\nabla^2 \Lhat(\thetab)) \geq - \frac{\beta_3(\thetab)}{\sqrt{H}} \Lhat(\thetab) \, . 
 \end{align*}
Applying \cite[Lemma B.1.]{taheri2023generalization}, two arbitrary points $\thetab,\thetab'$ satisfy the following  GD-expansiveness inequality: 
\begin{align}\label{eq:stab1}
    \norm{(\thetab-\eta\nabla \Lhat(\thetab)) - (\thetab'-\eta\nabla \Lhat(\thetab'))} \leq \max_{\alpha\in[0,1]}\left\{\left(1+\frac{\eta \beta_3(\thetab_\alpha)}{\sqrt{H}} \Lhat (\thetab_\alpha) \right) \, \maxi \, \eta \beta_2(\thetab_\alpha)  \right\} \norm{\thetab-\thetab'} \, ,
\end{align}
where $\thetab_\alpha = \alpha \thetab+ (1-\alpha)\thetab'$ denotes a point parameterized by $\alpha\in[0,1]$ in the line segment between $\thetab$ and $\thetab'$. We aim to bound the on-average model stability in the r.h.s of the inequality in Lemma \ref{lem:on-average}. Based on Eq. \eqref{eq:stab1},
\begin{align}
\left\|\thetab_{k+1}-\thetab_{k+1}^{\neg i}\right\| &= 
\left\|\left(\thetab_{k}-\eta \nabla \Lhat^{\neg i}\left(\thetab_{k}\right)\right)-\left(\thetab_{k}^{\neg i}-\eta \nabla \Lhat^{\neg i}\left(\thetab_{k}^{\neg i}\right)\right)-\frac{\eta}{n}\nabla \ell_i\left(\thetab_{k}\right)\right\| \nn \\
& \leq\left\|\left(\thetab_{k}-\eta \nabla \Lhat^{\neg i}\left(\thetab_{k}\right)\right)-\left(\thetab_{k}^{\neg i}-\eta \nabla \Lhat^{\neg i}\left(\thetab_{k}^{\neg i}\right)\right)\right\|+\frac{\eta}{n}\left\|\nabla \ell_i\left(\thetab_{k}\right)\right\| \notag \\
& \leq\left\|\left(\thetab_{k}-\eta \nabla \Lhat^{\neg i}\left(\thetab_{k}\right)\right)-\left(\thetab_{k}^{\neg i}-\eta \nabla \Lhat^{\neg i}\left(\thetab_{k}^{\neg i}\right)\right)\right\|+\frac{\eta \beta_1(\thetab_k)}{n} \ell_i\left(\thetab_{k}\right) \notag \\
& \leq \left( \max_{\alpha\in[0,1]} \left\{ \left(1 + \frac{\eta \beta_3(\thetab_{k_\alpha}^{\neg i})}{\sqrt{H}} \Lhat^{\neg i} (\thetab_{k_\alpha}^{\neg i}) \right) \, \maxi \, \eta \beta_2(\thetab_{k_\alpha}^{\neg i}) \right\}\right) \left\|\thetab_{k}-\thetab_{k}^{\neg i}\right\|+\frac{\eta \beta_1(\thetab_k)}{n} \ell_i\left(\thetab_{k}\right). \label{eq:model_stab3}
\end{align}
In the above we denoted $\thetab_{k_\alpha}^{\neg i}: = \alpha \thetab_k + (1-\alpha)\thetab_k^{\neg i}$. We note that based on our guarantees for the weights' norm during training \eqref{eq:wt minus w bound no square} it can be deduced that for all $\alpha\in[0,1]$,
\begin{align}
\beta_3(\thetab_{k_\alpha}^{\neg i}) &= 2 \, d \, \sqrt{T \, d} \, R^3 \, \left(3 \, \sqrt{d} \, R^2 \, \maxnorm{\alpha \thetab_k + (1-\alpha)\thetab_k^{\neg i}} + 1 \right) \nn \\
&= 2 \, d \, \sqrt{T \, d} \, R^3 \, \left(3 \, \sqrt{d} \, R^2 \, \Big(\maxnorm{\thetab_k} \maxi \maxnorm{\thetab_k^{\neg i}}\Big) + 1 \right) \nn \\
&\le 2 \, d \, \sqrt{T \, d} \, R^3 \, \left(3 \, \sqrt{d} \, R^2 \, \Big((\|\thetab_k-\thetab\| + \maxnorm{\thetab}) \maxi (\|\thetab_k^{\neg i}-\thetab\| + \maxnorm{\thetab}\Big)+ 1 \right) \nn \\
&\le 2 \, d \, \sqrt{T \, d} \, R^3 \, \left(3 \, \sqrt{d} \, R^2 \, \Big(3\|\thetab-\thetab_0\| + \maxnorm{\thetab}\Big)+ 1 \right)=: \tilde\beta_3(\thetab) \, .
\end{align}
Similarly, we obtain
\begin{align}
    \beta_2(\thetab_{k_\alpha}^{\neg i}) &= \frac{2 \, d \, \sqrt{T \, d} \, R^3}{\sqrt{H}} \left( 3 \, \sqrt{d} \, R^2 \, \maxnorm{\thetab_{k_\alpha}^{\neg i}} + 1 \right) + \frac{T \, R^2}{4} \, \bigg( 2 \, R^2 \, \maxnorm{\thetab_{k_\alpha}^{\neg i}} + 1 \bigg)^2 \nn \\
    &\le \frac{2 \, d \, \sqrt{T \, d} \, R^3}{\sqrt{H}} \left( 3 \, \sqrt{d} \, R^2 \, \Big(3\|\thetab-\thetab_0\| + \maxnorm{\thetab} \Big) + 1 \right) \nn \\
    &\qquad + \, \frac{T \, R^2}{4} \, \bigg( 2 \, R^2 \, \Big(3\|\thetab-\thetab_0\| + \maxnorm{\thetab}\Big) + 1 \bigg)^2 =: \tilde\beta_2(\thetab) \, .
\end{align}

Hence, by the notation introduced above and noting that by our assumption on the step-size it holds $\eta\le 1/\tilde\beta_2(\thetab)$, we can rewrite Eq. \eqref{eq:model_stab3} as follows,
\begin{align*}
        \left\|\thetab_{k+1}-\thetab_{k+1}^{\neg i}\right\| \le \left((1+\frac{\eta\tilde\beta_3(\thetab)}{\sqrt{H}})\max_{\alpha\in[0,1]} \Lhat^{\neg i} (\thetab_{k_\alpha}^{\neg i}) \right) \left\|\thetab_{k}-\thetab_{k}^{\neg i}\right\|+\frac{\eta \beta_1(\thetab_k)}{n} \ell_i\left(\thetab_{k}\right) \, .
\end{align*}
Assume 
\begin{align*}    
\sqrt{H} &\ge 128 \, \tilde\beta_3(\thetab) \, \|\thetab-\thetab_0\|^2 \ge 4 \, \tilde\beta_3(\thetab) \, \left(\|\thetab_k-\thetab_0\|^2 + \|\thetab_k^{\neg i}-\thetab_0\|^2\right) \ge 2 \, \tilde\beta_3(\thetab) \, \|\thetab_k-\thetab_k^{\neg i}\|^2 \\
&\ge 2 \, \left(\beta_3(\thetab_k) \maxi \beta_3(\thetab_k^{\neg i}) \right) \, \|\thetab_k-\thetab_k^{\neg i}\|^2,
\end{align*} where we used Theorem \ref{thm:train} in the first inequality. We also have $$\min_{\alpha \in [0,1]}\lambda_{\text{min}}(\nabla^2 \Lhat^{\neg i}(\thetab_{k_\alpha}^{\neg i})) \geq -\frac{\beta_3(\thetab_{k})\maxi \beta_3(\thetab_{k}^{\neg i})}{\sqrt{H}}\Lhat^{\neg i}(\thetab_{k_\alpha}^{\neg i})\,. $$
Thus, by applying Proposition \ref{prop:GLQCapp} on the leave-one-out loss, it holds that for all $\alpha\in[0,1]$,  
$$\max_{\alpha\in[0,1]} \Lhat^{\neg i}(\thetab_{k_\alpha}^{\neg i}) \le \frac{4}{3} \left( \Lhat^{\neg i}(\thetab_k) + \Lhat^{\neg i} (\thetab_k^{\neg i}) \right) \,.$$
Thus,

\begin{align}
    \left\|\thetab_{k+1}-\thetab_{k+1}^{\neg i}\right\| \le \left( \left(1 + \frac{4\eta\tilde\beta_3(\thetab)}{3\sqrt{H}} \right) \cdot \left(\Lhat^{\neg i}(\thetab_k) + \Lhat^{\neg i} (\thetab_k^{\neg i})\right)\right)\left\|\thetab_{k}-\thetab_{k}^{\neg i}\right\|+\frac{\eta \beta_1(\thetab_k)}{n} \ell_i\left(\thetab_{k}\right) \, .
\end{align}
In order to remove the dependence on $k$, note that 
\begin{align}
    \beta_1 (\thetab_k) &\le \sqrt{T} \, R \, \bigg( 2 \, R^2 \, \maxnorm{\thetab_k} + 1 \bigg) \nn \\
    &\le \sqrt{T} \, R \, \bigg( 2 \, R^2 \, \Big(\maxnorm{\thetab-\thetab_k} +\maxnorm{\thetab} \Big) + 1 \bigg) \nn \\
    &\le \sqrt{T} \, R \, \bigg( 2 \, R^2 \, \Big(\|\thetab-\thetab_k\| +\maxnorm{\thetab} \Big) + 1 \bigg) \nn \\
    &\le \sqrt{T} \, R \, \bigg( 2 \, R^2 \, \Big( 3\|\thetab-\thetab_0\| +\maxnorm{\thetab} \Big) + 1 \bigg) =: \tilde\beta_1(\thetab) \, .
\end{align}
For simplicity of exposition, denote $\alpha_{k,i}:=\frac{4\eta\tilde\beta_3(\thetab)}{3\sqrt{H}} \left(\Lhat^{\neg i}(\thetab_k) + \Lhat^{\neg i} (\thetab_k^{\neg i})\right)$. Then by unrolling the iterates we have,

\begin{align}
    \left\|\thetab_{k+1}-\thetab_{k+1}^{\neg i}\right\| &\le (1+\alpha_{k,i})\left\|\thetab_{k}-\thetab_{k}^{\neg i}\right\|+\frac{\eta\tilde\beta_1(\thetab)}{n}\ell_i\left(\thetab_{k}\right)\nn
    \\
    &\le (1+\alpha_{k,i})(1+\alpha_{k-1,i}) \left\|\thetab_{k-1}-\thetab_{k-1}^{\neg i}\right\| + \frac{(1+\alpha_{k,i})\eta\tilde\beta_1(\thetab)}{n}\ell_i\left(\thetab_{k-1}\right)+\frac{\eta\tilde\beta_1(\thetab)}{n}\ell_i\left(\thetab_{k}\right)\nn
    \\
    &\le \sum_{j=1}^k\prod_{l=j}^k (1+\alpha_{l,i})\eta\tilde\beta_1(\thetab)\frac{\ell_i(\thetab_{j-1})}{n}+\eta\tilde\beta_1(\thetab)\frac{\ell_i\left(\thetab_{k}\right)}{n}\nn
    \\
    &\le \sum_{j=1}^k \exp(\sum_{l=j}^k \alpha_{l,i}) \, \eta \tilde\beta_1(\thetab)\frac{\ell_i(\thetab_{j-1})}{n}+\eta\tilde\beta_1(\thetab)\frac{\ell_i\left(\thetab_{k}\right)}{n}\nn
    \\
    &\le \sum_{j=1}^k \exp(\sum_{l=1}^k \alpha_{l,i}) \, \eta \tilde\beta_1(\thetab)\frac{\ell _i(\thetab_{j-1})}{n}+\eta\tilde\beta_1(\thetab)\frac{\ell_i\left(\thetab_{k}\right)}{n} \, , \label{eq:model_stab5}
\end{align}
where in the above we used that $\thetab_0=\thetab_0^{\neg i}$ in unrolling the iterates as well as the fact that for $x\ge 0:1+x\le e^x$. Note that by definition $\Lhat^{\neg i}(\thetab_k)\le \Lhat(\thetab_k)$. By training loss guarantees from Eq. \eqref{eq:train_thm}, we have
\begin{align*}
    \sum_{l=1}^k \alpha_{l,i} &\le \frac{4\eta\tilde\beta_3(\thetab)}{3\sqrt{H}} \sum_{l=1}^k \left(\Lhat^{\neg i}(\thetab_l) + \Lhat^{\neg i} (\thetab_l^{\neg i})\right)\\
    &\le \frac{4\eta\tilde\beta_3(\thetab)}{3\sqrt{H}} \sum_{l=1}^k \left(\Lhat(\thetab_l) + \Lhat(\thetab_l^{\neg i})\right)\\
    &\le \frac{4\eta\tilde\beta_3(\thetab)}{3\sqrt{H}} \left(5k\Lhat(\thetab) + \frac{5\|\thetab-\thetab_0\|^2}{2\eta}\right) \\
    &\le \frac{10\tilde\beta_3(\thetab)\|\thetab-\thetab_0\|^2}{ \sqrt{H}}\\
    &\le \frac{1}{10} \, ,
\end{align*}
where the last step stems from the condition on $\sqrt{H}$. Proceeding from Eq. \eqref{eq:model_stab5}, we find that for the last iterate 
\begin{align}
\left\|\thetab_{K}-\thetab_{K}^{\neg i}\right\| &\le  2 \eta \tilde\beta_1(\thetab) \sum_{k=1}^{K-1} \frac{\ell_i(\thetab_{k-1})}{n}+\eta\tilde\beta_1(\thetab)\frac{\ell_i\left(\thetab_{K-1}\right)}{n} \nn \\
&\le 2 \eta \tilde\beta_1(\thetab) \sum_{k=0}^{K-1} \frac{\ell_i(\thetab_{k})}{n} \, .
\end{align}
It follows that the on-average model stability satisfies,
\begin{align*}
    \frac{1}{n} \sum_{i=1}^n \left\|\thetab_{K}-\thetab_{K}^{\neg i}\right\| &\le \frac{2 \eta \tilde\beta_1(\thetab)}{n^2} \sum_{i=1}^n\sum_{k=0}^{K-1} \ell_i(\thetab_{k})\\
    &= \frac{2 \eta \tilde\beta_1(\thetab)}{n}\sum_{k=0}^{K-1} \Lhat (\thetab_k) \, .
\end{align*}
Applying our training loss guarantees from Eq. \eqref{eq:train_proof} to the r.h.s. above yields,
\begin{align*}
    \frac{1}{n} \sum_{i=1}^n \left\|\thetab_{K}-\thetab_{K}^{\neg i}\right\|\le \frac{2 \eta \tilde\beta_1(\thetab)}{n} \left(2 \, K \Lhat (\thetab) + \frac{9\|\thetab-\thetab_0\|^2}{4\eta}\right) \, ,
\end{align*}
where here we used the assumption that, $\Lhat(\thetab_0)\le \|\thetab-\thetab_0\|^2/\eta$ to simplify the final result. This completes the proof.
\end{proof}

\subsection{Proof of Theorem \ref{thm:gen}}
We restate the theorem here for convenience, this time also including exact constants.
\begin{theo}[Restatement of Thm. \ref{thm:gen}]\label{thm:gen appendix}
Fix any $\thetab$ and $H$ satisfying
$$
{\sqrt{H}} \geq 256 \, d \, \sqrt{T \, d} \, R^3 \, \left( 3 \, \sqrt{d} \, R^2 \, \bigg(3\norm{\thetab-\thetab_0} + \maxnorm{\thetab}\bigg) + 1 \right) \, \|\thetab-\thetab_0\|^2.
$$
Further, denote for convenience
\begin{align*}
    \alpha(\thetab) := 3 \, d \, \sqrt{d} \, R^2 \, \left[3 \, \sqrt{T} \, R^3 \left(3 \, \norm{\thetab - \thetab_0} + \maxnorm{\thetab} \right) +  2\, \sqrt{T} \, R \right]
\end{align*}
and 
$
    \rho(\thetab) = \left(\frac{2 \, d \, \sqrt{T \, d} \, R^3}{\sqrt{H}} +  \frac{T \, R^2}{4}\right)\, \alpha(\thetab)^2.
$
Then, for any step-size $\eta \leq 1 \mini 1/\rho(\thetab) \mini \frac{\|\thetab-\thetab_0\|^2}{K\Lhat(\thetab)} \mini \frac{\|\thetab-\thetab_0\|^2}{\Lhat(\thetab_0)}$,
the expected generalization gap at iteration $K$ satisfies,
\begin{align}
    \E\big[ L(\wt\thetab^{(K)})-\Lhat(\wt\thetab^{(K)})\big] \le \frac{4}{n} \, \E\big[2 \, K \, \Lhat (\wt\thetab) + \frac{9\|\wt\thetab-\wt\thetab^{(0)}\|^2}{4\eta}\big] \, .
\end{align}
\end{theo}
\begin{proof}
Note that the assumptions of Lemma \ref{lem:gen_stab} are satisfied. Moreover, as per Corollary \ref{coro:objective gradient/hessian_mainbody} the objective is Lipschitz at all iterates with parameter $\tilde\beta_1(\thetab)$ since $\forall k\in[K] : \maxnorm{\thetab_k} \le 3\|\thetab-\thetab_0\| + \maxnorm{\thetab}$. Thus, by Lemma \ref{lem:on-average} and Lemma \ref{lem:gen_stab} we have,
    \begin{align}
    \E \Big[L(\thetab_K) - \Lhat(\thetab_K)\Big] \le \frac{4}{n} \, \E\left[ \eta \, \left(\tilde\beta_1(\thetab)\right)^2 \, \Big(2 \, K \, \Lhat(\thetab) + \frac{9\|\thetab-\thetab_0\|^2}{4\eta}\Big)\right] \, .
\end{align}
Recalling the condition on step-size $\eta\le \frac{1}{\rho(\thetab)}\le\frac{1}{ (\tilde\beta_1(\thetab))^2}$ concludes the proof.
\end{proof}

\subsection{Corollary \ref{coro:gen interpol}}\label{app:D2}

\begin{coro}[Generalization loss under realizability]\label{coro:gen interpol}
Let boundedness Assumption \ref{ass:features} hold. Also let realizability assumption \ref{ass:NNR} holds almost surely over the data distribution. Fix $K \geq 1$. Assume any $H$ such that 
\begin{align} \label{eq: number of heads gen}
\sqrt{H} \geq 256 \, d \, \sqrt{T \, d} \, R^3 \, \left(3 \, \sqrt{d} \, R^2 \, \Big(3 g_0(\frac{1}{K}) + g(\frac{1}{K}) \Big)+ 1 \right) \, g_0(\frac{1}{K})^2 \, .
\end{align}
    Let the step-size satisfy $\eta \leq 1 \mini 1 / \rho(K) \mini g_0(1)^2 \mini \frac{g_0(1)^2}{\Lhat(\thetab_0)}$ where $\rho(K)$ is as defined in Corollary \ref{coro:train interpol}.
    Then the expected generalization gap at iteration $K$ of GD satisfies,
    \begin{align}
    \E\Big[L(\thetab_K) - \Lhat(\wt\thetab_K) \Big] \le \frac{17 \, g_0(\frac{1}{K})^2}{\eta \, n} \, .
\end{align}
\end{coro}

\begin{proof}
According to Assumption \ref{ass:NNR}, for any sufficiently small $\varepsilon > 0$, there exists a $\thetab^{(\varepsilon)}$ such that $\Lhat(\thetab^{(\varepsilon)}) \leq \varepsilon$ and $\| \thetab^{(\varepsilon)} - \thetabt_0 \| = g_0(\varepsilon)$. Pick $\varepsilon = \frac{1}{K}$. Let the step-size, $\eta >0$ satisfies the assumption of Descent Lemma \ref{lem:descent}. Since $\Lhat(\thetab^\pare{1/K}) \leq \frac{1}{K}$, we have:
\begin{align*}
    \frac{\| \thetab^\pare{1/K} - \thetab_0 \|^2}{K \Lhat(\thetab^\pare{1/K})} \geq \| \thetab^\pare{1/K} - \thetab_0 \|^2 = g_0(\frac{1}{K})^2 \geq g_0(1)^2
\end{align*}
and
\begin{align*}
    \frac{\| \thetab^\pare{1/K} - \thetab_0 \|^2}{ \Lhat(\thetab_0)} = \frac{g_0(\frac{1}{K})^2}{\Lhat(\thetab_0)} \geq \frac{g_0(1)^2}{\Lhat(\thetab_0)} \, .
\end{align*}
Therefore, we can conclude that
\begin{align}
    \eta \leq g_0(1)^2 \mini \frac{g_0(1)^2}{\Lhat(\thetab_0)} \leq \frac{\|\thetab^\pare{1/K}-\thetab_0\|^2}{K\Lhat(\thetab^\pare{1/K})} \mini \frac{\|\thetab^\pare{1/K}-\thetab_0\|^2}{\Lhat(\thetab_0)} \, .
\end{align}
where in the second inequality we used the fact that $g_0(\cdot)$ is a non-increasing function. The desired result is
obtained by Theorem \ref{thm:gen} and the fact that
\begin{align*}
    K\Lhat(\thetab^\pare{1/K}) \leq \frac{\|\thetab^\pare{1/K}-\thetab_0\|^2}{\eta} = \frac{g_0(\frac{1}{K})^2}{\eta} \, .
\end{align*}
\end{proof}

\subsection{From \good initialization to realizability}
The proposition below shows that starting from a \good initialization we can always find $\thetab^\pare{\eps}$ satisfying the realizability Assumption \ref{ass:NNR} provided the number of heads is large enough.

\begin{propo}[From \good initialization to realizability] \label{propo:init to real}
Suppose \good initialization $\thetab_0$ as per Definition \ref{def:good init}. Fix any $1 \geq \varepsilon > 0$ and let 
    \begin{align} \label{eq: realizability of NTK-separable head}
        \sqrt{H} \geq \frac{5 \, d \, \sqrt{T \, d} \, R^3 \, \Bnorm} {\Bphi} \cdot \left( 3 \sqrt{d} \, R^2 + 1 \right) \cdot \left(\frac{2\Bphi + \log(1/\varepsilon)}{\gamma} \right)^2 \cdot \left(1 \maxi \frac{2\Bphi + \log(1/\varepsilon)}{\gamma} \right) \, .
    \end{align} 
    Then, the realizability Assumption \ref{ass:NNR} holds with $g_0(\varepsilon) = \frac{1}{\gamma} \, \left( 2\Bphi + \log(1/\varepsilon) \right)$ and $g(\eps)=\Bnorm+g_0(\eps)$.
\end{propo}

\begin{proof}
    By Taylor expansion there exists $\thetab' \in [\thetab, \thetab_0]$ such that,
    \begin{align} \label{eq: taylor propo}
        y_i\Phi(\X_i; \thetab) = y_i\Phi(\X_i; \thetab_0) + y_i\left\langle\nabla \Phi \left(\X_i; \thetab_0 \right), \thetab - \thetab_0 \right\rangle + \frac{1}{2}y_i\left\langle \thetab - \thetab_0, \nabla^2 \Phi(\X_i; \thetab') \left(\thetab - \thetab_0 \right)\right\rangle\,.
    \end{align}
    Pick 
    \[
    \thetab := \thetab^\pare{\varepsilon} = \thetab_0 + \frac{2\Bphi + \log(1/\varepsilon)}{\gamma }\,\thetab_\star \, .
    \]    
    Substituting this  in \eqref{eq: taylor propo} and using that $\thetab_0$ is a \good initialization, we obtain for all $i\in[n]$:
    \begin{align}\nn
        y_i\Phi(\X_i; \thetab) &\geq - \abs{y_i\Phi(\X_i; \thetab_0)} + \left(2\Bphi + \log(1/\varepsilon)\right) - \frac{1}{2} \| \nabla^2 \Phi(\X_i; \thetab') \|_2 \, \| \thetab-\thetab_0 \|^2 \\
        &\geq - \Bphi + \left(2\Bphi + \log(1/\varepsilon)\right) - \frac{1}{2} \| \nabla^2 \Phi(\X_i; \thetab') \|_2 \, \left( \frac{2\Bphi + \log(1/\varepsilon)}{\gamma} \right)^2 \, . \label{eq:taylor almost}
    \end{align}
    To continue, we show in Lemma \ref{lem:multihead} in the appendix that $\| \nabla^2 \Phi(\X_i; \thetab') \|_2\leq \beta_3(\thetab')/\sqrt{H}$ where recall 
    $\beta_3(\thetab) := 2 \, d \, \sqrt{T \, d} \, R^3 \, \left(3 \, \sqrt{d} \, R^2 \, \maxnorm{\thetab} + 1 \right) \,$
    defined in Corollary \ref{coro:objective gradient/hessian_mainbody}. Now, note that
    \begin{align*}
         \beta_3(\thetab') &\leq \beta_3(\thetab^\pare{\varepsilon}) + \beta_3(\thetab_{0}) \leq \frac{2\Bphi + \log(1/\varepsilon)}{\gamma} \cdot \beta_3(\thetab_\star) + 2\beta_3(\thetab_{0})   \\
         &\leq 2\sqrt{2} \, d \, \sqrt{T \, d} \, R^3 \, \Bnorm \left( 3 \, \sqrt{d} \, R^2 + 1 \right) \cdot \left(2 + \frac{2\Bphi + \log(1/\varepsilon)}{\gamma }\right) \\
         &\leq 10 \, d \, \sqrt{T \, d} \, R^3 \, \Bnorm \left( 3 \, \sqrt{d} \, R^2 + 1 \right) \cdot \left(1 \maxi \frac{2\Bphi + \log(1/\varepsilon)}{\gamma} \right) \, ,
    \end{align*}
    where the first two inequalities follow by triangle inequality and the inequality after those follows because $\maxnorm{\thetab_\star}\leq\|\thetab_\star\|_2\leq \sqrt{2}$, $1\leq B_2$ and also $\maxnorm{\thetab_{0}}\leq\Bnorm$ by \good initialization assumption. 
    
    Plugging in this bound in \eqref{eq:taylor almost} and using the assumption on $H$ yields that $y_i\Phi(\X_i; \thetab)\geq \log(1/\eps)$ for all $i\in[n].$
    This in turn implies that $\Lhat_i(\thetab) := \ell(y_i\Phi(\X_i; \thetab)) \leq \log(1+\varepsilon) \leq \varepsilon$, and thus $\Lhat(\thetab) \leq \varepsilon$ as desired. Furthermore, note by definition of $\thetab^\pare{\eps}$ that $g_0(\eps)=\frac{2\Bphi + \log(1/\varepsilon)}{\gamma }$ and $g(\eps)=\Bnorm+g_0(\eps)\,.$ For the latter, we used triangle inequality and the rough bound $\maxnorm{\thetab_\star}\leq\|\thetab_\star\|_2\leq \sqrt{2}$\,. This completes the proof.
\end{proof}

\subsection{Proof of Corollary \ref{cor:general with good}}
We restate the corollary here for convenience, this time also including exact constants.
\begin{coro}[Restatement of Cor. \ref{cor:general with good}] \label{cor:general with good appendix}
Suppose \good initialization $\thetab_0$ and let 
\begin{align*}
    \sqrt{H} \geq 256 \, d \, \sqrt{T \, d} \, R^3 \, \Bnorm \, \left( 3 \, \sqrt{d} \, R^2 \, \bigg(4 \, g_0(\frac{1}{K}) +  \, \Bnorm \bigg) + 1 \right) \, g_0(\frac{1}{K})^2 \, ,
\end{align*}
where $g_0(\frac{1}{K}) = \frac{2\Bphi + \log(K)}{\gamma}$. Further, denote for convenience
\begin{align*}
    \alpha(K) := 3 \, d \, \sqrt{d} \, R^2 \, \left[3 \, \sqrt{T} \, R^3 \left(4 g_0(\frac{1}{K}) + \Bnorm \right) +  2\, \sqrt{T} \, R \right]
\end{align*}
and 
$
    \rho(K) = \left(\frac{2 \, d \, \sqrt{T \, d} \, R^3}{\sqrt{H}} +  \frac{T \, R^2}{4}\right)\, \alpha(K)^2.
$
Then, for any fixed step-size $$\eta \leq 1 \mini 1 / \rho(K) \mini \frac{4\Bphi^2}{\gamma^2} \cdot \frac{1}{\log(1+e^{\Bphi})}\, ,$$
the following bounds hold:
\begin{align}
    &\Lhat(\thetab_K)\leq \frac{2}{K} + \frac{5 \left( 2\Bphi + \log(K) \right)^2}{4 \gamma^2 \, \eta \, K} \, \quad\text{and}\quad
    \E\big[L(\thetab_K)-\Lhat(\thetab_K)\big] \le \frac{17 \left( 2\Bphi + \log(K) \right)^2}{\gamma^2 \, \eta \, n} \,.
\end{align}
\end{coro}
\begin{proof}
First, we prove that the given assumption on $H$ satisfies the conditions of Proposition \ref{propo:init to real} for $\varepsilon = \frac{1}{K}$.
\begin{align*}
    \sqrt{H} &\geq 256 \, d \, \sqrt{T \, d} \, R^3 \, \Bnorm \, \left( 3 \, \sqrt{d} \, R^2 \, \bigg(4 \, g_0(\frac{1}{K}) +  \, \Bnorm \bigg) + 1 \right) \, g_0(\frac{1}{K})^2 \\ &= 256 \, d \, \sqrt{T \, d} \, R^3 \, \Bnorm \, \left( 3 \, \sqrt{d} \, R^2 \, \bigg(3 \, g_0(\frac{1}{K}) +  \, g(\frac{1}{K}) \bigg) + 1 \right) \, g_0(\frac{1}{K})^2 \, .
\end{align*}
If $1 > g_0(\frac{1}{K})$,
\begin{align*}
    \sqrt{H} &\geq 256 \, d \, \sqrt{T \, d} \, R^3 \, \Bnorm \, \left( 3 \, \sqrt{d} \, R^2 \, \bigg(4 \, g_0(\frac{1}{K}) + \, \Bnorm \bigg) + 1 \right) \, g_0(\frac{1}{K})^2 \\ &\geq 5 \, d \, \sqrt{T \, d} \, R^3 \, \Bnorm \, \left( 3 \, \sqrt{d} \, R^2 + 1 \right) \, g_0(\frac{1}{K})^2 \, .
\end{align*}
It means that the condition of Proposition \ref{propo:init to real} on $\sqrt{H}$ is satisfied. Moreover, if $1 \leq g_0(\frac{1}{K})$,
\begin{align*}
    \sqrt{H} &\geq 256 \, d \, \sqrt{T \, d} \, R^3 \, \Bnorm \, \left( 3 \, \sqrt{d} \, R^2 \, \bigg(4 \, g_0(\frac{1}{K}) +  \, \Bnorm \bigg) + 1 \right) \, g_0(\frac{1}{K})^2 \\ &\geq 256 \, d \, \sqrt{T \, d} \, R^3 \, \Bnorm \, \left(12 \, \sqrt{d} \, R^2 \, g_0(\frac{1}{K}) \right) \, g_0(\frac{1}{K})^2 \\ &\geq 5 \, d \, \sqrt{T \, d} \, R^3 \, \Bnorm \, \left( 3 \, \sqrt{d} \, R^2 + 1 \right) \, g_0(\frac{1}{K})^3 \, ,
\end{align*}
and again the condition of Proposition \ref{propo:init to real} on $\sqrt{H}$ is satisfied. Then, we can apply the results of Corollaries \ref{coro:train interpol} and \ref{coro:gen interpol} for a fixed $K$ which satisfies $K \geq 1$. Note that
\begin{align*}
    g_0(1) = \frac{2\Bphi}{\gamma} \, , \quad \text{and} \quad \Lhat(\thetab_0) \leq \log(1 + e^{\Bphi}) \, .
\end{align*}
Thus, the condition on step-size simplifies to $\eta \leq 1 \mini 1 / \rho(K) \mini \frac{4\Bphi^2}{\gamma^2} \cdot \frac{1}{\log(1+e^{\Bphi})}\,$. This completes the proof. 
\end{proof}

\section{Proofs for Section \ref{sec:main dm1}}

\subsection{Useful facts}
\begin{fact}\label{propo:subgn}
    Let $\x\in\R^d$ be subgaussian vector with $\tsub{\x}\leq K$. Then, for any $\delta\in(0,1)$ and absolute constant $C>0$ it holds with probability at least $1-\delta$ that 
    $\|\x\|_2 \leq CK\big(\sqrt{d}+\sqrt{\log(1/\delta)}\big)$\,.
\end{fact}

\begin{fact}\label{propo:Hoeffding}
    Suppose $X_h, h\in[H]$ are \iid realizations of random variable $X$ for which $E[X]=\mu$ and $|X|\leq B$ almost surely. Then, for any $\delta\in(0,1)$, with probability at least $1-\delta$ it holds that
    \[
    \abs{\frac{1}{H}\sum_{h\in[H]}{X_h}-\mu}\leq 2B\sqrt{\frac{\log(1/\delta)}{2H}}\,.
    \]
\end{fact}

\subsection{Proof of Proposition \ref{propo:good init dm1}}
Recall 
\begin{align*}
    \Phit(\X; \thetabt) = \frac{1}{\sqrt{H}} \sum_{h\in[H]} \Phi_h(\X; \W_h, \Ub_h) =  \frac{1}{\sqrt{H}}\sum_{h\in[H]}\inp{\Ub_h}{\sft{\X\W_h\X^\top}\X}\,, 
\end{align*}
where $\thetabt =  \concat({\thetab_1, \thetab_2, . . ., \thetab_H})$ denotes the trainable parameters. After completing the first phase of training, the initialization for the second phase is as follows for all $h\in[H]$: 
\begin{align} \label{eq:init}
    \thetab_h^\pare{1} = \concat(\Ub_{h}^\pare{1} , \W_{h}^\pare{1}) \quad : \quad \W_{h}^{\pare{1}} = \mathbf{0}, \quad  \Ub_{h}^{\pare{1}} = \alpha_h \Bigg(\frac{\zeta}{2} \ones_T \ub_\star^{\top} + \ones_T\pb^{\top} \Bigg) \, , 
\end{align}
where recall that $\alpha_h\sim\Unif{\pm1}$ and from Lemma \ref{lem:first phase} it holds with probability $1-\delta$ that $\|\pb\|\leq P$ where the parameter $P$ is defined in \eqref{eq:p bound}. Onwards, we condition on this good event.

We prove the three properties $\textbf{P1}, \textbf{P2}$, and $\textbf{P3}$ in the order stated below.


\subsubsection{Proof of \textbf{P1}: Bounded norm per head}

This is straightforward by noting that for all $h\in[H]$:
\begin{align*}
    \|\thetab_h^\pare{1}\|_2 = \|\Ub_h^\pare{1}\|_F \leq \frac{\zeta}{2}\sqrt{T}\|\ub_\star\|_2+\sqrt{T}\|\pb\|_2 \leq \frac{\zeta}{2}\sqrt{T}\sqrt{2}S+\sqrt{T}P\,.
\end{align*}

\subsubsection{Proof of \textbf{P2}: Bounded initialization}
\begin{lemma}[Initialization bound]
\label{lem:init bound}
Let any $\X$ sampled from \ref{model1} and satisfying Assumption \ref{ass:data}. Given the initialization in \eqref{eq:init}, for any  $\delta \in (0,1)$ it holds with probability at least $1-\delta$ that
\end{lemma}
\[
|\Phit(\X; \thetabt^\pare{1})| \leq T\,R (S + P)\sqrt{2\log(1/\delta)}\,.
\]
\begin{proof}
\begin{align*}
    \Phit(\X; \thetabt^\pare{1}) = \frac{1}{\sqrt{H}} \sum_{h\in[H]} \Phi_h(\X; \thetab_h^\pare{1}) = \frac{1}{\sqrt{H}} \sum_{h\in[H]} \inp{\Ub_h^\pare{1}}{\sft{\X\W_{h}^\pare{1}\X^\top}\X}.
\end{align*}
Using the initialization in \eqref{eq:init} and recalling $\sft{\mathbf{0}}=\ones_T \ones_T^\top / T$, we have  
\begin{align}\label{eq:init bound hoeffding}
\frac{1}{\sqrt{H}} \sum_{h\in[H]} \inp{\Ub_h^\pare{1}}{\sft{\X\W_{h}^\pare{1}\X^\top}\X} &= \frac{1}{\sqrt{H}}\sum_{h\in[H]} \frac{\alpha_h}{T} \inpb{\Ub_h^\pare{1}}{\ones_T \ones_T^\top \X} =: \frac{1}{\sqrt{H}}\sum_{h\in[H]}X_h \,.
\end{align}
Note for each $h\in[H]$ that $X_h$ in \eqref{eq:init bound hoeffding} depends only on the random variable $\alpha_h$. Recall that $\alpha_h, h\in[H]$ are \iid $\Unif{\pm1}$. Thus, $\{X_h\}_{h\in[H]}$ are \iid with $0$ mean as $\E \alpha = 0$. Further, note that
\begin{align*}
    \abs{X_h} &= \abs{\frac{\alpha}{T} \inpb{\Ub^\pare{1}}{\ones_T\ones_T^\top \X}} = \abs{\frac{\alpha}{T} \inpb{\frac{\zeta}{2}\ones_T \ub_\star^\top + \ones_T\pb^\top}{\ones_T\ones_T^\top \X}} \\ &\leq \frac{1}{T}\bigg( \frac{\zeta}{2}\sqrt{T}\norm{\ub_\star}+\sqrt{T}\norm{\pb}\bigg)\sqrt{T}\norm{\sum_{t \in [T]}\x_t} \leq \frac{1}{T}\bigg(\frac{\zeta}{2}\sqrt{T}\norm{\ub_\star}+\sqrt{T}\norm{\pb}\bigg)\sqrt{T}RT\leq TR(S+P).
\end{align*} 
Thus, using Hoeffding's inequality (see Fact \ref{propo:Hoeffding}) we have for some absolute constant $c>0$, with probability at least $1-\delta$:
\[
|\Phit(\X; \thetabt^\pare{1})| \leq TR (S + P)\sqrt{2\log(1/\delta)}.
\]
\end{proof}

\subsubsection{Proof of \textbf{P3}: NTK separability}\label{sec:ntk separability}

We prove property \textbf{P3} in two steps each stated in a separate lemma below

\begin{lemma}\label{lem: singlehead ntk}
Let 
\begin{align}\label{eq:W not normalized}
\Wstar&=\mub_+\mub_+^{\top} + \mub_{-}\mub_{-}^{\top} + \sum_{\ell\in[M]}\nub_\ell(\mub_++\mub_{-})^{\top}, \\
\Ub_\star &= \ones_T\ub_\star^{\top}.
\end{align}
Given the initialization $\thetab^\pare{1}$ in \eqref{eq:init} and $\thetab_\star := (\overline{\Ub}_\star, \, \sign{\alpha}\overline{\W}_\star)$ we have for any $(\X,y)$ sampled from \ref{model1} and satisfying Assumption \ref{ass:data}:
    \[
    \E_{\thetab^\pare{1}}\,y\inp{\nabla_{\thetab}\Phi(\X; \thetab^\pare{1})}{\thetab_\star} \geq \gamma_\star,
    \]    
where the expectation is taken over $\alpha \sim \Unif{\pm1}$ and
\begin{align*} 
\gamma_\star &:= \frac{T(1-\zeta)\zeta}{4\sqrt{2(M+1)}}\bigg(\zeta S^4-7\zbar S^2-12\zbar^2 -16\frac{\zbar^3}{S^2}\bigg) -PT^{5/2}(S+Z)^3 +  \frac{S\,\sqrt{T}}{\sqrt{2}}\, \Big(\zeta- 2(1-\zeta) \frac{\zmu}{S^2}\Big), \\
& \text{where} \, \, \zbar = \zmu \maxi \znu, \, \, \text{and} \, \,  \overline{\Ub}_\star, \overline{\W}_\star \, \,  \text{denote normalized} \, \, \, \Ub_\star, \W_\star, \, \text{respectively}. 
\end{align*}    
\end{lemma}
\begin{proof}
\begin{align*}
\inp{\nabla_{\thetab}\Phi(\X; \thetab^\pare{1})}{\thetab_\star} = \sign{\alpha}\inp{\nabla_{\W}\Phi(\X; \thetab^\pare{1})}{\overline{\W}_\star} + \inp{\nabla_{\Ub}\Phi(\X; \thetab^\pare{1})}{\overline{\Ub}_\star} \, .
\end{align*}
Using $\thetab^\pare{1}=\concat(\Ub^\pare{1},0)$ and $\overline{\Ub}_\star = \frac{\Ub_\star}{\norm{\Ub_\star}_F} = \frac{1}{S\sqrt{2T}}\ones_T \ub_\star^{\top}$, we have:
\begin{align}\label{eq: margin-U}
    y\,\inp{\nabla_{\Ub}\Phi(\X; \thetab^\pare{1})} {\overline{\Ub}_\star} &= y\,\inp{\sft{\X \W^{\pare{1}} \X^\top}\X}{\overline{\Ub}_\star} = \frac{y}{TS\sqrt{2T}}\inp{\ones_T\ones_T^{\top}\X}{\ones_T \ub_\star^{\top}} \nn \\ 
    &=\frac{y}{S\sqrt{2T}}\ub_\star^{\top}\left(\sum_{t}\x_t\right) \geq \frac{S\,\sqrt{T}}{\sqrt{2}}\, \Big(\zeta- 2(1-\zeta) \frac{\zmu}{S^2}\Big) \, .
\end{align}
The gradient with respect to $\W$ evaluated at $\thetab^\pare{1}=\concat(\Ub^\pare{1},0)$ is 
\begin{align}
\nabla_{\W}\Phi(\X; \thetab^\pare{1}) = \frac{\alpha\zeta}{2}\sum_{t \in [T]}\x_t\ub_\star^\top\X^{\top}\Rb_{t}\X + \alpha \sum_{t \in [T]}\x_t\pb^\top\X^{\top}\Rb_{t}\X\,,\nn
\end{align}
where 
\[\Rb_t =  \Rb_0:= \frac{1}{T} \cdot \Iden_T - \frac{1}{T^2} \cdot \ones_T \ones_{T}^{\top}, \,\, \forall t \in [T].
\]
Thus,
\begin{align} \label{eq: grad w ntk}
    &y\inp{\nabla_{\W} \Phi\left(\X;\thetab^\pare{1} \right)}{ \sign{\alpha}\overline{\W}_\star} \nn \\
    &\qquad \qquad = \frac{\abs{\alpha}}{\norm{\Wstar}_F}\bigg(\frac{\zeta}{2} \underbrace{y\sum_{t \in [T]}\ub_\star^{\top}\X^{\top}\Rb_0 \rb_t(\X;\Wstar)}_{\term{I}} +  \underbrace{y\inpb{\sum_{t \in [T]}\x_t\pb^\top\X^{\top}\Rb_0\X}{\Wstar}}_{\term{II}}\bigg),
\end{align}
where we set for convenience:
\[
\rb_t(\X;\Wstar) = \X{\Wstar}^\top\x_t\in\R^T\,.
\]
To compute $\rb_t(\X;\Wstar)$, we consider two cases where row corresponds to a signal relevant of noisy token. We denote the $t'\in[T]$ entry of $\rb_t$ as $[\rb_t]_{t'}\in\R.$

\noindent\underline{Case 1. Relevance scores of signal tokens:}
Consider signal token $t\in\Rc$ so that $\x_t=\mub_y$. Using orthogonality in Assumption \ref{ass:data} we can compute for all $t'\in[T]$
\begin{align} \label{eq:relevance W signal exact}   t \in \Rc: [\rb_t]_{t'} =
\begin{cases}
    S^4 & \text{,} \, t' \in \Rc\,, \\
    S^2(\mub_y^{\top}\z_{t'}) & \text{,} \, t' \in \Rc^c\,. \\
\end{cases}
\end{align}
Therefore, again using Assumption \ref{ass:data},
\begin{align} \label{eq:relevance W signal}   
t \in \Rc: [\rb_t]_{t'} 
\begin{cases}
    =S^4 & \text{,} \, t' \in \Rc\,, \\
    \leq S^2\zmu & \text{,} \, t' \in \Rc^c\,. \\
\end{cases}
\end{align}

\noindent\underline{Case 2. Relevance scores of noisy tokens:} Similar to the calculations above, using Assumption \ref{ass:data} for parameters $\W^\star$ as in \eqref{eq:W not normalized} it holds for noisy tokens $t\in\Rcc$ that
\begin{align}\label{eq:relevance W noise exact}
t \in \Rc^c: [\rb_t]_{t'} =
\begin{cases}
    S^4 + S^2(\mub_y^{\top}\z_{t}) + S^2(\sum_\ell \nub_\ell^{\top}\z_t) & \text{,} \, t' \in \Rc \\
    (\mub_{+}^{\top}\z_{t})(\mub_{+}^{\top}\z_{t'}) +(\mub_{-}^{\top}\z_{t})(\mub_{-}^{\top}\z_{t'}) & \text{,} \, t' \in \Rc^c \\ + \sum_\ell (\nub_\ell^{\top}\z_t)(\mub_{+}^{\top}\z_{t'}) 
    + \sum_\ell (\nub_\ell^{\top}\z_t)(\mub_{-}^{\top}\z_{t'}) + S^2(\mub_{+}+\mub_{-})^{\top}\z_{t'}\\
\end{cases}
\end{align}
Now, we can start to bound each of the two terms in \eqref{eq: grad w ntk} separately below.

\noindent\underline{Bounding $\term{I}$:}
Recall the data matrix complies with Assumption \ref{ass:data}, hence.
\begin{align}    
[\X \ub_\star]_{t'} =
\begin{cases}
    y S^2 & \text{,} \, t' \in \Rc \\
    \z_{t'}^{\top} (\mubp - \mubn) & \text{,} \, t' \in \Rc^c \\
\end{cases} \, .
\end{align}
Using this, we can compute:
\begin{align*}
    \frac{y}{T} \sum_{t \in [T]}\ub_\star^{\top}\X^{\top}\rb_t(\X;\W^*) &= \frac{y}{T}\Biggl\{\sum_{t \in \Rc}\ub_\star^{\top}\X^{\top} \rb_t(\X;\Wstar) + \sum_{t \in \Rc^c}\ub_\star^{\top}\X^{\top}\rb_t(\X;\Wstar) \Biggr\} \\
    &= \frac{y}{T}\Biggl\{\sum_{t \in \Rc} \left(\sum_{t' \in \Rc}[\X\ub_\star]_{t'} \, \rb_t(\X;\Wstar)_{t'} + \sum_{t' \in \Rc^c}[\X\ub_\star]_{t'} \, \rb_t(\X;\Wstar)_{t'}\right) \\
    &\quad \;\; + \sum_{t \in \Rc^c}\left(\sum_{t' \in \Rc}[\X\ub_\star]_{t'} \, \rb_t(\X;\Wstar)_{t'} + \sum_{t' \in \Rc^c}[\X\ub_\star]_{t'} \, \rb_t(\X;\Wstar)_{t'}\right)\Biggr\} \\
    &= \frac{y}{T}\Biggl\{\zeta T \left(\zeta T \, yS^2 \, S^4 + \sum_{t' \in \Rc^c}\z_{t'}^{\top}(\mubp - \mubn) \, S^2(\mub_y^{\top}\z_{t'})\right) \\ 
    &\quad \;\;+ \sum_{t \in \Rc^c}\bigg[\zeta TyS^2 \cdot \left(S^4 + S^2(\mub_y^{\top}\z_{t}) + S^2(\sum_{j_t} \nub_{j_t}^{\top}\z_t)\right) \\
    &\quad \;\;+ \sum_{t' \in \Rc^c} \z_{t'}^{\top}(\mubp - \mubn) \cdot \bigg( (\mubp^{\top}\z_{t}) (\mubp^{\top}\z_{t'}) +(\mubn^{\top}\z_{t})(\mubn^{\top}\z_{t'}) \\
    &\quad \;\;+ \; \, \sum_{j_t}  (\nub_{j_t}^{\top}\z_t)(\mubp^{\top}\z_{t'}) + \sum_{j_t} (\nub_{j_t}^{\top}\z_t)(\mubn^{\top}\z_{t'}) + S^2(\mubp+\mubn)^{\top}\z_{t'}\bigg)\bigg]\Biggr\}.
\end{align*}
Further using {the noise bounds in} Assumption \ref{ass:data}, and $\rb_t(\X; \Wstar)$ from \eqref{eq:relevance W signal}, \eqref{eq:relevance W noise exact} we have:
\begin{align}\nn
    &\frac{y}{T} \sum_{t \in [T]} \ub_\star^{\top} \X^{\top} \rb_t(\X;\Wstar) \\ \nn
    &\geq \zeta T\, \bigg[ \zeta S^6 - 2 (1-\zeta) \zmu^2 S^2 \bigg] + (1-\zeta)T \, \bigg[\zeta S^6-\zeta (\zmu + \znu) S^4 - 2 \zmu (1-\zeta) (2\zmu^2 + 2\zmu \znu + 2 \zmu S^2) \bigg] \, 
\\
\label{eq:gamma-term11}
    &\geq {T}  \bigg[ \zeta \, S^6 - \, \zeta \, (1-\zeta) \, (\zmu+\znu) \, S^4 - 2(1-\zeta) \,   (\zeta + 2(1-\zeta)) \, \zmu^2\, S^2 - 4 \, (1-\zeta)^2 \, \zmu^2 \, (\zmu + \znu) \bigg] \, .
\end{align}
For the second part of $\term{I}$:
\begin{align*}
    - &\frac{y}{T^2} \sum_{t \in [T]} \ub_\star^{\top} \X^{\top} \ones_T \ones_T^{\top} \rb_t(\X;\Wstar) = - \frac{y}{T^2} \sum_{t \in [T]} \ones_T^\top \X \ub_\star \, \ones_T^\top \rb_t(\X;\Wstar) \\ & \quad= - \frac{1}{T^2} \left\{ \sum_{t' \in \Rc} y^2 S^2 + \sum_{t' \in \Rc^c} y \z_{t'}^{\top} (\mubp - \mubn) \right\} \cdot \Bigg\{ \sum_{t \in \Rc} \left[ \sum_{t' \in \Rc} S^4 + \sum_{t' \in \Rc^c} S^2 \z_{t'}^\top \mub_y \right] \\ &\qquad \;\;\; + \sum_{t \in \Rc^c} \bigg[ \sum_{t' \in \Rc} \left(S^4 + S^2 \z_{t}^\top \mub_y + S^2 \sum_{j_t} \nub_{j_t}^\top \z_t \right) + \sum_{t' \in \Rc^c} \bigg( (\z_{t}^{\top}\mubp)(\z_{t'}^{\top}\mubp) + (\z_{t}^{\top}\mubn)\mubn) \\ &\qquad \;\;\;  + (\sum_{j_t}\nub_{j_t}^{\top}\z_t) (\z_{t'}^{\top} \mubp) + (\sum_{j_t} \nub_{j_t}^{\top}\z_t) (\z_{t'}^{\top} \mubn) + S^2 \z_{t'}^\top (\mubp+\mubn) \bigg) \bigg] \Bigg\} \, .
\end{align*}
Using Assumption \ref{ass:data} to simplify the second term:
\begin{align*}
    &- \frac{y}{T^2} \sum_{t \in [T]} \ub_\star^{\top} \X^{\top} \ones_T \ones_T^{\top} \rb_t(\X;\Wstar) \\ &\qquad \geq T \bigg[ -\zeta^3 S^6 - \zeta^2 (1-\zeta) \zmu S^4 - \zeta^2 (1-\zeta) S^6 - \zeta^2 (1-\zeta) (\zmu + \znu) S^4 - 2 \zeta (1-\zeta)^2 \zmu S^4 \\ & \qquad \qquad \, \, - 2 \zeta (1-\zeta)^2 \zmu (\zmu + \znu)S^2 - 2 \zeta^2 (1-\zeta) \zmu S^4 - 2 \zeta (1-\zeta)^2 \zmu^2 S^2 - 2 \zeta (1-\zeta)^2 \zmu S^4\\ & \qquad \qquad \, \,   - 2 \zeta (1-\zeta)^2 \zmu (\zmu + \znu) S^2 - 4 (1-\zeta)^3 \zmu^2 S^2 - 4 (1-\zeta)^3 \zmu^2 (\zmu + \znu)  \bigg] \, .
\end{align*}
Therefore,
\begin{align}\label{eq:gamma-term12}
     -\frac{y}{T^2} \sum_{t \in [T]} \ub_\star^{\top} \X^{\top} \ones_T \ones_T^{\top} \rb_t(\X;\Wstar) &\geq T\bigg[-\zeta^2S^6 -\zeta(1-\zeta)(4\zmu+\zeta\znu)S^4 - 2(1-\zeta)^2((2+\zeta)\zmu\nn \\ &\qquad+2\zeta\znu)\zmu S^2 -4(1-\zeta)^3\zmu^2(\zmu+\znu)\bigg] \, .
\end{align}
\noindent\underline{Bounding $\term{II}$:}
\begin{align}\label{eq:gamma-p}
    y\inp{\sum_{t \in [T]}\x_t\pb^\top\X^{\top}\Rb_0\X}{\Wstar} &\geq -\sum_{t \in [T]}\norm{\x_t}\norm{\pb}\norm{\X^{\top}\Rb_0\X}_F\norm{\Wstar}_F \nn \\
    &\geq-P(\zeta TS+(1-\zeta)T(S+Z))\norm{\Rb_0}_F\norm{\X}_F^2S^2\sqrt{2(M+1)}\nn \\
    &\geq-\sqrt{2(M+1)}S^2PT^{5/2}(\zeta S+(1-\zeta)(S+Z))(\zeta S^2 + (1-\zeta)(S+Z)^2) \, .
\end{align}
Combining \eqref{eq:gamma-term11}, \eqref{eq:gamma-term12}, \eqref{eq:gamma-p} in \eqref{eq: grad w ntk} we get: 

\begin{align} \label{eq:gamma-attn}
&y\inp{\nabla_{\W}\Phi(\X; \thetab^\pare{1})}{\Wstar} \nn \\ 
&\geq \alpha\Bigg\{
\frac{T\zeta}{2} \bigg[ \zeta \, S^6 - \, \zeta \, (1-\zeta) \, (\zmu+\znu) \, S^4 - 2(1-\zeta) \,   (\zeta + 2(1-\zeta)) \, \zmu^2\, S^2 - 4 \, (1-\zeta)^2 \, \zmu^2 \, (\zmu + \znu) \bigg] \nn \\ &\qquad - \frac{T\zeta}{2}\bigg[\zeta^2S^6 +\zeta(1-\zeta)(4\zmu+\zeta\znu)S^4 + 2(1-\zeta)^2((2+\zeta)\zmu+2\zeta\znu)\zmu S^2 +4(1-\zeta)^3\zmu^2(\zmu+\znu)\bigg] \nn \\ &\qquad-\sqrt{2(M+1)}S^2PT^{5/2}(\zeta S+(1-\zeta)(S+Z))(\zeta S^2 + (1-\zeta)(S+Z)^2) \Bigg\} \nn \\
&= \alpha \Bigg\{\frac{T(1-\zeta)\zeta}{2}\bigg[\zeta S^6 -\zeta (5\zmu+(1+\zeta)\znu)S^4-2((4-\zeta^2)\zmu + 2\zeta(1-\zeta)\znu)\zmu S^2 \nn \\&\qquad-4(1-\zeta)(2-\zeta)(\zmu+\znu)\zmu^2\bigg]-\sqrt{2(M+1)}S^2PT^{5/2}(S+(1-\zeta)Z)(\zeta S^2 + (1-\zeta)(S+Z)^2)\Bigg\} \nn \\
&\geq \alpha \Bigg\{\frac{T(1-\zeta)\zeta}{2}\bigg[\zeta S^6-(5\zmu+2\znu)S^4-4(2\zmu+\znu)\zmu S^2-8(\zmu+\znu)\zmu^2\bigg] \nn \\ &\qquad -\sqrt{2(M+1)}S^2PT^{5/2}(S+Z)(S^2+(S+Z)^2)\Bigg\} \nn \\
&\geq \alpha \bigg\{\frac{T(1-\zeta)\zeta}{2}\bigg(\zeta S^6-7\zbar S^4-12\zbar^2 S^2-16\zbar^3\bigg) -2\sqrt{2(M+1)}PT^{5/2}S^2(S+Z)^3\bigg\},
\end{align}
where $\zbar = \zmu \, \maxi \, \znu$. Using this, we get
\begin{align} 
\mu:&= \E_{\alpha \sim \Unif{\pm1}}y\inpb{\nabla_{\W}\Phi(\X; \thetab^\pare{1})}{\frac{\sign{\alpha}\Wstar}{\norm{\Wstar}_F}} \nn \\
&\geq \E_{\alpha \sim \Unif{\pm1}}\frac{\abs{\alpha}}{\norm{\Wstar}} \bigg\{\frac{T(1-\zeta)\zeta}{2}\bigg(\zeta S^6-7\zbar S^4-12\zbar^2 S^2-16\zbar^3\bigg) -2\sqrt{2(M+1)}PT^{5/2}S^2(S+Z)^3\bigg\} \nn \\
&\geq \frac{T(1-\zeta)\zeta}{4\sqrt{2(M+1)}}\bigg(\zeta S^4-7\zbar S^2-12\zbar^2 -16\frac{\zbar^3}{S^2}\bigg) -PT^{5/2}(S+Z)^3  \, . \label{eq: margin-W}
\end{align}
Combining \eqref{eq: margin-U} and \eqref{eq: margin-W} concludes the proof.
\end{proof}

\begin{lemma}[NTK separability]\label{lem:empirical margin general} 
Assume initialization $\thetabt^\pare{1}=\concat\big(\thetab_1^\pare{1},\ldots,\thetab_1^\pare{H}\big)$ as in \eqref{eq:init} and \iid $\alpha_h\sim\Unif{\pm1}$ for all $h\in[H]$. Recall $\gammastar$ and $\thetab_\star(\cdot)$ from Lemma \ref{lem: singlehead ntk} above. Set \[\thetabt_\star=\frac{1}{\sqrt{H}}\concat\left(\thetab_\star\big(\thetab_1^\pare{1}\big),\ldots,\thetab_\star\big(\thetab_H^\pare{1}\big)\right).\] Let any $(\X,y)$ from \ref{model1}. Then, with probability at least $1-\delta$ over the randomness of $\alpha_h, h\in[H]$ it holds
\[
y\left\langle\nabla \Phit\left( \X; \thetabt^\pare{1} \right), \thetabt_\star\right\rangle \geq \gammastar - 2\left(2R^3T(S+P)+\sqrt{T}R\right)\sqrt{\frac{2\log(1/\delta)}{H}} \, .
\]
\end{lemma}
\begin{proof}
We start by expanding the empirical margin:
\begin{align}\label{eq:expand empirical margin}
y\left\langle \nabla \Phit\left( \X; \thetabt^\pare{1} \right), \thetabt_\star\right\rangle
&=  \frac{1}{H} \sum_{h \in [H]} y \inp{\nabla_{\thetab} \Phi\left(\X; \thetab_{h}^\pare{1} \right)}{\thetab_\star(\thetab_h^\pare{1})} =:
\frac{1}{H}\sum_{h \in [H]} X_h \,.
\end{align}
Note that each summand $X_h$ defined above depends only on $\alpha_h, h\in[H]$. Thus,  $\{X_h\}_{h\in[H]}$ are \iid random variables because $\{\alpha_h\}_{h\in[H]}$ are \iid random variables. Moreover, $X_h$ is almost-surely bounded satisfying
\begin{align}
|X_h|\leq \|\nabla_{\thetab} \Phi\left(\X; \thetab_{h}^\pare{1} \right)\| \|\thetab_\star(\thetab_h^\pare{1})\| \leq \sqrt{2T} \, R \, \left(2 \, R^2 \, \norm{\thetab_h^\pare{1}}_F + 1 \right) \leq 2 \, (2R^3T(S+P)+\sqrt{T}R) \, .    
\end{align}

where the second inequality follows by Proposition \ref{propo:model bounds general} and by the assumption $\|\thetab_\star(\thetab_h^\pare{1})\|_2=\sqrt{2}$ for all $h\in[H]$, and the last inequality follows because $\|\thetab_{h}^\pare{1}\|\leq \frac{\zeta}{2}\sqrt{T}\|\ub_\star\|_2+\sqrt{T}\|\pb\|_2\leq \sqrt{T}(S+P)$.

Finally, note that $\E[X_h]\geq \gammastar$ from Lemma \ref{lem: singlehead ntk}. 
Given these the desired claim follows by applying Hoeffding's inequality (see Fact \ref{propo:Hoeffding}) to \eqref{eq:expand empirical margin}.

\end{proof}

\begin{lemma}[Margin] \label{lem: margin whp} 
Define $\zbar := \zmu \maxi \znu$ and 
\begin{align}
    \gammastar := \frac{T(1-\zeta)\zeta}{4\sqrt{2(M+1)}}\big(\zeta S^4-7\zbar S^2-12\zbar^2 -16\frac{\zbar^3}{S^2}\big) -PT^{5/2}(S+Z)^3 + \frac{S\,\sqrt{T}}{\sqrt{2}}\, \Big(\zeta- 2(1-\zeta) \frac{\zmu}{S^2}\Big)\,.
\end{align}
Suppose 
\[
    \sqrt{H} \geq 4 \cdot \frac{2R^3T(S+P)+\sqrt{T}R}{\gammastar}\cdot\sqrt{2\log(n/\delta)}\,.
\]
Then, with probability $1-\delta\in(0,1)$,
\textbf{P3} holds with $\gamma=\gammastar/2$.
\end{lemma}
\begin{proof}
    The proof is straightforward by using union bound and plugging the condition of $H$ in the result of Lemma \ref{lem:empirical margin general}.
\end{proof}

\subsection{Proof of Lemma \ref{lem:first phase}}

The proof of the lemma follows directly by      combining the two lemmas below and using $\zeta\leq 1$.

\begin{lemma}\label{lem:W1 U1 formulas} Fix any $h\in[H]$. Suppose zero initialization $\thetab_{h}^\pare{0}=0$ and consider first gradient step as in \eqref{eq:init}. 
It then holds that $\W_h^\pare{1} = 0$ and $\Ub_h^\pare{1}\in\R^{T\times d}$ has identical rows all equal to 
\begin{align}\label{eq:U1 formula}
\frac{\zeta\alpha_h}{2}\ub_\star+\frac{\zeta\alpha_h}{2}\underbrace{\Big(\frac{1}{n}\sum_{i\in[n_1]}y_i\mub_{y_i} - \ub_\star\Big)}_{\pb_1} + \frac{\alpha_h(1-\zeta)}{2}\underbrace{\frac{1}{n}\sum_{i\in[n_1]}y_i\frac{1}{(1-\zeta)T}\sum_{t\in\Rcc_i}(\nub_{j_t}+\z_{i,t})}_{\pb_2} \, .
\end{align}
where recall that $\ub_\star=\mub_+-\mub_-$. 
\end{lemma}
\begin{proof} We start by computing
    \begin{align*}
    n\,\nabla_{\thetab_h}\Lhat(\thetab_{h}^\pare{0})
    &= \sum_{i\in[n_1]}\nabla_{\thetab_h}\ell(y_i\Phit(\X;\thetab_{h}^\pare{0}))
    = \sum_{i\in[n_1]}y_i\ellp(y_i\Phit(\X;\thetab_{h}^\pare{0}))\nabla_{\thetab_h}\Phit(\X;\thetab_{h}^\pare{0})
    \\
    &=  \ellp(0)\,\sum_{i\in[n_1]}y_i\nabla_{\thetab_h}\Phit(\X;\thetab_{h}^\pare{0})
    = \frac{\ellp(0)}{\sqrt{H}}\,\sum_{i\in[n_1]}y_i\nabla_{\thetab_h}\Phi_h(\X;\thetab_{h}^\pare{0}) \\
    &= \frac{1}{2\sqrt{H}}\,\sum_{i\in[n_1]}y_i\nabla_{\thetab_h}\Phi_h(\X;\thetab_{h}^\pare{0})
    \end{align*}
where we used that $\Phi_h(\X;\thetab_{h}^\pare{0})=\Phit_h(\X;\thetab_{h}^\pare{0})=0$ because $\Ub_h^{(0)}=\mathbf{0}$ and also $\ellp(0)=1/2$ for the logistic loss.  Now, recall that $\Phi_h(\X;\theta_h)=\inp{\Ub_h}{\attn_h(\X;\W_h)}$. Hence, $\nabla_{\W_h}\Phi_h(\X;\thetab_{h}^\pare{0})=\mathbf{0}$, which gives us $\W_h^\pare{1}=\W_h^\pare{0}$. Also, 
\[
\nabla_{\Ub_h}\Phi_h(\X;\thetab_{h}^\pare{0}) = \attn(\X;\W_{h}^\pare{0}) = \sft{\mathbf{0}}\X = \frac{1}{T}\ones_T \ones_T^\top \X \,.
\]
Hence, the $\tau$-th column of $\Ub_h^\pare{1}$ becomes for all $\tau\in[T]$:
\begin{align*}
[\Ub_h^\pare{1}]_{:,\tau} &= \frac{1}{2nT}{\alpha_h}\sum_{i\in[n_1]}y_i\sum_{t\in[T]}\x_{i,t}
\\
&= \frac{\zeta}{2n}{\alpha_h} \sum_{i\in[n_1]}y_i \mub_{y_i} + \frac{1}{2nT}{\alpha_h} \sum_{i\in[n_1]} y_i \sum_{t\in\Rcc_i} (\nub_{j_t} + \z_{i,t}) \, .
\end{align*}
The claim of the lemma follows by rearranging the above.
\end{proof}

\begin{lemma}
    Suppose labels are \iid and equal probable, i.e. $y_i\sim\Rad{\pm1}$. Then for the two terms $\pb_1, \pb_2$ in \eqref{eq:U1 formula} it holds with probability at least $1-2\delta$ and absolute constant $C>0$ over the randomness of labels that
    \begin{align}\label{eq:p1 bound}
        \|\pb_1\|\leq CS\Big(\sqrt{\frac{d}{n_1}} + \sqrt{\frac{\log(1/\delta)}{n_1}}\Big)\,,
    \end{align}
    and 
    \begin{align}\label{eq:p2 bound}
        \|\pb_2\|\leq C(S+Z)\Big(\sqrt{\frac{d}{n_1}} + \sqrt{\frac{\log(1/\delta)}{n_1}}\Big)\,.
    \end{align}
\end{lemma}

\begin{proof}
    For arbitrary $\|\vb\|=1$, let $X_v=\inp{\vb}{\frac{1}{n}\sum_{i\in[n_1]}y_i\mub_{y_i}}$.
 Then,
    \begin{align*}
        \tsub{X_v} = \frac{1}{n}\tsub{\sum_{i\in[n_1]}y_i\inp{\vb}{\mub_{y_i}}} \leq \frac{C}{n}\sqrt{\sum_{i\in[n_1]}\tsub{y_i\inp{\vb}{\mub_{y_i}}}^2} \leq \frac{C\,S}{\sqrt{n_1}} \, .
    \end{align*}
    where in the second inequality we used approximate rotation invariance of subgaussians and in the last step we used that for all $i\in[n_1]$, $|y_i\inp{\vb}{\mub_{y_i}}|\leq S$, thus they are  subgaussians with parameter $CS$. Further note that Note that $\E[X_v]=\inp{\vb}{\ub_\star}$. Thus, by centering property of subgaussians $\inp{\vb}{\pb_1}$ is also subgaussian with same constant $CS/\sqrt{n_1}$. Since this holds for all $\vb$ on the sphere, we conclude that $\pb_1$ is $CS/\sqrt{n_1}$-subgaussian.
 From this, we can directly apply Fact \ref{propo:subgn} for concentration of Euclidean norm of random vectors to arrive at \eqref{eq:p1 bound}.

    We can follow exactly same steps to prove \eqref{eq:p2 bound} for $\pb_2$. The only difference is noting that for all $i\in[n_1]$ and unit norm $\vb$:
    $$
    \abs{y_i\frac{1}{(1-\zeta)T}\sum_{t\in\R_i^c}\inp{\nub_{j_t}+\z_{i,t}}{\vb}} \leq S+Z\,.
    $$
\end{proof}

\section{Optimal Model} \label{app:F}

We first restate the optimal parameters $\thetab_\text{opt} = (\Ub_\text{opt}, \W_\text{opt})$:
\begin{align}
        \Ub_\text{opt} &:= \frac{1}{S\sqrt{T}}\,\ones_d(\mubp-\mubn)^\top\,,\label{eq:Ustar}
        \\
        \W_\text{opt} &:= \frac{1}{S^2\sqrt{2(M+1)}}\Big(\mub_+\mub_+^{\top} + \mub_{-}\mub_{-}^{\top} + \sum_{\ell\in[M]}\nub_\ell(\mub_++\mub_{-})^{\top}\Big)\,,\label{eq:Wstar}
    \end{align}
    normalized so that $\|\thetab_\text{opt}\|_F = \sqrt{2}$. 

The following lemma about saturation in softmax scores is used to prove Proposition \ref{propo:single head good theta}.
\begin{lemma}[Softmax saturation]
\label{lem:softmax saturation}
    Let relevance-scores vector $\rb=[r_1,\ldots,r_T]\in\R^T$ be such that for some $L\in[T]$ and $A,B\in\R$: 
    \[
    r_1\geq r_2 \geq \ldots\geq r_L \geq A
    \qquad\text{and}\qquad 
    B\geq \max\{ r_i\,|\, i=L+1,\ldots,T \}.
    \]
    Further assume $A>2B$. Fix any $\eps>0$ and 
    \[
    \Gamma\geq \frac{2}{A}\log\Big(\frac{T/L-1}{\eps}\Big)\,.
    \]
    Then, for the attention weights $\ab=[a_1,\ldots,a_T]:=\sft{\Gamma\rb}\in\R^T$ 
    it holds that
    \begin{align}\label{eq:attn eps sandwich}
    0\leq 1-\sum_{i\in[L]} a_i  = \sum_{i=L+1}^{T} a_i\leq \eps\,.
    \end{align}
\end{lemma}
\begin{proof}
For convenience denote $D:=\sum_{j\in[T]}e^{\Gamma r_j}$. Note that  $D \geq L e^{\Gamma A}.$
    Consider any $i>L$. Then,
    \begin{align}\nn
        a_i = \frac{e^{\Gamma r_i}}{D} \leq \frac{e^{\Gamma B}}{D} 
        \leq \frac{e^{\Gamma B}}{Le^{\Gamma A}} = \frac{1}{Le^{\Gamma (A-B)}}
        \leq \frac{1}{Le^{\Gamma A/2}}\,.
    \end{align}
    Suppose $\Gamma\geq \frac{2}{A}\log\big(\frac{C}{\eps}\big)$, which ensures that 
    \[
     e^{\Gamma A/2} \geq {C}/{\eps} 
     \]
    Setting $C=\frac{T-L}{L}=\frac{T}{L}-1$, and combining the above two displays yields the desired:
    \begin{align*}
        a_i &\leq \frac{\eps}{T-L},\qquad i>L\,.
    \end{align*}
    Thus, $\sum_{i>L}a_i \leq \eps $. The proof is complete by recalling 
     that $\sum_{i\in[T]}a_i=1$, hence $\sum_{i\in[L]}a_i\geq 1-\eps$.
\end{proof}

\subsection{Proof of Proposition \ref{propo:single head good theta}}\label{sec:proof of proposition single head good weights}
First, we  compute the attention matrix when the parameter $\W$ is set to the value below:
\begin{align}\label{eq:W not normalized}
\W=\mub_+\mub_+^{\top} + \mub_{-}\mub_{-}^{\top} + \sum_{\ell\in[M]}\nub_\ell(\mub_++\mub_{-})^{\top}\,.
\end{align}
For convenience, use the notation
$\rb_{t}^\top := \x_t^{\top}\W\X^{\top} \in \R^T$ for the $t$-th row of matrix used to find attention scores.  Similar to the proof of Lemma \ref{lem: singlehead ntk} we consider two cases where row corresponds to a signal relevant of noisy token. We denote the $t'\in[T]$ entry of $\rb_t$ as $[\rb_t]_{t'}\in\R.$

\noindent\underline{Case 1. Relevance scores of signal tokens:}
Consider signal token $t\in\Rc$ so that $\x_t=\mub_y$. Then for weights $\W$ in \eqref{eq:W not normalized}, and using orthogonality in Assumption \ref{ass:data} we can compute for all $t'\in[T]$
\begin{align*}    
t \in \Rc: [\rb_t]_{t'} =
\begin{cases}
    S^4 & \text{,} \, t' \in \Rc\,, \\
    S^2(\mub_y^{\top}\z_{t'}) & \text{,} \, t' \in \Rc^c\,. \\
\end{cases}
\end{align*}
Therefore, again using Assumption \ref{ass:data},
\begin{align} \label{eq:relevance W signal}   
t \in \Rc: [\rb_t]_{t'} 
\begin{cases}
    =S^4 & \text{,} \, t' \in \Rc\,, \\
    \leq S^2\zmu & \text{,} \, t' \in \Rc^c\,. \\
\end{cases}
\end{align}

\noindent\underline{Case 2. Relevance scores of noisy tokens:} Similar to the calculations above, using Assumption \ref{ass:data} for parameters $\W$ as in \eqref{eq:W not normalized} it holds for noisy tokens $t\in\Rcc$ that
\begin{align*}
t \in \Rc^c: [\rb_t]_{t'} =
\begin{cases}
    S^4 + S^2(\mub_y^{\top}\z_{t}) + S^2(\sum_\ell \nub_\ell^{\top}\z_t) & \text{,} \, t' \in \Rc \\
    (\mub_{+}^{\top}\z_{t})(\mub_{+}^{\top}\z_{t'}) +(\mub_{-}^{\top}\z_{t})(\mub_{-}^{\top}\z_{t'}) & \text{,} \, t' \in \Rc^c \\ + \sum_\ell (\nub_\ell^{\top}\z_t)(\mub_{+}^{\top}\z_{t'}) 
    + \sum_\ell (\nub_\ell^{\top}\z_t)(\mub_{-}^{\top}\z_{t'}) + S^2(\mub_{+}+\mub_{-})^{\top}\z_{t'}\\
\end{cases}
\end{align*}
Therefore, using the noise bound assumptions, we have
\begin{align} \label{eq:relevance W noise}
t \in \Rc^c: [\rb_t]_{t'} 
\begin{cases}
    \geq S^4 - S^2(\zmu+\znu) & \text{,} \, t' \in \Rc \\
    \leq 2\zmu^2+2\zmu\znu+2S^2\zmu
     & \text{,} \, t' \in \Rc^c \,.
\end{cases}
\end{align}

Combining the above two cases, specifically Equations \eqref{eq:relevance W signal} and \eqref{eq:relevance W noise}, we find that for all $t\in[T]$ the relevance-score vectors $\rb_t$ are such that
\begin{align}
t \in [T]: [\rb_t]_{t'} 
\begin{cases}
    \geq S^4 - S^2(\zmu+\znu):=A & \text{,} \, t' \in \Rc \\
    \leq 2(\zmu^2+\zmu\znu+S^2\zmu):=B
     & \text{,} \, t' \in \Rc^c \,,
\end{cases}    
\end{align}
where we defined the parameters $A$ and $B$ for convenience. Note from assumption that 
\[\zmu=\znu\leq \frac{S^2}{8} 
\implies A \geq \frac{3}{4} S^4 > \frac{1.25}{4} S^4 \geq 2B\,.\]

Thus, the conditions of Lemma \ref{lem:softmax saturation} hold for $L=|\Rc|=\zeta T$. 
Applying the lemma we can immediately conclude that for
\[
\Gamma_* \geq \frac{8}{3S^4}\log\left(\frac{\zeta^{-1}-1}{\eps}\right) \geq \frac{2}{A}\log\left(\frac{\zeta^{-1}-1}{\eps}\right) \, .
\]
it holds:
\begin{align}\label{eq:attn bound for W}
    \forall t\in[T]\,:\, 0\leq 1-\sum_{t'\in\Rc} \left[\sft{\Gamma_*\rb_t}\right]_{t'} =\sum_{t'\in\Rcc} \left[\sft{\Gamma_*\rb_t}\right]_{t'} \leq \eps \, .
\end{align}

Now, recall that \[
\Wstar=\widetilde\Gamma\W\qquad\text{for}\quad \widetilde\Gamma=\Gamma/(S^2\sqrt{2(M+1)}) \, .
\]
Thus, it holds for all $t\in[T]$ that
$
\sft{\x_t^\top\Wstar\X^T} = \sft{\widetilde\Gamma\rb_t}\,.
$
Combining this with \eqref{eq:attn bound for W} and the proposition's assumption on $\Gamma$ (satisfying $\widetilde{\Gamma}\geq \Gamma_*$), we have found that
\begin{align}\label{eq:attn bound for Wstar}
    \forall t\in[T]\,:\, 0\leq 1-\sum_{t'\in\Rc} \left[\sft{\x_t^\top\Wstar\X^T}\right]_{t'} =\sum_{t'\in\Rcc} \left[\sft{\x_t^\top\Wstar\X^T}\right]_{t'} \leq \eps\,.
\end{align}

In the rest of the proof, we use \eqref{eq:attn bound for Wstar} to lower-bound the margin:
\begin{align*}
y\Phi(\X;\thetabstar) &= y\inp{\Ubstar}{\attn(\X;\Wstar)} = \frac{y}{S\sqrt{2T}}\sum_{t\in[T]}\sum_{t'\in[T]}\left[\phi(\x_t^\top\Wstar\X^T)\right]_{t'}(\mubp-\mubn)^\top\x_{t'} =:\frac{y}{S\sqrt{2T}}\sum_{t\in[T]} \psi_t \, ,
\end{align*}
where we defined $\psi_t, t\in[T]$ for convenience.
For any $t\in[T]$, we have
\begin{align*}
    \psi_t &= \sum_{t'\in\Rc}\left[\phi(\x_t^\top\Wstar\X^T)\right]_{t'}yS^2
+\sum_{t'\in\Rcc}\left[\phi(\x_t^\top\Wstar\X^T)\right]_{t'}(\mubp-\mubn)^\top\z_{t'} 
\\&\geq yS^2\sum_{t'\in\Rc}\left[\phi(\x_t^\top\Wstar\X^T)\right]_{t'}
-2\zmu\sum_{t'\in\Rcc}\left[\phi(\x_t^\top\Wstar\X^T)\right]_{t'}
\\&\geq yS^2(1-\eps)
-2\eps\zmu \, .
\end{align*}
Putting the last two displays together and using $y^2=1$ completes the proof of the proposition.


\section{Linear Model} \label{app:linear}
To gain additional insights into the classification of the data model \ref{model1} and also contrast our results to a simplified model, we  examine here a linear classifier:
\begin{align*}
    \Philin(\X;\Ub) =  \inp{\Ub}{\X}\,.
\end{align*}
For this linear model, consider the oracle classifier 
\[
\Ubstar = \frac{1}{S\sqrt{2T}}\,\ones_T\ub_\star^\top,\quad \text{with}\quad \ub_\star=\mubp-\mubn\,, 
\]
and normalization such that $\|\Ubstar\|_F=1$\,.
By using Assumption \ref{ass:data},  almost surely  for all examples $(\X,y)$ the margin of the oracle classifier is lower bounded by:
\begin{align}
    y\Philin(\X;\Ubstar) &= \frac{1}{S\sqrt{2T}}\Big(|\Rc|\cdot S^2 + y\,\sum_{t\in\Rc^c} \inp{\mubp-\mubn}{\z_t}\Big) \notag \\ &\geq \frac{1}{S\sqrt{2T}}\Big(|\Rc|\cdot S^2 - 2 |\Rc^c| \cdot \zmu\Big)
    =\frac{S\,\sqrt{T}}{\sqrt{2}}\, \Big(\zeta- 2(1-\zeta) \frac{\zmu}{S^2}\Big)=:\gammalin \,. \label{eq:gammalin}
\end{align}

\section{Experiments}
In this section we provide some experiments discussing the role of number of heads $H$ in the training dynamics on synthetic data models. 

\noindent \textbf{Data Model \ref{model1}} We set the number of tokens $T=10$ and sparsity level $\zeta = 0.1$. We set $\{\mubp,\mubn,\nub_1, \nub_2, ..., \nub_M\}$ as the canonical basis vectors in $\R^d$, with $d=4, M=2$ and signal strength $S=2$. Noisy tokens $\z$ are sampled from a Gaussian $\Nc(0, \sigma^2\Iden_d)$, with $\sigma = 0.1$. We use $n=100$ training samples in each experiment and evaluate on a test set of size $300$ (total 5 trials). All models are initialized as $\thetabt^\pare{0} = \mathbf{0}$. 


 Figure \ref{fig:context-GD} shows the effect of increasing the number of heads when running GD with constant step-size $\eta=1.0$ and data generated from data model \ref{model1}. Notice that rate of train loss decay reduces as we increase $H$, highlighting a potential downside of overparameterization. A similar observation has been recently noted by \cite{simondu-gd-slow} when optimizing with GD to learn a single neuron with ReLU activation under Gaussian input. We also observe that at least for smaller $H$, GD indeed achieves margin $\gamma_\text{attn}$. It is worth noting that these observations do not contradict our theoretical findings in Theorems \ref{thm:train} and \ref{thm:gen} which guarantee training convergence and generalization decay given sufficiently large number of heads $H$, without making an explicit connection between the rates of convergence as we increase $H$. We also test how these rates change if we scale step-size as $\eta = \order{\sqrt{H}}$ for GD (Figure \ref{fig:context-GD-OH-Adam}, left) or optimize with Adam (Figure \ref{fig:context-GD-OH-Adam}, right) using constant step-size $\eta = 0.06$. It is interesting to observe that in both these cases, the convergence speeds up for larger $H$, especially when optimizing with Adam, but somewhat strangely the margin attained by GD for larger $H$ continues to fall away from $\gamma_\text{attn}$. Further, note that our theory only covers step-size $\eta = \order{1}$ and the trends observed in Figure \ref{fig:context-GD-OH-Adam} with $\eta = \order{\sqrt{H}}$ for GD fall outside this regime. In essence, we believe that it would be interesting future work to see how well these observations generalize to different datasets and develop theory that explains the relation of overparameterization to rates of convergence.


\begin{figure}[h]
    \centering
   \begin{overpic}[width=0.55\textwidth]{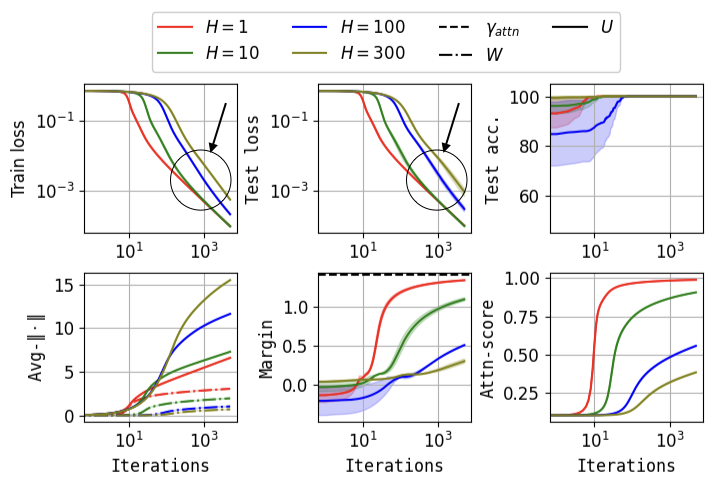}
        \put(33,68){GD with $\eta = \order{1}$}
    \end{overpic}
    \caption{\textbf{For data model \ref{model1}}. Effect of number of heads $H$ on convergence rates when trained with GD for constant step-size $\eta = \order{1}$. The average $\norm{\cdot}$ illustrates $1/H$ and $1/\sqrt{H}$ average for $\W$ and $\Ub$ across heads, respectively. Attn-score denotes the softmax scores for the relevant tokens averaged across all train samples and heads. The average $\norm{\W}$ indicates the saturation of softmax scores and consequently the token-selection (attn-score), and the average $\norm{\Ub}$ controls the loss behaviour. Results demonstrate that overparameterization slows down GD with constant step-size. The circled area shows a $\Oc(1/t)$ trend similar to what our training and generalization bounds predict.}
    \label{fig:context-GD}
\end{figure}
\begin{figure}[h]
    \centering
    \begin{overpic}[width=0.48\textwidth]{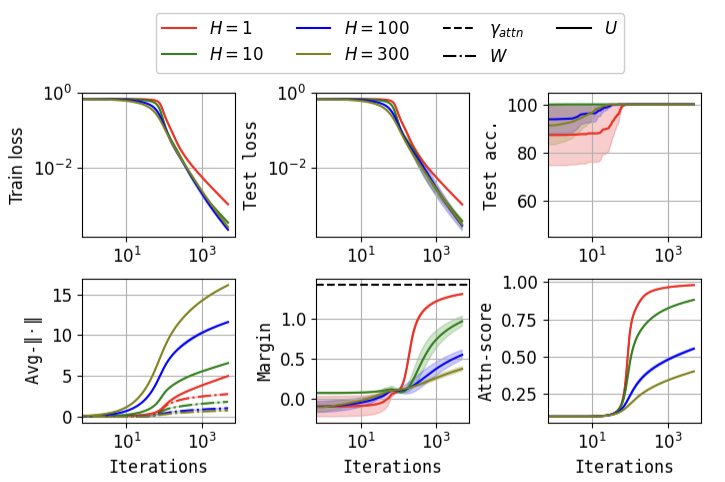}
        \put(29,68){GD with $\eta = \order{\sqrt{H}}$}
    \end{overpic}
    \begin{overpic}[width=0.48\textwidth]{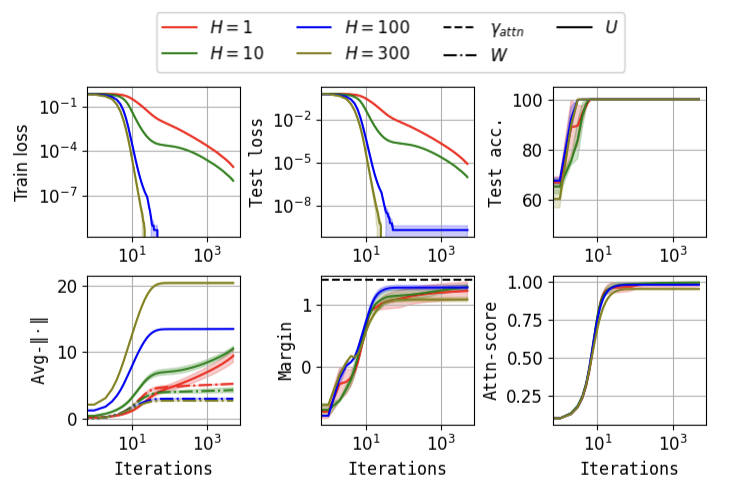}
        \put(32,68){Adam with $\eta = \order{1}$}
    \end{overpic}
    \caption{\textbf{For data model \ref{model1}}. Effect of number of heads $H$ on convergence rates when (left) trained with GD when scaling step-size as $\eta = \order{\sqrt{H}}$; (right) trained with Adam with constant step-size  $\eta = \order{1}$. 
    Quantities plotted are same as in Figure \ref{fig:context-GD}. Results demonstrate that overparameterization speeds-up with train and test loss convergence in both the scenarios.}
    \label{fig:context-GD-OH-Adam}
\end{figure}


\noindent \textbf{Planted data model} Fix some $\W^* \in \R^{d \times d} , \, \Ub^* \in \R^{T \times d}$. Entries within $\X$ are sampled \iid $X_{ij} \sim \Nc(0, 1), \,\, \forall i \in [T], \, j \in [d]$. Given such an $\X \in \R^{T \times d}$, generate the label $y$ using $\W^*$ as the attention matrix and $\Ub^*$ as the projection classifier:
\begin{align} \label{model planted}\tag{DM2}
    y = \sign{\Phi(\X; \W^*, \Ub^*)} = \sign{\inp{\Ub^*}{\sft{\X\W^*\X^{\top}}\X}} \, .
\end{align}

Data generated using model \ref{model planted} is used to train a multi-head self-attention model as given in equation \eqref{eq:SA model}. Such a \emph{teacher-student} setting (train the student network to learn the ground truth parent) has been well explored in the past in the context of neural networks \citep{student-teacher-local, student-teacher-shamir}. We set $d=5,\, T=10$. The train set contains $n=1000$ samples in each experiment and we evaluate on a test set of size 3000. Each result is averaged over 5 trials. For numerical ease, while generating (example, label) pairs we drop the samples for which $\abs{\text{output logit}} \leq \gamma_\text{attn}$, where we call $\gamma_\text{attn} > 0$ to be margin for the data model. We set $\gamma_\text{attn} = 0.2$ in all the experiments. All models are initialized as $\thetabt^\pare{0} = \mathbf{0}$. From Figure \ref{fig:planted-GD} we observe that overparameterization somewhat improves convergence speeds for GD with step-size $\order{\sqrt{H}}$, similar to tokenized mixture model (Figure \ref{fig:context-GD-OH-Adam}, left). Addition of momentum significantly helps speeding up convergence (see Figure \ref{fig:planted-GD-mom-Adam}, left) and so does optimizing with Adam (Figure \ref{fig:planted-GD-mom-Adam}, right). Interesting to note that all the models reach the expected margin $\gamma_\text{attn}$ which was not the case for large $H$ for the tokenized mixture model. Further, we can observe that initializing at $\thetabt^\pare{0} = \mathbf{0}$, all the optimizers find the planted model. 

\noindent \textbf{SST2 dataset} We conduct an additional experiment on a simple real-world dataset. The SST2 dataset \citep{sst2} consists of sentences, with each sentence having a associated binary label to classify the sentiment. We fine-tune RoBERTa based models with varying number of heads using AdamW \citep{adamW} optimizer with a learning rate of $5e-6$. We train all the models for $5$ epochs, with the batch-size set to $32$. We use the Hugging Face \tsc{pytorch-transformers} implementation of the \tsc{roberta-base} model, with pretrained weights \citep{roberta}. In Figure \ref{fig:roberta} we see that increasing the number of heads speeds up the optimization and generalization. This behaviour is similarly observed for GD with momentum and Adam in Figure \ref{fig:planted-GD-mom-Adam}.

\begin{figure}[h]
    \centering
   \begin{overpic}[width=0.55\textwidth]{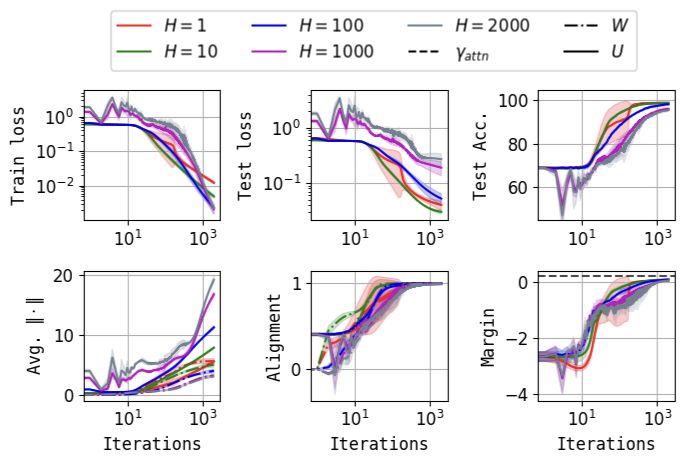}
        \put(33,68){GD with $\eta = \order{\sqrt{H}}$}
    \end{overpic}
    \caption{\textbf{For data model \ref{model planted}}. Effect of number of heads $H$ on convergence rates when trained with GD when scaling step-size as $\eta = \order{\sqrt{H}}$. See Figure \ref{fig:context-GD} caption for to get more context on average $\norm{\cdot}$. Alignment of $\W$ with the planted-head $\Wt^\star$ at any iteration $k$ is given by $\frac{\inp{\Wt_k}{\Wt^\star}}{\norm{\Wt_k}\norm{\Wt^\star}}$, where $\Wt^\star:=\concat\big(\{\W^\star\}_{h\in[H]}\big)$ contains  $\W^\star$ repeated $H$ times. Alignment between $\Ubt$ and $\Ubt^\star$ is computed similarly. }
    \label{fig:planted-GD}
\end{figure}

\begin{figure}[h]
    \centering
    \begin{overpic}
    [width=0.48\textwidth]{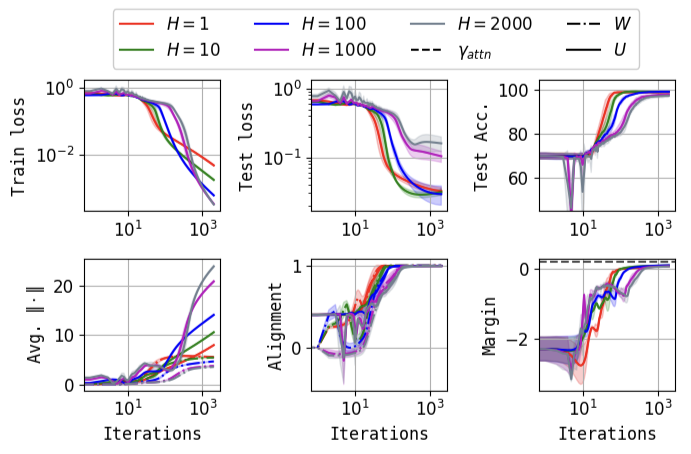}
        \put(20,68){GD with $\eta = \order{\sqrt{H}}$ + momentum}
    \end{overpic}
    \begin{overpic}[width=0.48\textwidth]{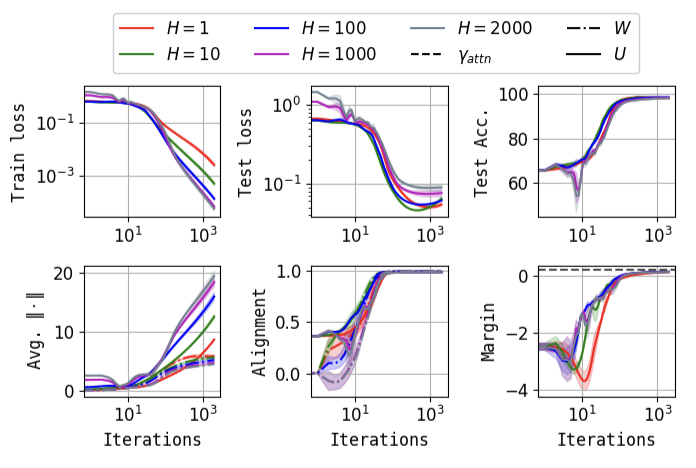}
        \put(33,68){Adam with $\eta = \order{1}$}
    \end{overpic}
    \caption{\textbf{For data model \ref{model planted}}. Effect of number of heads $H$ on convergence rates when trained with (left) GD + momentum where step-size scales as $\eta = \order{\sqrt{H}}$; (right) Adam with constant step-size $\eta = \order{1}$. Quantities plotted are same as in Figure \ref{fig:planted-GD}. Results demonstrate that overparameterization speeds up convergence in both scenarios.}
    \label{fig:planted-GD-mom-Adam}

\end{figure}

\begin{figure}[h]
    \centering
    \includegraphics[width=0.55\textwidth]{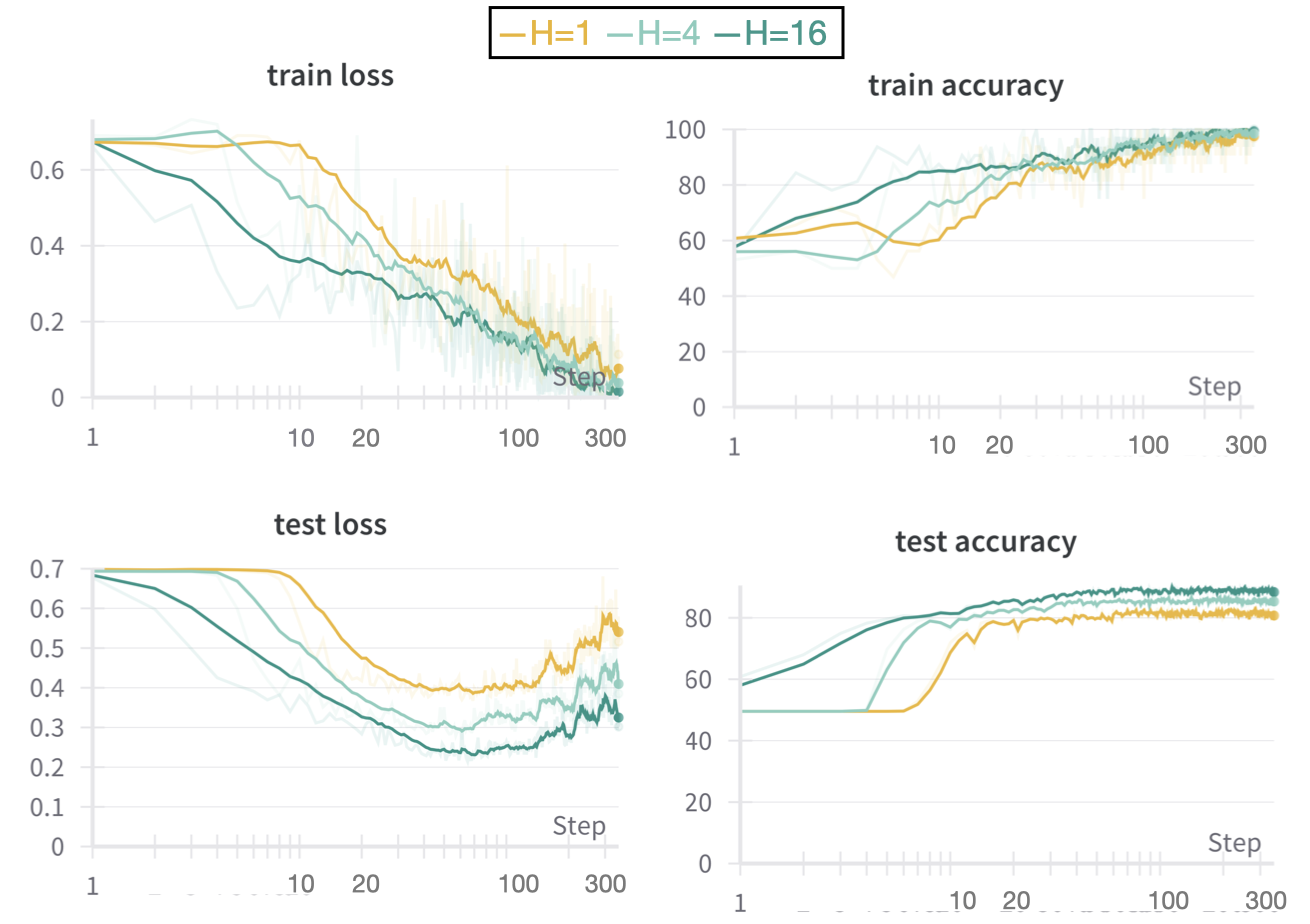}
    \caption{\textbf{For SST-2 dataset \citep{sst2}}. Effect of number of heads $H$ on convergence rates when trained with AdamW. 
    Results demonstrate that increasing the number of heads speeds up the training and generalization dynamics.}
    \label{fig:roberta}
\end{figure}
\section{Related work}\label{sec:rel}

This section elaborates on the paragraph on related work in Section \ref{sec:intro}.

\paragraph{Transformers and Self-attention.}
The landscape of NLP and machine translation was profoundly reshaped by the advent of Transformers, as pioneered by \cite{Vaswani2017AttentionIA} building upon earlier investigations into self-attention as explored in the works of \cite{cheng-etal-2016-long, lin2017structured, parikh-etal-2016-decomposable}. More recent developments include the transformative success of large language models like, LLaMA \citep{touvron2023llama}, ChatGPT \citep{chatgpt}, and GPT4 \citep{gpt4}. Despite this, the learning dynamics of the self-attention mechanism  remain largely unknown. Some recent works have focused on understanding the expressive power \citep{baldi2022quarks,dong2021attention,yun2020transformers, yun2020on,  sanford2023representational, bietti2023birth} and memory capacity of the attention mechanism \citep{baldi2022quarks,dong2021attention,yun2020transformers,yun2020on,mahdavi2023memorization}.  Other aspects which are explored include obtaining convex reformulations for the training problem \citep{sahiner2022unraveling,ergen2022convexifying}, studying sparse function representations in self-attention mechanism \citep{edelman2021inductive, likhosherstov2021expressive} and investigating the inductive bias of masked self-attention models. Additionally, a sub-area gaining increasing popularity is theoretical investigation of in-context learning, e.g.  \citep{Oswald2022TransformersLI, zhang2023trained, akyrek2023what, li2023transformers}. 

Here, we discuss works that aim to understand the optimization and generalization dynamics of self-Attention or its variants. 
\cite{prompt-attention} diverges from traditional self-Attention by focusing on a variant called prompt-Attention, aiming to gain understanding of prompt-tuning. \cite{diff-eqn-attn} show that for  a bag of words model, an attention model optimized with gradient flow for a topic classification task discovers the “topic” word as training proceeds. However, they don’t provide finite-time optimization-generalization rates.
\cite{jelassi2022vision} shed light on how ViTs learn spatially localized patterns, even though this spatial structure is no longer explicitly represented after the image is split into patches. Specifically, they show that for binary classification using gradient-based methods from random initialization, transformers implicitly prefer the solution that learns the spatial structure of the dataset. \cite{li2023theoretical} provided theoretical results on training three-layer ViTs for classification tasks for a similar data model as ours (tokenized-mixture data). They provide sample complexity for achieving zero generalization error, and also examined the degree of sparsity in attention maps when trained using SGD. Contemporaneous works include \citep{tian2023scan, tarzanagh2023transformers}: The former presents SGD-dynamics of single-layer transformer for the task of next-token prediction by re-parameterizing the original problem in terms of the softmax and classification logit matrices, and analyzing their training dynamics instead. The latter studies the implicit bias of training the softmax weights $\W$ with a fixed decoder $\Ub$ via a regularization path analysis. 
All these works focus on a single attention head. Instead, we leverage the use of multiple heads to establish connections to the literature on GD training of overparameterized neural networks. Conceptually similar connections have also been studied by \cite{hron2020infinite} who connect multi-head attention to a limiting Gaussian process when the number of heads increase to infinity. In contrast, we study performance in the more practical regime of finite number of heads and obtain
and obtain \emph{finite-time} optimization and generalization bounds. 



\paragraph{Overparameterized MLPs.}
There has been an abundance of literature that discusses NN training convergence and generalization dynamics via an NTK type approach, e.g. \citep{allen2019convergence,oymak2020toward,arora2019fine, nguyen2020global,banerjee2022restricted,nguyensmallest21,zhu2023benign}. However, most of these works focus on GD dynamics on regression problems using square loss and relate the training convergence and generalization to the minimum eigenvalue of the Hessian of the NTK. On the other hand, relatively fewer works focus on classification with logistic loss under an NTK separable data assumption, and \citep{nitanda2019gradient,Ji2020Polylogarithmic,cao2019generalization,chen2020much, telgarsky2013margins,taheri2023generalization} are most relevant works that share overlapping ideas with our work. We  refer the reader to these for a more thorough overview of the NTK-regime analysis for NNs.

Other than these, \cite{richards2021stability, richards2021learning, taheri2023generalization, leistabilitynn} use algorithmic-stability based tools to understand the training dynamics and generalization of GD in shallow NNs. \cite{leistabilitynn, richards2021stability} establish generalization bounds with polynomial width $\Tilde{\Omega}(\text{poly}(n))$ requirement while minimizing square loss. Here, we make use of the tools developed by \cite{taheri2023generalization} who work with self-bounded Lipschitz loss functions, like logistic loss, similar to our analysis. The algorithmic stability arguments in the analysis of the above referenced papers are rooted on a technique to bound generalization gap by directly relating it to train loss based on the notion of average model stability introduced by \cite{lei2020fine}. This technique has also been leveraged by \cite{pmlr-v178-schliserman22a} to study linear logistic regression and \cite{nikolakakis2022beyond} who establish  generalization-risk bounds for Lipschitz function optimization with bounded optimal set.
 





\end{document}